\documentclass[11pt]{article}
\pdfoutput=1

\usepackage[utf8]{inputenc}
\usepackage[T1]{fontenc}

\usepackage[verbose=true,
letterpaper,
textheight=8.7in,
textwidth=6.4in,
margin=1in,
headheight=12pt,
headsep=25pt,
footskip=30pt]{geometry}

\usepackage{microtype}
\usepackage{mathtools}
\usepackage{booktabs}
\usepackage[numbers,sort&compress]{natbib}
\usepackage{graphicx}
\usepackage{amsmath,amssymb,amsfonts,amsxtra,bm}
\usepackage{amsthm}
\usepackage{xspace}
\usepackage{enumerate}
\usepackage{hyperref}
\usepackage[capitalise]{cleveref}
\usepackage{url}

\usepackage{colortbl}

\usepackage{makecell,multirow}       
\usepackage{nicefrac}       
\usepackage{thmtools}  
\usepackage{xspace}
\usepackage{footnote, tablefootnote}
\usepackage{float}
\floatstyle{plaintop}
\restylefloat{table}
\usepackage{epstopdf}
\usepackage{latexsym}
\usepackage{graphicx}

\usepackage{caption,subcaption,enumitem}
\usepackage{tablefootnote}
\usepackage{placeins}

\numberwithin{equation}{section}
\usepackage{ss}
\usepackage{physics}

\usepackage{natbib}
\setcitestyle{authoryear,round,citesep={;},aysep={,},yysep={;}}

\usepackage{amsmath,amsfonts,bm}

\def\eqref#1{equation~\ref{#1}}

\def\1{\bm{1}}

\def\eps{{\epsilon}}

\def\vzero{{\bm{0}}}
\def\vone{{\bm{1}}}
\def\vmu{{\bm{\mu}}}

\def\va{{\bm{a}}}
\def\vb{{\bm{b}}}
\def\vc{{\bm{c}}}

\def\ve{{\bm{e}}}

\def\vg{{\bm{g}}}

\def\vq{{\bm{q}}}
\def\vr{{\bm{r}}}
\def\vs{{\bm{s}}}

\def\vu{{\bm{u}}}
\def\vv{{\bm{v}}}
\def\vw{{\bm{w}}}
\def\vx{{\bm{x}}}
\def\vy{{\bm{y}}}

\def\mA{{\bm{A}}}
\def\mB{{\bm{B}}}

\def\mD{{\bm{D}}}

\def\mG{{\bm{G}}}
\def\mH{{\bm{H}}}
\def\mI{{\bm{I}}}

\def\mM{{\bm{M}}}

\def\mS{{\bm{S}}}

\def\mW{{\bm{W}}}
\def\mX{{\bm{X}}}
\def\mY{{\bm{Y}}}
\def\mZ{{\bm{Z}}}

\def\mLambda{{\bm{\Lambda}}}

\DeclareMathAlphabet{\mathsfit}{\encodingdefault}{\sfdefault}{m}{sl}
\SetMathAlphabet{\mathsfit}{bold}{\encodingdefault}{\sfdefault}{bx}{n}

\def\sS{{\mathbb{S}}}

\DeclareMathOperator*{\argmax}{arg\,max}

\DeclareMathOperator{\interior}{int}
\DeclareMathOperator{\relinterior}{relint}
\DeclareMathOperator{\aff}{aff}
\newcommand{\RR}{\mathbb{R}}
\DeclareMathOperator{\EE}{\mathbb{E}}
\DeclareMathOperator{\PP}{\mathbb{P}}

\DeclareMathOperator{\diag}{diag}
\DeclareMathOperator{\poly}{poly}
\newcommand{\loss}{\mathcal L}
\newcommand{\ellpi}{\mathcal L_{\pi}}
\newcommand{\ellrr}{\mathcal L_{\mathrm{RR}}}

\newcommand{\ellgd}{\mathcal L_{\mathrm{GD}}}
\renewcommand{\vec}{\mathrm{vec}}
\DeclareMathOperator{\conv}{conv}
\DeclareMathOperator{\Span}{Span}

\newcommand{\ol}{\overline}
\DeclareMathOperator{\sgn}{sgn}

\newcommand{\Xgdbn}{\ol{\mX}_{\mathrm{GD}}}

\newcommand{\Xrrbn}{\ol{\mX}_{\mathrm{RR}}}
\newcommand{\Xpi}{\mX_{\pi}}
\newcommand{\Xpibn}{\ol{\mX}_{\pi}}
\newcommand{\Xmax}{\ol{\mX}_{\mathrm{max}, 2}}
\newcommand{\Xmaxf}{\ol{\mX}_{\mathrm{max}, F}}

\newcommand{\Ygd}{\mY_{\mathrm{GD}}}
\newcommand{\Yrr}{\mY_{\mathrm{RR}}}
\newcommand{\Ypi}{\mY_{\pi}}
\newcommand{\Yhatpi}{\hat{\mY}_{\pi}}

\newcommand{\Zgd}{\ol{\mZ}_{\mathrm{GD}}}
\newcommand{\Zrr}{\ol{\mZ}_{\mathrm{RR}}}
\newcommand{\Zpi}{\ol{\mZ}_{\pi}}

\newcommand{\Zrrbn}{\ol{\mZ}_{\mathrm{RR}}}
\newcommand{\Zpibn}{\ol{\mZ}_{\pi}}

\newcommand{\Sgdls}{\mS_{\mathrm{GD}}^{\mathrm{LS}}}
\newcommand{\Sgdsc}{\mS_{\mathrm{GD}}^{\mathrm{SC}}}

\newcommand{\Sls}{\mS^{\mathrm{LS}}}
\newcommand{\Ssc}{\mS^{\mathrm{SC}}}
\newcommand{\Xls}{\mX^{\mathrm{LS}}}
\newcommand{\Xsc}{\mX^{\mathrm{SC}}}

\newcommand{\Xpisc}{\Xpibn^{\mathrm{SC}}}

\newcommand{\Xgdsc}{\Xgdbn^{\mathrm{SC}}}
\newcommand{\Yls}{\mY^{\mathrm{LS}}}
\newcommand{\Ysc}{\mY^{\mathrm{SC}}}

\newcommand{\Mgd}{\mM_{\mathrm{GD}}}
\newcommand{\Mpi}{\mM_{\pi}}
\newcommand{\Mrr}{\mM_{\mathrm{RR}}}

\newcommand{\vgd}{\vv_{\mathrm{GD}}}
\newcommand{\vpi}{\vv_{\pi}}

\newcommand{\Grr}{G_{\mathrm{RR}}}
\newcommand{\Gpi}{G_{\pi}}
\newcommand{\alpharr}{\alpha_{\mathrm{RR}}}
\newcommand{\alphapi}{\alpha_{\pi}}
\newcommand{\vsigma}{\bm{\sigma}}
\newcommand{\vgt}{\Tilde{\bm{g}}}
\newcommand{\Cweight}{C_{\mathrm{w}}}
\newcommand{\Bweight}{A_{\mathrm{w}}}
\newcommand{\Bloss}{A_L}
\newcommand{\Brr}{A_{\mathrm{RR}}}

\newcommand{\Tpi}[2]{T_{\pi}^{\qty{#1,#2}}}
\newcommand{\kpi}[1]{k_{\pi}^{#1}}
\newcommand{\partI}{\textcolor{red}{\text{(I) }}}
\newcommand{\partII}{\textcolor{blue}{\text{(II) }}}
\newcommand{\mGamma}{\bm{\Gamma}}

\newcommand{\bn}{{\mathsf{BN}}}
\newcommand{\bnpi}{{\mathsf{BN}_{\pi}}}
\newcommand{\bnrr}{{\mathsf{BN}_{\mathrm{RR}}}}
\newcommand{\errorterm}{A}

\DeclareMathOperator{\relu}{ReLU}
\usepackage{graphicx}
\graphicspath{ {pics/} }
\crefname{assumption}{assumption}{assumptions}
\Crefname{assumption}{Assumption}{Assumptions}

\title{On the Training Instability of Shuffling SGD with Batch Normalization}

\author{\name David X. Wu \email{david\_wu@berkeley.edu}\\
  \addr{UC Berkeley, Berkeley, CA, USA 94720} \\ 
  \name Chulhee Yun \email{chulhee.yun@kaist.ac.kr}\\
    \addr{KAIST, Seoul, Korea, 02455}\\
  \name Suvrit Sra \email{suvrit@mit.edu}\\
    \addr{Massachusetts Institute of Technology, Cambridge, MA, USA 02139}
}

\begin{document}

\maketitle

\begin{abstract}
We uncover how SGD interacts with batch normalization and can exhibit undesirable training dynamics such as divergence. More precisely, we study how Single Shuffle (SS) and Random Reshuffle (RR)---two widely used variants of SGD---interact surprisingly differently in the presence of batch normalization: \emph{RR leads to much more stable evolution of training loss than SS}. As a concrete example, for regression using a linear network with batch normalization, we prove that SS and RR converge to distinct global optima that are ``distorted'' away from gradient descent. Thereafter, for classification we characterize conditions under which training divergence for SS and RR can, and cannot occur. We present explicit constructions to show how SS leads to distorted optima in regression and divergence for classification, whereas RR avoids both distortion and divergence.  We validate our results by confirming them empirically in realistic settings, and conclude that the separation between SS and RR used with batch normalization is relevant in practice.
\end{abstract}

\section{Introduction}
\label{sec:intro}
Recent work in deep learning theory attempts to uncover how 
the choice of optimization algorithm and architecture
influence training stability and efficiency. On the optimization front, stochastic gradient descent (SGD) is the \emph{de facto} workhorse, and its importance has correspondingly led to the development of many different variants that aim to increase the ease and speed of training, such as AdaGrad \citep{duchi2011adaptive} and Adam \citep{kingma2014adam}.

In reality, practitioners often do not use with-replacement sampling of gradients as required by SGD. Instead they use \emph{without-replacement} sampling, leading to two main variants of SGD: single-shuffle (SS) and random-reshuffle. SS randomly samples and fixes a permutation at the beginning of training, while RR randomly resamples permutations at each epoch. These shuffling algorithms are often more practical and can have improved convergence rates \citep{haochen2019random,safran2020good,yun2021open,yun2022minibatch,cho2023sgda,cha2023tighter}. 

Architecture design offers another avenue for practitioners to train networks more efficiently and encode salient inductive biases. Normalizing layers such as BatchNorm (BN) \citep{ioffe2015batch}, LayerNorm \citep{ba2016layer}, or InstanceNorm \citep{ulyanov2016instance} are often used with SGD to accelerate convergence and stabilize training. Recent work studies how these scale-invariant layers affect training through the effective learning rate \citep{li2019exponential,li2020reconciling,wan2021spherical,lyu2022understanding}.

Motivated by these practical choices, we study how SS and RR interact with batch normalization at \emph{training time}. 
Our experiments (Fig.~1) suggest that combining SS and BN can lead to surprising and undesirable training phenomena: 
\begin{enumerate}[label=(\roman*)]
    \item The training risk often diverges when using SS+BN to train linear networks (i.e. without nonlinear activations) on real datasets (see \Cref{fig:3l_lnn_training_loss}), while using SS without BN does not cause divergence (see \Cref{fig:no_bn}). 
    \item Divergence persists after tuning the learning rate and other hyperparameters (\Cref{sec:clf-experiments}) and  also manifests more quickly in deeper linear networks (\Cref{fig:3l_lnn_training_loss}).
    \item SS+BN usually converges slower than RR+BN in nonlinear architectures such as ResNet18 (see \Cref{fig:resnet_slow_cifar10}).  
\end{enumerate}
In light of these surprising experimental findings, we seek to develop a theoretical explanation.

\subsection{Summary of our contributions}
We develop a theoretical and experimental understanding of how shuffling SGD and BN collude to create divergence and other undesirable training behavior. Since these phenomena manifest themselves on the training risk, our results are not strictly coupled with generalization. 

Put simply, the aberrant training dynamics stem from BN \emph{not} being permutation invariant across epochs. This simple property interacts with SS undesirably, although \emph{a priori} it is not obvious whether it should. More concretely, one expects SGD+BN to optimize the gradient descent (GD) risk in expectation. However, due to BN's sensitivity to permutations, both SS+BN and RR+BN implicitly train induced risks different from GD, and also from each other. 

\begin{figure}[t]
     \centering
     \begin{subfigure}[b]{0.48\textwidth}
         \centering
         \includegraphics[width=\textwidth]{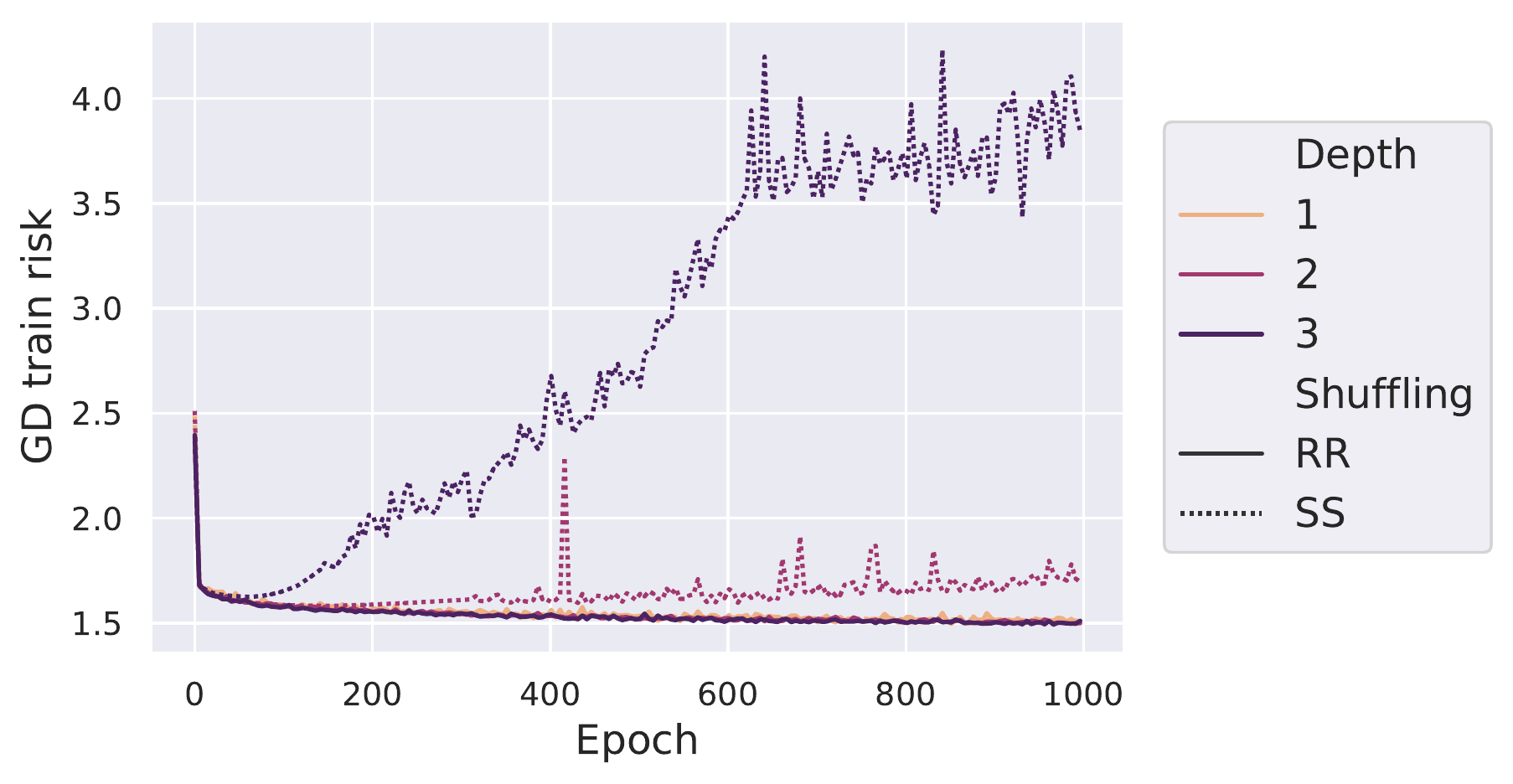}
         \caption{Depths 1, 2, and 3 linear+BN networks.}
         \label{fig:3l_lnn_training_loss}
     \end{subfigure}
     \hfill
     \begin{subfigure}[b]{0.48\textwidth}
         \centering
         \includegraphics[width=\textwidth]{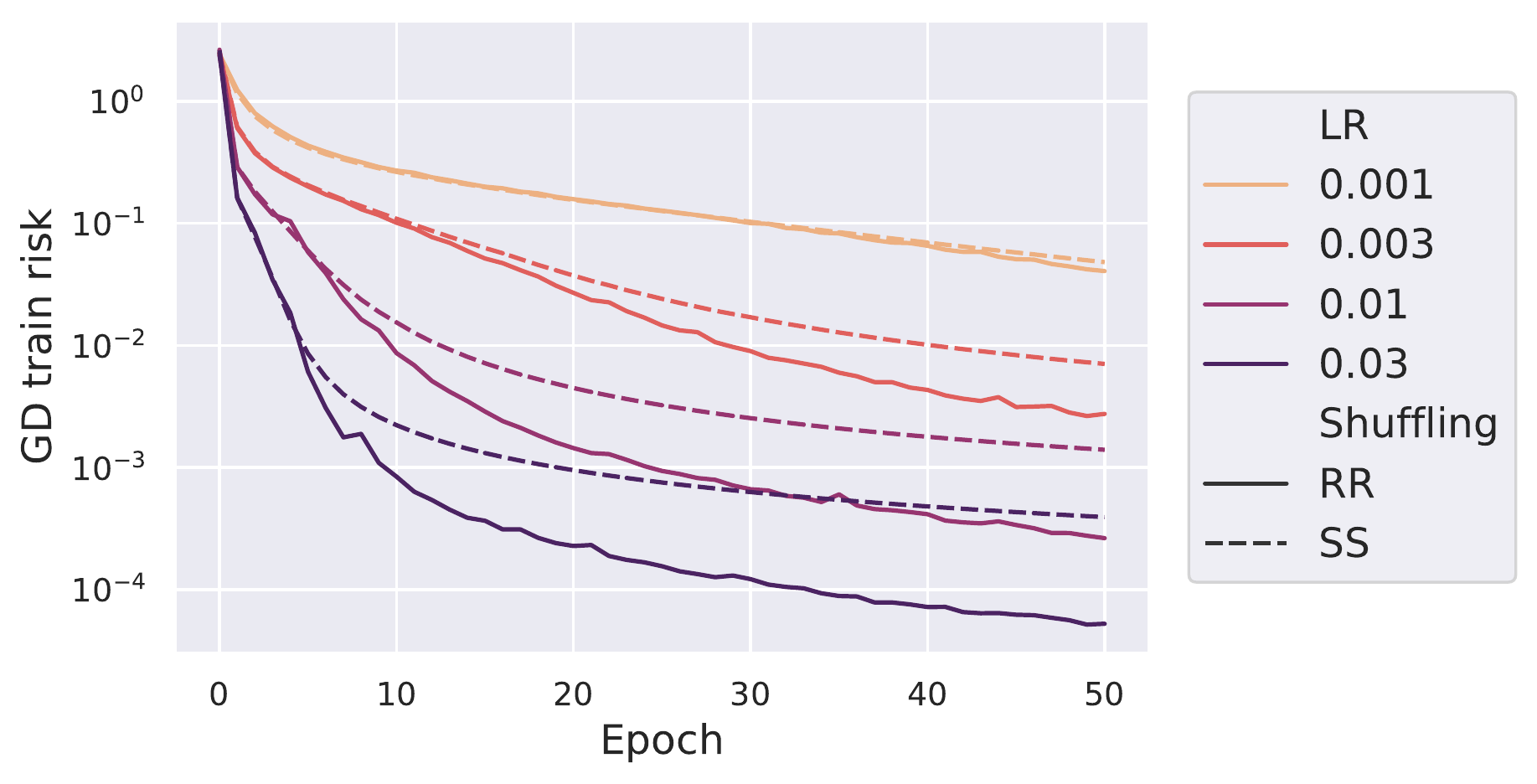}
         \caption{Finetuned ResNet18.}
         \label{fig:resnet_slow_cifar10}
     \end{subfigure}
     \caption{Surprising training phenomena using SS/RR+BN.} \label{fig:surprising-phenomena}
\end{figure}

\begin{itemize}
    \item In \Cref{subsec:regression-results}, we prove that the network $f(\mX; \Theta) = \mW\mGamma \bn(\mX)$ converges to the optimum for the distorted risk induced by SS and RR (\Cref{thm:ss-convergence,thm:rr-convergence,}); the diagonal matrix $\mGamma$ denotes the \emph{trainable} scale parameters in the BN layer. Our proof requires a delicate analysis of the evolution of gradients, the noise arising from SS, and the two-layer architecture. Due to the presence of $\mGamma$, our results do not assume a fully-connected linear network, which distinguishes them from prior convergence results. In \Cref{sec:opt-analysis-reg}, we present a toy dataset for which SS is distorted away from GD with constant probability while RR averages out the distortion to align with GD. We validate our theoretical findings on synthetic data in \Cref{sec:reg-experiments}.
    
    \item In \Cref{subsec:classification-results}, we connect properties of the distorted risks to divergence. With this step, we provide insights into which regimes can lead to divergence of the training risk  (\Cref{thm:ss-divergence-informal,thm:rr-divergence-informal}). We show that in certain regimes, SS+BN can suffer divergence, whereas RR+BN provably avoids divergence. These results motivate our construction of a toy dataset where SS leads to divergence with constant probability while RR avoids divergence (\Cref{example:clf-divergence}). In \Cref{sec:clf-experiments}, we empirically validate our results on deeper linear+BN networks on a variety of datasets and hyperparameters. Our experiments also demonstrate that SS trains more slowly than RR in more realistic nonlinear settings, including ReLU+BN networks and ResNet18. In doing so, we extend the relevance of our theoretical results to more complex and realistic settings. 
\end{itemize}

\subsection{Related work}\label{sec:related-work}
\paragraph{Interplay between BN and SGD.}
Prior theoretical work primarily studied how BN interacts with GD or with-replacement SGD \citep{arora2018theoretical,santurkar2018does,li2019exponential,cai2019quantitative,wan2021spherical,lyu2022understanding}.  
\citet{arora2018theoretical,wan2021spherical} assumed global bounds on the smoothness with respect to network parameters 
and the SGD noise to analyze convergence to stationary points. We instead prove convergence to the global minimum of the SS distorted risk $\ellpi$ with \emph{no} such 
assumptions (\Cref{thm:ss-convergence}). \citet{li2019exponential} assumed the batch size is large enough to ignore SGD noise, whereas we explicitly exhibit and study the separation between shuffling SGD and GD. For fully scale-invariant networks trained with GD, \citet{lyu2022understanding} identified an oscillatory edge of stability behavior around a manifold of minimizers. Our BN network has trainable scale-variant parameters $\mW$ and $\mGamma$, and we train with shuffling SGD instead of GD. Hence, the noise that leads to distorted risks is fundamentally different.

\paragraph{BN's effect on risk function.}
Previous work identified the distortion of risk function due to noisy batch statistics in BN. \citet{yong2020momentum} studied the asymptotic regularization effect of noisy batch statistics \emph{in expectation} for with-replacement SGD. In contrast, we characterize this noise  nonasymptotically w.h.p.\ over $\pi$ for SS and a.s.\ with respect to the data for RR. \citet{wu2021rethinking} studied the difficulty of precisely estimating the population statistics at train time, especially when using an exponential moving average. We avoid these issues altogether by evaluating directly on the GD risk. Moreover, we prove concentration inequalities for without-replacement batch statistics (\cref{prop:concentration-features}).

\paragraph{Ghost batch normalization.}
In the presence of BN, it is common practice to use \emph{ghost batch normalization}, a scheme which break up large batches into virtual ``ghost'' batches, as this tends to improve the generalization of the network \citep{hoffer2017train,shallue2019measuring,Summers2020Four}. Minibatch statistics are calculated with respect to the ghost batches, and each gradient step is computed by summing the gradient contributions from the ghost batches. This algorithm is closely related to our method of analysis for SS+BN/RR+BN. Indeed, in our setup we also break up the full batch into mini-batches, and our analysis reduces to showing that SS+BN and RR+BN trajectories track those obtained by following the aggregate gradient signal from summing over mini-batches. We comment more on the similarities between ghost BN and our setup in \cref{sec:distorted-risks}.

\paragraph{Shuffling and optimization.} Outside SGD, the effect of random shuffling has also been studied for classical nonlinear optimization schemes such as coordinate gradient descent (CGD) and ADMM (see \citet{sun2020efficiency,gurbuzbalaban2020randomness} and references therein). On convex quadratic optimization problems, they demonstrate separations in convergence rates between SS, RR, and with-replacement sampling. Our main focus is the optimum that the algorithms converge to rather than their convergence rates.

\paragraph{Implicit bias.}
Our work is also motivated by a burgeoning line of work which studies the \emph{implicit bias} of different optimization algorithms \citep{soudry2018implicit,gunasekar2018characterizing,ji2018gradient,ji2019implicit,ji2020directional,yun2021unifying,jagadeesan2022inductive}. 
These results establish how optimization algorithms such as gradient flow (GF), gradient descent (GD) or even with-replacement SGD are biased towards certain optima. For example, in the interpolating regime, GD converges to the min-norm solution \citep{gunasekar2018characterizing,woodworth2020kernel} for linear regression and the max-margin classifier for classification \citep{soudry2018implicit,pmlr-v89-nacson19a,pmlr-v89-nacson19b}.

Most directly related to our work is \citet{cao2023implicit}; they establish that linear (CNN) models $\mGamma \bn(\mW\mX)$ with BN as the final layer trained with GD converge to the (patchwise) \emph{uniform}-margin classifier with an explicit convergence rate faster than linear models without BN. Notably, their techniques are able to control the training dynamics of the $\mW$ inside of BN. In contrast, our networks are of the form $\mW\mGamma \bn(\mX)$, so the network is no longer scale-invariant with respect to $\mW$, which is essential to their analysis. Furthermore, we study the surprising interactions between shuffling SGD and BN compared to full-batch GD and BN, whereas they use full-batch GD. 

Finally, while our work does not focus on generalization, it is connected in spirit to implicit bias. Indeed, our analysis centers the study of how the risk functions and optima are affected by choices of the optimizer (SS/RR) and the architecture (BN).
\section{Problem setup}\label{sec:setup}
For $n \in \mathbb{Z}^+$ we use the notation $[n] \triangleq \qty{1, \ldots, n}$. We write $\pi$ to denote a permutation of $[n]$, and $\sS_n$ is the symmetric group of all such $\pi$. For any matrix $\mA \in \RR^{d \times n}$, $\pi \circ \mA \in \RR^{d \times n}$ is result of shuffling the columns of $\mA$ according to $\pi$. Also, $\norm{\mA}_2$ and $\norm{\mA}_F$ refer to the spectral norm and Frobenius norm, respectively. We write $\sigma_{\min}(\mA) \triangleq \inf_{\norm{\vv} = 1} \norm{\mA\vv}$ to denote minimum singular value of $\mA$.
According to our notation, $\sigma_{\min}(\mA) > 0$ \emph{only if} $\mA$ is tall or square.
We use $\Span(\mA)$ to denote the span of $\mA$'s columns.
The (coordinatewise) sign function $\sgn(\cdot): \RR \to \qty{-1, 0, 1}$ is defined as $\sgn(x) = x/\abs{x}$ for $x \neq 0$ and $\sgn(0) = 0$.

\paragraph{Data.} Let $\mZ = (\mX, \mY)$ be the given dataset, with $\mX = \begin{bmatrix}\vx_1 & \cdots & \vx_n \end{bmatrix} \in \RR^{d \times n}$ representing the feature matrix and corresponding labels $\mY = \begin{bmatrix}\vy_1 & \cdots & \vy_n \end{bmatrix}\in \RR^{p \times n}$. In the classification setting we will assume $\mY \in \qty{\pm 1}^{1 \times n}$.

\paragraph{Prediction model.} A batch normalization (BN) layer can be separated into a normalizing component $\bn$ and a scaling component $\mGamma$; we ignore the bias parameters for analysis. Given any matrix $\mB = \begin{bmatrix}\vx_1 & \cdots &\vx_q\end{bmatrix} \in \RR^{d \times q}$ (here, $q \geq 2$ is arbitrary), the normalizing transform $\bn(\cdot)$ maps it to $\bn(\mB) \in \RR^{d \times q}$ by operating coordinatewise on each $\vx_i$ in $\mB$. 
In particular, for the $k$th coordinate of $\vx_i$, denoted as $x_{i,k}$, the transform $\bn$ sends $x_{i,k} \mapsto \tfrac{x_{i,k} - \mu_{k}}{\sqrt{\sigma_{k}^2 + \epsilon}}$ where $\mu_{k}$ and $\sigma_{k}^2$ are the batch empirical mean and variance of the $k$th coordinate, respectively, and $\epsilon$ is an arbitrary positive constant used to avoid numerical instability. For technical reasons, we omit $\epsilon$ in our analysis. 
The scaling matrix $\mGamma \in \RR^{d \times d}$ is a diagonal matrix which models the tunable coordinatewise scale parameters inside the BN layer.

Throughout the paper, we consider neural networks of the form $f(\cdot; \Theta) = \mW\mGamma \bn(\cdot)$\footnote{ We can readily generalize to arbitrary learned (but frozen) feature mappings under suitable changes to the assumptions.}.
We use $\Theta = (\mW,\mGamma)$ to denote the collection of all parameters in the network. With the presence of batch normalization layers, the output of  $f$ is a function of the input datapoint \emph{as well as} the batch it belongs to. Even changing one point of a batch $\mB$ can affect the batch statistics (i.e., $\mu_k$'s and $\sigma^2_k$'s) and in turn change the outputs of $f$ for the entire batch.
The collection of network outputs for $\mB$ reads $f(\mB;\Theta) = \mW \mGamma \bn(\mB)$.

\paragraph{Loss functions.} We study regression with squared loss $\ell(\hat \vy,\vy) \triangleq \norm{\hat \vy - \vy}^2$ and binary classification with logistic loss $\ell(\hat y, y) \triangleq -\log(\rho(y \hat y))$, where $\rho(t) = 1/(1+e^{-t})$. Let $\hat \mY, \mY \in \RR^{p \times q}$ denote network outputs and true labels for a mini-batch of $q$ datapoints, respectively. Define the mini-batch risk as the columnwise sum 
\[
\mathcal L(\hat \mY, \mY) \triangleq \sum_{i=1}^q \ell (\hat \mY_{:,i},\mY_{:,i}),
\]
where $\mY_{:,i}$ denotes the $i$th column of $\mY$.

\paragraph{Optimization methods.}
We consider shuffling-based variants of SGD, namely single-shuffle (SS) and random-reshuffle (RR). These algorithms proceed in \emph{epochs}, i.e., full passes through shuffled dataset.
As the names suggest, SS randomly samples a permutation $\pi \in \sS_n$ at the beginning of the first epoch and adheres to this permutation. RR randomly resamples permutations $\pi_k \in \sS_n$ at each epoch $k$.

Throughout, the (mini-)batch size will be denoted as $B$. For simplicity, we assume that the $n$ datapoints can be divided into $m$ batches of size $B$. With a permutation $\pi \in \sS_n$, the dataset $\mZ = (\mX, \mY)$ is thus perfectly partitioned into $m$ batches $(\mX_\pi^1, \mY_\pi^1), \ldots, (\mX_\pi^m, \mY_\pi^m)$, where $\mX_\pi^j \in \RR^{d \times B}$ and $\mY_\pi^j \in \RR^{p \times B}$ consist of the $(j(B-1)+1, \ldots, jB)$th columns of the shuffled $\pi \circ \mX$ and $\pi \circ \mY$, respectively. 

For a parameter $\Theta$ optimized with SS or RR, we denote the $j$th iterate on the $k$th epoch by $\Theta_j^k$. The starting iterate of the $k$th epoch is $\Theta_0^k$ which is equal to the last iterate of the previous epoch $\Theta_m^{k-1}$.  
For each $j \in [m]$, SS and RR perform a mini-batch SGD update with stepsize $\eta_k > 0$:
\[
\Theta_j^k \leftarrow \Theta_{j-1}^k - \eta_k \grad_\Theta \mathcal L(f(\mX_{\pi_k}^j;\Theta_{j-1}^k),\mY_{\pi_k}^j).
\]

\section{Main regression results: convergence to optima of distorted risks}\label{sec:regression}
In this section, we introduce the framework of distorted risks to elucidate the distinction between SS+BN and RR+BN. This framework also applies to classification; we continue to study it in \cref{sec:classification}. We then present our global convergence results (\cref{thm:ss-convergence,thm:rr-convergence}) for the distorted risks induced by SS and RR for squared loss regression. In the one-dimensional case, we uncover an averaging relationship between the SS and RR optima (\cref{prop:ss-rr-relationship}) which can help RR reduce distortion. We exemplify this averaging relationship with a simple example and extend it to higher dimensions with experiments on synthetic data.

\subsection{Framework: the idea of distorted risks}\label{sec:distorted-risks}
We now formally introduce the notion of a \emph{distorted risk}. Distorted risks are crucial to our analysis, as they encode the interaction between shuffling SGD and BN. We show that these distorted risks $\ellpi$ and $\ellrr$ are respectively induced by certain batch normalized datasets $\Xpibn$ and $\Xrrbn$.

Recall that the network outputs for a batch depend on the entire batch.
The \emph{undistorted} risk we actually want to minimize is the risk which corresponds to full-batch GD. Define the GD features $\Xgdbn \triangleq \bn(X)$, which induces this GD risk:
\[
\ellgd(\Theta) \triangleq \mathcal L(f(\mX;\Theta),\mY)
= \mathcal L(\mW \mGamma \Xgdbn, \mY).
\]

However, during epoch~$k$, SS or RR optimizes a distorted risk dependent on $\pi_k$. To see why, define the SS dataset
\begin{align*}
    \Xpibn \triangleq \bnpi (\mX) &\triangleq \mqty[\bn(\mX_\pi^1) & \cdots & \bn(\mX_\pi^m)] \\
    \mY_\pi &\triangleq \mqty[\mY_\pi^1 & \cdots & \mY_\pi^m],
\end{align*}
for every permutation $\pi \in \sS_n$. Similarly, form the RR dataset $(\Xrrbn, \Yrr) \in \RR^{d \times (n \cdot n!)} \times \RR^{p \times (n \cdot n!)}$ by concatenating the SS datasets $(\Xpibn, \Ypi)$ across all $\pi$.

Crucially, the SS data $\Xpibn$ encodes the distortion due to the interaction between SS with permutation $\pi$ and BN; the RR data $\Xrrbn$ does the same for RR and BN. Indeed, since SS uses fixed $\pi$, it implicitly optimizes the SS distorted risk
\[
\mathcal L_{\pi}(\Theta) \triangleq 
\sum_{j=1}^m \mathcal L(f(\mX_{\pi}^j;\Theta),\mY_{\pi}^j) = \mathcal{L}(\mW\mGamma \ol{\mX}_{\pi}, \mY_{\pi}).
\]

Likewise, by collapsing the epoch update into a noisy ``SGD'' update, we observe that RR over epochs implicitly optimizes the RR distorted risk 
\[
\ellrr(\Theta) \triangleq \frac{1}{n!} \sum_{\pi \in \sS_n} \ellpi(\Theta) = \frac{1}{n!}\mathcal{L}(\mW\mGamma \Xrrbn, \Yrr).
\label{eq:ellrr}
\]

We reiterate that SS and RR distortions originate from using \emph{both} shuffling and batch normalization: shuffling alters the batch-dependent affine transforms that BN applies. With this notation, the connection between SS+BN/RR+BN and ghost BN becomes more evident: one can view the full batch as the batch in ghost BN and the mini-batches as the virtual ghost batches. Moreover, the proofs of \cref{thm:ss-convergence,thm:rr-convergence} demonstrate that ghost BN would witness the same type of distortion as SS+BN/RR+BN.

To aid clarity, we adopt the convention that overlines connote batch normalization with \emph{some} batching, and vice versa. For example, the SS dataset $\Xpibn \triangleq \bnpi(\mX)$ is normalized, while the shuffled dataset $\Xpi = \pi \circ \mX$ is not. 
\subsection{Convergence results for regression}\label{subsec:regression-results}
We now present our main regression results: SS+BN and RR+BN converge to the global optima of their respective distorted risks encoded by the SS dataset $\Xpibn$ and the RR dataset $\Xrrbn$. We require the following rank assumptions.

\begin{assumption}[Full rank assumption]\label{assumption:full-rank-reg}
\,\\ \vspace*{-3ex}
\begin{enumerate}[label=\normalfont{(\alph*)},ref={Assumption~\theassumption(\alph*)}]
\item \label{assumption:full-rank-ss} $\Xpibn \in \RR^{d \times n}$ satisfies $\rank(\Xpibn) \ge d$. In particular, $\sigma_{\min}(\Xpibn\Xpibn^\top) > 0$. 
\item \label{assumption:full-rank-rr}  $\Xrrbn \in \RR^{d \times (n \cdot n!)}$ satisfies $\rank(\Xrrbn) \ge d$. In particular, $\sigma_{\min}(\Xrrbn\Xrrbn^\top) > 0$. 
\end{enumerate}
\end{assumption}

It is natural to ask when \Cref{assumption:full-rank-reg} holds. We demonstrate that the following milder assumption implies it; the assumption states that the feature matrix $\mX$ is drawn from a joint density on matrices in a \emph{potentially non-i.i.d.\ fashion}. 
\begin{assumption}\label{assumption:density}
$\mX$ is drawn from a density with respect to the Lebesgue measure on $\RR^{d \times n}$. 
\end{assumption}
Since $\bn$ centers the mini-batch features, we have $\rank(\Xpibn) \le \min\qty{d, (B-1)\frac{n}{B}}$ and $\rank(\Xrrbn) \le \min\qty{d, (B-1)\binom{n}{B}}$.\footnote{Note that $\Xrrbn$ contains many duplicate batches; only $\binom{n}{B}$ of them are unique, up to permutations of $B$ columns inside a batch.} We now show that if $B > 2$ these upper bounds are achieved almost surely. Thus, we identify reasonable conditions under which \Cref{assumption:full-rank-reg} holds almost surely over the draw of data, irrespective of shuffling. 
\begin{proposition}\label{prop:full-rank}
Assume \Cref{assumption:density} and $B > 2$. Then we have $\rank(\Xpibn) = \min\qty{d, (B-1)\frac{n}{B}}$ and $\rank(\Xrrbn) = \min\qty{d, (B-1)\binom{n}{B}}$ a.s.. Consequently, if $(B-1)\frac{n}{B} \ge d$, \ref{assumption:full-rank-ss} holds a.s. for $\Xpibn$, and if $(B-1)\binom{n}{B} \ge d$, \ref{assumption:full-rank-rr} holds a.s. for $\Xrrbn$. 
\end{proposition}
Although we could have just assumed \Cref{assumption:full-rank-reg}, the nonlinearity introduced by BN makes it nontrivial to identify mild sufficient conditions on the original features to control the rank of SS and RR datasets. Furthermore, controlling the rank of these datasets is crucial to our analysis of GD risk divergence in the classification setting (see \cref{sec:classification}).

Next, we present our main SS convergence result: SS converges for appropriate stepsizes. We defer the proof and explicit convergence rates to \cref{app:ss-proof}.
\begin{theorem}[Convergence of SS]\label{thm:ss-convergence}
Let $f(\cdot; \Theta) = \mW\mGamma \bn(\cdot)$ be a linear+BN network initialized at $\Theta_0^1 = (\mW_0^1,\mGamma_0^1) = (\vzero,\mI)$.  
We train $f$ using SS with permutation $\pi$ and suppose that \ref{assumption:full-rank-ss} holds for this $\pi$. SS uses the following decreasing stepsize, which is well-defined:
\begin{align*}
    \eta_k = \frac{1}{k^\beta} \cdot  \min\Bigl\{&O\Bigl(\frac{1}{\sigma_{\min}(\Xpibn\Xpibn^\top)}\Bigr), \frac{\sqrt{2\beta-1}\poly(\sigma_{\min}(\Xpibn^\top))}{\poly(n, d, \norm{\mY}_F)}\Bigr\},
\end{align*}
where $1/2 < \beta < 1$.
Then the risk $\ellpi(\Theta_0^k)$ converges to the global minimum $\ellpi^*$ as $k \to \infty$.
\end{theorem}
\Cref{thm:ss-convergence} shows that using both SS and BN induces the network to converge to the global optimum of the SS distorted risk instead of the usual GD risk. The proof proceeds by aggregating the epoch-wise gradient updates on the collapsed matrix $\mW\mGamma$. The main difficulty lies in carefully bounding the accumulation of various types of noise. 

We now turn to RR convergence. For the sake of analysis, we make the following compact iterates assumption which is common in the RR literature  \citep{haochen2019random,nagaraj2019sgd,ahn2020sgd,rajput2020closing}.
\begin{assumption}\label{assumption:compact}
For all $(i,k)$, the iterates $\Theta_i^k = (\mW_i^k, \mGamma_i^k)$ satisfy  $\norm{\mW_i^k \mGamma_i^k}_2 \le \Brr$ for some absolute constant $\Brr$.
\end{assumption}

Finally, we can show that RR converges in expectation to the global optimum of the RR distorted risk $\ellrr$. We defer the proof and explicit convergence rates to \Cref{app:rr-proof}.
\begin{theorem}[Convergence of RR]\label{thm:rr-convergence}
Assume \ref{assumption:full-rank-rr} and \Cref{assumption:compact}.  Using the same $f$ and initialization as in \Cref{thm:ss-convergence}. 
we train training $f$ using RR with the following decreasing stepsize, which is well-defined:
\begin{align*}
    \eta_k = \frac{1}{k^\beta} \cdot \min\Bigl\{&O\Bigl(\frac{1}{\sigma_{\min}(\Xrrbn \Xrrbn^\top)}\Bigr), \frac{\sqrt{2\beta-1}}{\poly(n, d, \norm{\mY}_F, \Brr)}\Bigr\},
\end{align*}
where $1/2 < \beta < 1$. Then the risk $\ellrr(\Theta_0^k)$ converges in expectation to the global minimum $\ellrr^*$ as $k \to \infty$.
\end{theorem}

The proof of \cref{thm:rr-convergence} is similar to the SS case; the main subtlety is using \Cref{assumption:compact} to handle expectations.

The main takeaway of \cref{thm:ss-convergence,thm:rr-convergence} is that SS+BN and RR+BN converge to the optima of the SS and RR distorted risks, respectively. These distorted optima may differ from optimum of the GD risk. Moreover, the required stepsize for convergence is usually smaller for SS (where the requirement depends on $\pi$) compared to RR.  

\subsection{RR averages out SS distortion}\label{sec:opt-analysis-reg}
Having shown that the two different algorithms drive the network parameters to global optima of two different distorted risks, it behooves us to study these optima. 
By collapsing the final layers $\mW$ and $\mGamma$ into a single matrix $\mM = \mW\mGamma \in \RR^{p \times d}$, we can study the global optima $\Mpi^*$ and $\Mrr^*$ on the normalized datasets $\Xpibn$ and $\Xrrbn$. These global optima naturally correspond to the global optima of $\ellpi$ and $\ellrr$.
In this section we illustrate how RR can average out SS distortion in the one-dimensional case.

We first relate the SS optima $\Mpi^*$ to the RR optimum $\Mrr^*$. A simple gradient calculation reveals $\Mrr^* = \sum_{\pi} \Ypi \Xpibn^\top (\sum_{\pi} \Xpibn \Xpibn^\top)^{-1}.$ Since $\bn$ enforces the unit variance constraint, $\Xpibn\Xpibn^\top = n$ if $d=1$. Simple algebraic manipulation then implies the following proposition.

\begin{proposition}\label{prop:ss-rr-relationship}
If $d=1$, $\Mrr^* = \frac{1}{n!} \sum_{\pi \in \sS_n} \Mpi^*.$
\end{proposition}

\Cref{prop:ss-rr-relationship} identifies an explicit averaging relationship between RR and SS in the one-dimensional case.
This motivates the following simple construction where RR's averaging behavior removes SS distortion. 

\paragraph{Dataset: SS distorted with constant probability, RR averages out distortion.} We visualize our toy dataset with $16n$ datapoints where $d=p=1$, $B=2$, and $n=3$ in \Cref{fig:toy-dataset-reg}, along with the possible SS optima $\Mpi^*$. The dataset is comprised of four clusters of $4n$ points in the square $[-1, 1]^2$. By vertical symmetry of the clusters and \Cref{prop:ss-rr-relationship}, the \emph{RR and GD optima coincide at zero}. However, SS is distorted away from GD. An anticoncentration calculation shows $\Mpi^* \neq 0$ with probability $1 - O(\frac{1}{\sqrt{n}})$ and $\abs{\Mpi^*} = \Omega(\frac{1}{\sqrt{n}})$ with constant probability. The key insight is linking SS distortion to breaking symmetry in the SS dataset (see \Cref{prop:toy-regression-dataset} for details).

\subsection{Regression experiments}\label{sec:reg-experiments}
For our regression experiments, we used synthetic data with $n = 100$, $B=10$, and $d=10$. For $i \in [n]$, we sampled $\vx_i \sim N(\vzero, \mI_d)$ and generated $y_i = \mM_{\mathrm{true}}\vx_i + \epsilon_i \in \RR$ with $\mM_{\mathrm{true}} \sim U[-1, 1]^d$ and $\epsilon_i \sim N(0, 1)$. We trained the network $\mW\mGamma\bn(\mX)$ using SS and RR with an inverse learning rate schedule. We observed convergence to near optimal values on the SS and RR risks (\Cref{fig:regression-convergence}), which supports the convergence results (\Cref{thm:ss-convergence,thm:rr-convergence}).

We also extended the toy dataset to the synthetic setup described above. As \Cref{fig:normalize_distances} makes apparent, SS is consistently distorted away from the GD optimum, whereas RR averages out this distortion effect. We generated 500 datasets and evaluated the distortion for each one with the normalized distance $d(\mM) \triangleq \frac{\norm{\mM - \Mgd^*}}{\norm{\Mgd*}}$. For SS, we computed the mean $d(\Mpi^*)$ for $1000$ random draws of $\pi$. For RR, we approximated $d(\Mrr^*)$ as follows. We sampled $1000$ fresh random permutations to approximate the RR dataset $\Xrrbn$, which we then used to approximate $\Mrr^*$ (since it is intractable to average over all $n!$ permutations). 

\begin{figure}[!ht]
     \centering
     \begin{subfigure}[b]{0.48\textwidth}
         \centering
         \includegraphics[width=\textwidth]{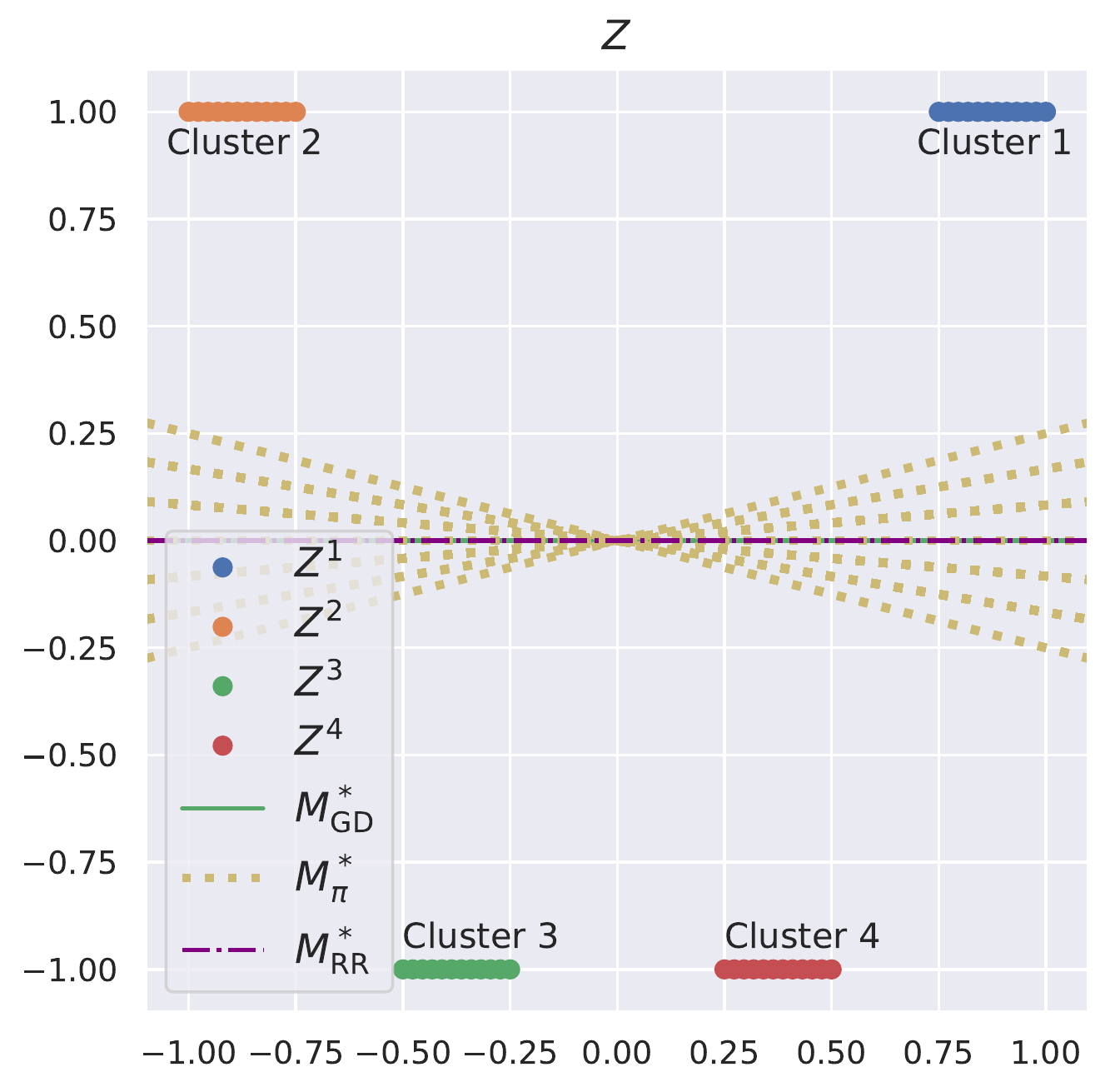}
         \caption{Dataset with 48 datapoints demonstrating distortion of SS optima $\Mpi^*$.} 
         \label{fig:toy-dataset-reg}
     \end{subfigure}
     \hfill
     \begin{subfigure}[b]{0.48\textwidth}
         \centering
         \includegraphics[width=\textwidth]{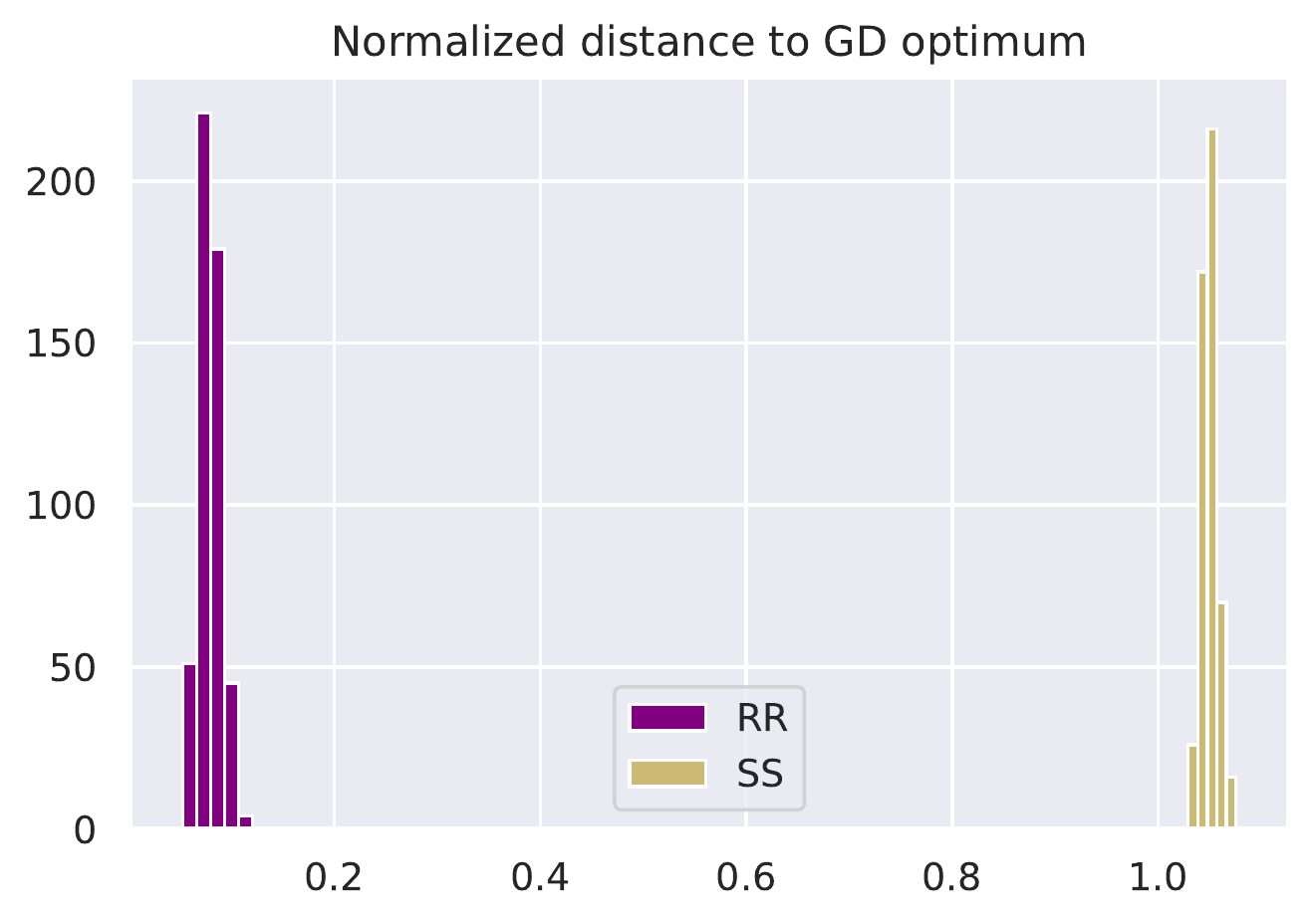}
         \caption{Normalized distance to GD optimum $d(\mM) = \norm{\mM-\Mgd^*}/\norm{\Mgd^*}$.}
         \label{fig:normalize_distances}
     \end{subfigure}
     \hfill
     
     \caption{Top: toy dataset for regression, showing how RR can average out the distortion of SS. Bottom: histogram of distortion of SS and RR optima on synthetic data for $d=10$. The SS optima significantly deviate from the GD optima, whereas the RR optima are relatively close. This supports the intuition that RR can nontrivially smooth out the bias of SS in higher dimensions.}\label{fig:reg-opt-distortion}
\end{figure}
\section{Main classification results: divergence regimes based on distorted risks}\label{sec:classification}
We now turn to analyzing linear+BN binary classifiers $f(\mX; \Theta) = \sgn(\mW\mGamma \bn(\mX))$ trained with the logistic risk. To characterize divergence, we identify salient properties of the distorted risks first introduced in \cref{sec:distorted-risks}. These properties identify regimes where the SS+BN classifier can diverge on the GD risk (\cref{thm:ss-divergence-informal}) yet the RR+BN classifier does not diverge (\cref{thm:rr-divergence-informal}). This motivates the construction of a toy dataset (\cref{example:clf-divergence}) where the optimal SS classifier diverges on the GD risk with constant probability. In \cref{sec:clf-experiments} we extend our results to more realistic networks and datasets, demonstrating that these phenomena are not an artifact of our theoretical setup. Our theoretical results offer some justification for the empirical phenomenon of divergence when SS SGD is combined with BN for classification. 

We briefly remark on why we analyze divergence conditions instead of directional convergence. The main difficulty lies in analyzing SGD instead of GD. One could hope to extend the techniques for directional convergence for homogeneous networks in \citet{lyu2019gradient} to the stochastic setting, but this is outside the scope of our paper. Furthermore, the analyses for deep linear networks such as \citet{ji2020directional} rely on invariants which do not hold for us due to the diagonal $\mGamma$ and the $\bn$ layers for deeper networks. 

Throughout, we use $\vv = (\mW\mGamma)^\top \in \RR^d$ to refer to the vector that determines the decision boundary of our classifier $f$. We remind the reader of the datasets which induce the different distorted risks (\cref{sec:distorted-risks}). Given dataset $\mZ = (\mX, \mY)$, the GD dataset is $\Zgd \triangleq (\Xgdbn, \Ygd) = (\bn(\mX), \mY)$. Similarly define the SS dataset $\Zpi \triangleq (\Xpibn, \Ypi) = (\bnpi(\mX), \pi \circ \mY)$ and the RR dataset $\Zrr \triangleq (\Xrrbn, \Yrr)$ by concatenating $\Zpi$ over all permutations $\pi$. If the labels are clear from context, we occasionally abuse terminology and refer to the features as the dataset. 

\subsection{Analysis of problem structure for classification}\label{subsec:classification-results}
To analyze the optima of the distorted risks, we introduce relevant concepts from \citet{ji2019implicit}. Given a dataset $\mZ = (\mX, \mY) = \qty{(\vx_i, y_i)}_{i=1}^n$, with labels $y_i \in \qty{\pm 1}$, greedily define a \emph{maximal linearly separable subset} $\Sls \triangleq (\Xls, \Yls)$ as follows. Include $(\vx_i, y_i)$ in $\Sls$ if there exists a classifier $\vu_i \in \RR^d$ with $y_i\vu_i^\top \vx_i > 0$ and $y_j\vu_i^\top \vx_j \ge 0$ for all $j$. For reasons that will be clear shortly, denote the complement of $\Sls$ in $\mZ$ by $\Ssc \triangleq (\Xsc, \Ysc)$. 

In particular, there exists a classifier $\vu$ such that: (1) $\Sls$ is perfectly separated by $\vu$ (2) the datapoints $\Xsc$ in $\Ssc$ are orthogonal to $\vu$, so they are on the decision boundary. We can choose $\vu$ to be the max-margin classifier $\vu^{\mathrm{MM}}$ on $\Sls$. The notation $\Ssc$ is chosen because the logistic risk is strongly convex when restricted to bounded subsets of $\Span(\Xsc)$, meaning there is a unique finite minimizer $\vv^{\mathrm{SC}}$ in this subspace. \citet{ji2019implicit} show that linear classifiers trained on the logistic risk with GD are implicitly biased towards the ray $\vv^{\mathrm{SC}} + t \cdot \vu^{\mathrm{MM}}$ for $t > 0$. 

We now identify a salient property of the distorted risks.
\begin{definition}[Separability decomposition]\label{def:separability-decomp}
The \emph{separability decomposition} of dataset $\mZ$ refers to  $\mZ = \Sls \sqcup \Ssc$.
\end{definition}
If $\Sls = \mZ$,  we say $\mZ$ is linearly separable (LS). If both $\Sls$ and $\Ssc$ are nonempty, we say $\mZ$ is \emph{partially linearly separable} (PLS). Finally, if $\Ssc = \mZ$, we slightly abuse terminology and say $\mZ$ is strongly convex (SC). 
\footnote{Here, PLS refers to the ``general case'' discussed in \citet{ji2019implicit}, but we chose to use this alternative terminology because we found the term ``general'' can lead to confusion.}

Because the logistic loss does not always have finite infima, we now introduce the notion of an optimal direction.
\begin{definition}[Optimal direction]\label{def:optimal} Given dataset $\mZ = (\mX, \mY)$, we say a sequence of iterates $\vv(t)$ \emph{infimizes} $\mathcal{L}$ if $\mathcal{L}(\vv(t)^\top \mX, \mY) \to \inf_{\vw \in \RR^d} \mathcal{L}(\vw^\top \mX, \mY)$. We call $\vv \in \RR^d$ an \emph{optimal direction} if 
there exists $\vu \in \RR^d$ such that $\qty{\vu + t\vv}_{t \ge 1}$ infimizes $\mathcal{L}$.\footnote{This definition is catered towards the SC+full rank $\mX$ or PLS/LS case. However, since \cref{prop:full-rank} provides sufficient conditions for full-rank data, this subtlety is unimportant.}
\end{definition}
\Cref{def:separability-decomp} is motivated by the following results which identify how the separability decomposition affects optimal directions. Their proofs are deferred to \cref{sec:characterize-direction}.
\begin{lemma}\label{lemma:infimize} Let $\mZ = \Sls \sqcup \Ssc$. If $\vv$ is an optimal direction for $\mathcal{L}$, then $\vv^\top \vx = 0$ for all $\vx \in \Span(\Xsc)$ and $y_i\vv^\top \vx_i > 0$ for every $(\vx_i, y_i) \in \Sls$. 
\end{lemma}
Combining the above lemma and the definitions yields the following proposition, which characterizes SS and RR divergence using the separability decomposition.  
\begin{proposition}\label{prop:divergence-possibility}
Suppose \ref{assumption:full-rank-ss} holds, the iterates $\vpi(t)$ infimize $\ellpi$, and their projections onto $\Span(\Xpisc)^{\perp}$ converge in direction to some optimal direction $\vpi^*$ for $\ellpi$. Then the GD risk $\ellgd$ diverges if and only if $\Zpi$ is PLS or LS and there exists some $(\vx_i, y_i) \in \Zgd$ such that $y_i\vpi^{*\top} \vx_i < 0$. The analogous statement holds true for $\Zrr$ under \ref{assumption:full-rank-rr}. Furthermore, the ``if'' part holds true for SS and RR without \Cref{assumption:full-rank-reg}.
\end{proposition}
In particular, \Cref{prop:divergence-possibility} implies that if the RR dataset is SC and rank $d$, the GD risk does not diverge. Moreover, it naturally leads to the question of understanding ranks and separability decompositions of the SS and RR datasets; the former question is already answered by \cref{prop:full-rank}. 

To analyze the separability decomposition with high probability or almost surely, we assume the labels are balanced.
\begin{assumption}[Balanced classes]\label{assumption:balanced} The data $\mZ$ either has
\,\\\vspace{-3.5ex}
\begin{enumerate}[label=\normalfont(\alph*), ref={Assumption~\theassumption(\alph*)}]
    \item \label{assumption:ss-balanced} an equal number of positive and negative examples; or 
    \item \label{assumption:rr-balanced} at least $B$ positive and $B$ negative examples.
\end{enumerate}
\end{assumption}

Finally, we informally state our main classification: SS+BN can diverge in some regimes (see \cref{prop:ss-structure} for details).
\begin{theorem}[SS+BN can diverge (informal))]\label{thm:ss-divergence-informal}
Assume \Cref{assumption:density},  \ref{assumption:ss-balanced}, and $B>2$. If $d \le (B-1)\frac{n}{B}$, SS can diverge if $B = \Omega(\log n)$ and $\Zgd$'s separability decomposition can change with small perturbations.  Otherwise, SS can diverge regardless of the batch size and the separability decomposition of $\Zgd$. 
\end{theorem}
Whereas \cref{thm:ss-divergence-informal} establishes regimes where SS+BN can diverge, we can show that RR+BN prevents divergence in a much larger regime (see \cref{prop:rr-structure} for details).
\begin{theorem}[RR+BN does not diverge (informal)]\label{thm:rr-divergence-informal}
Assume \Cref{assumption:density},  \ref{assumption:rr-balanced}, and $B>2$. If $d \le (B-1)\binom{n}{B}$, RR does not diverge almost surely.
\end{theorem}
\Cref{thm:ss-divergence-informal} implies that one cannot prevent SS divergence by simply increasing the batch size $B$; it is also necessary for the GD dataset to be ``robustly'' LS or SC. Moreover, as soon as $d > (B-1)\frac{n}{B}$, SS can diverge. In stark contrast, \Cref{thm:rr-divergence-informal} establishes that even for small $B$, RR is \emph{almost surely} robust to divergence as long as $d \le (B-1)\binom{n}{B}$. Although our theorems do not prove that SS+BN \emph{necessarily} diverges, they offer some theoretical explanation for why SS+BN appears to be less stable than RR+BN for classification.

\subsection{RR prevents divergence while SS diverges}\label{example:clf-divergence}
We present a toy dataset where SS drastically distorts the optimal direction, leading to divergence with constant probability. Meanwhile, RR does not diverge on this dataset. We use $d=B=2$ to simplify the construction.\footnote{Since $B=2$, there is no contradiction with \cref{thm:ss-divergence-informal}.}

\begin{figure}[!ht]
     \centering
     \begin{subfigure}[b]{0.48\textwidth}
         \centering
         \includegraphics[width=\textwidth]{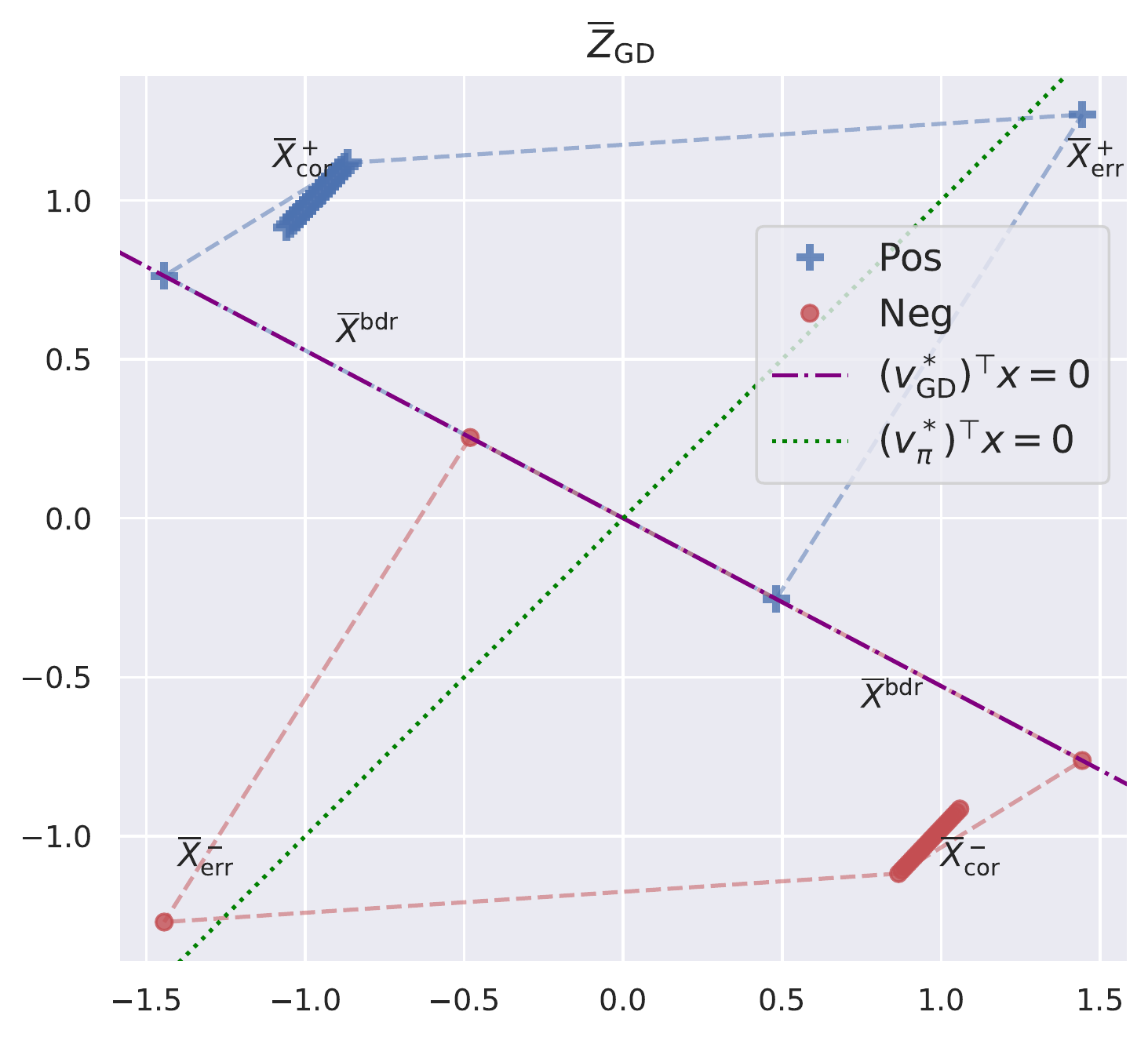}
         \caption{Toy classification dataset showing divergence of SS with constant probability.}
         \label{fig:toy-dataset-clf}
     \end{subfigure}
     \hfill
     \begin{subfigure}[b]{0.48\textwidth}
         \centering
         \includegraphics[width=\textwidth]{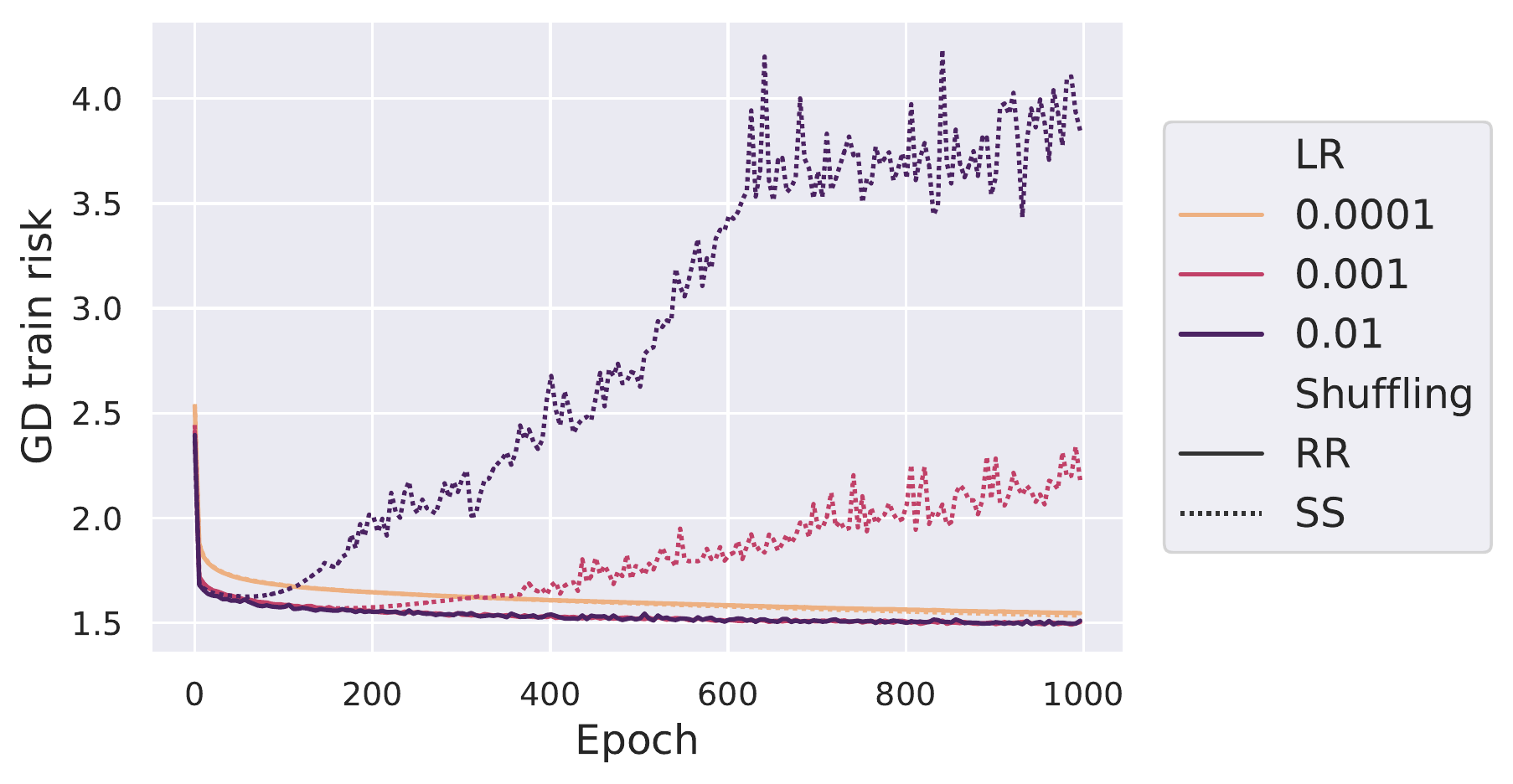}
         \caption{3 layer linear+BN networks trained with varying stepsizes.}
         \label{fig:lnn_divergence}
     \end{subfigure}
    \caption{Left: Toy dataset demonstrating divergence of GD risk with constant probability. The dashed lines trace out the convex hulls of the positive and negative points. Right: divergence of GD risk for a variety of stepsizes on CIFAR10. Note that there was eventually a separation for $\eta = 10^{-4}$ (see \Cref{fig:eventual_separation}).}
\end{figure}

\paragraph{Dataset: SS diverges with contsant probability; RR does not.}
We describe our construction (\Cref{fig:toy-dataset-clf}) at a high level; see \cref{prop:toy-classification-dataset} for details. The GD dataset is PLS with unique optimal direction $\vgd^*$ (its decision boundary is the purple dash-dotted line). Moreover, with constant probability the SS dataset is PLS with unique optimal direction $\vpi^*$ (green dotted line) distorted away from $\vgd^*$. Also, $\vpi^*$ misclassifies points in the GD dataset ($\ol{\mX}^+_{\mathrm{err}}$ and $\ol{\mX}^-_{\mathrm{err}}$). Under the additional assumptions in \cref{prop:divergence-possibility}, the GD risk diverges. Finally, since the RR dataset is SC and rank $d$, RR does not diverge on the GD dataset.

\begin{figure*}[t]
     \centering
     \begin{subfigure}[b]{0.48\textwidth}
         \centering
         \includegraphics[width=\textwidth]{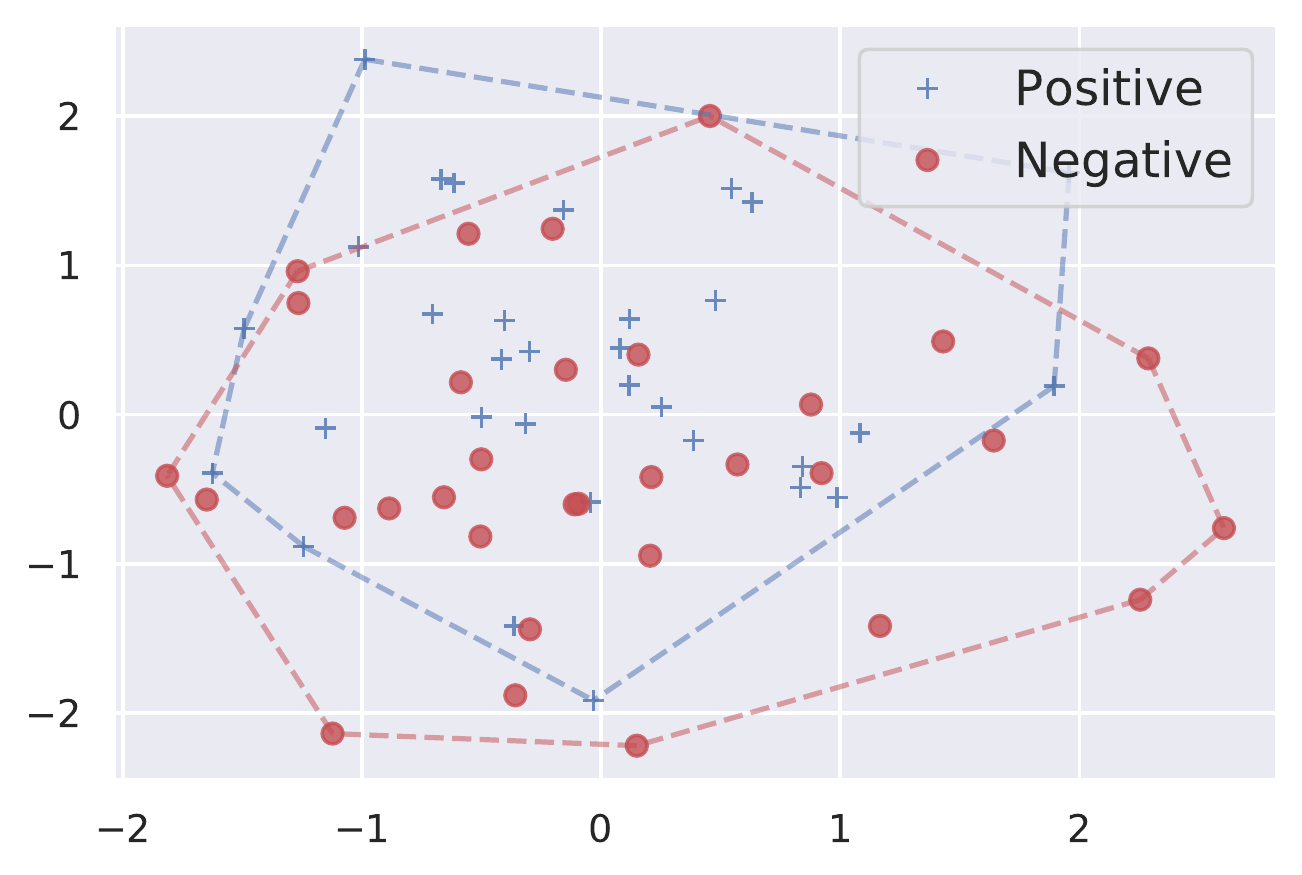}
         \caption{$\Zgd = (\bn(\mA_0^0 \mX), \mY)$ at initialization.}
         \label{fig:xgd-initial}
     \end{subfigure}
     \hfill
     \begin{subfigure}[b]{0.48\textwidth}
         \centering
         \includegraphics[width=\textwidth]{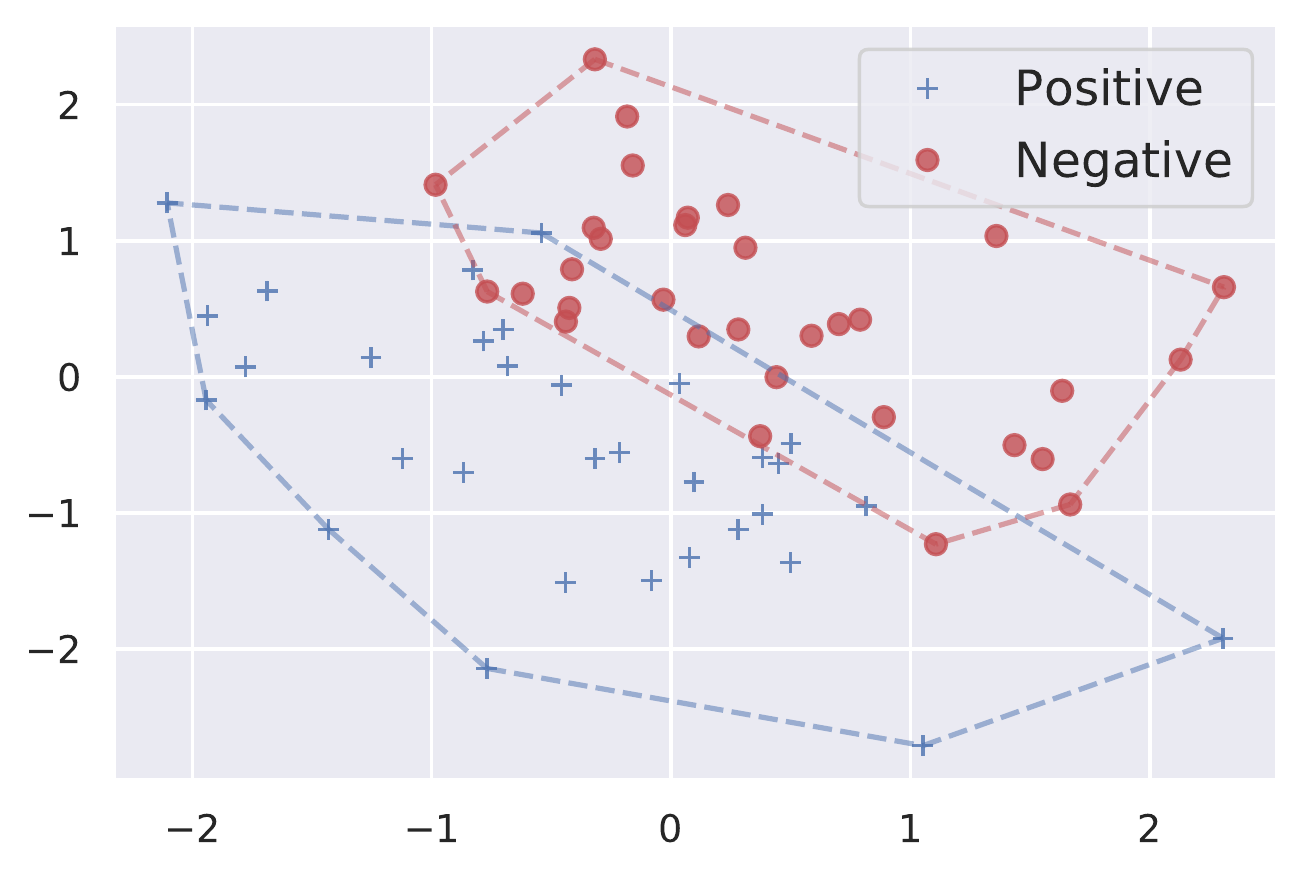}
         \caption{$\Zgd = (\bn(\mA_0^T \mX),\mY)$ after training.}
         \label{fig:xgd-final}
     \end{subfigure} 
 \hfill \\
     \begin{subfigure}[b]{0.48\textwidth}
         \centering
         \includegraphics[width=\textwidth]{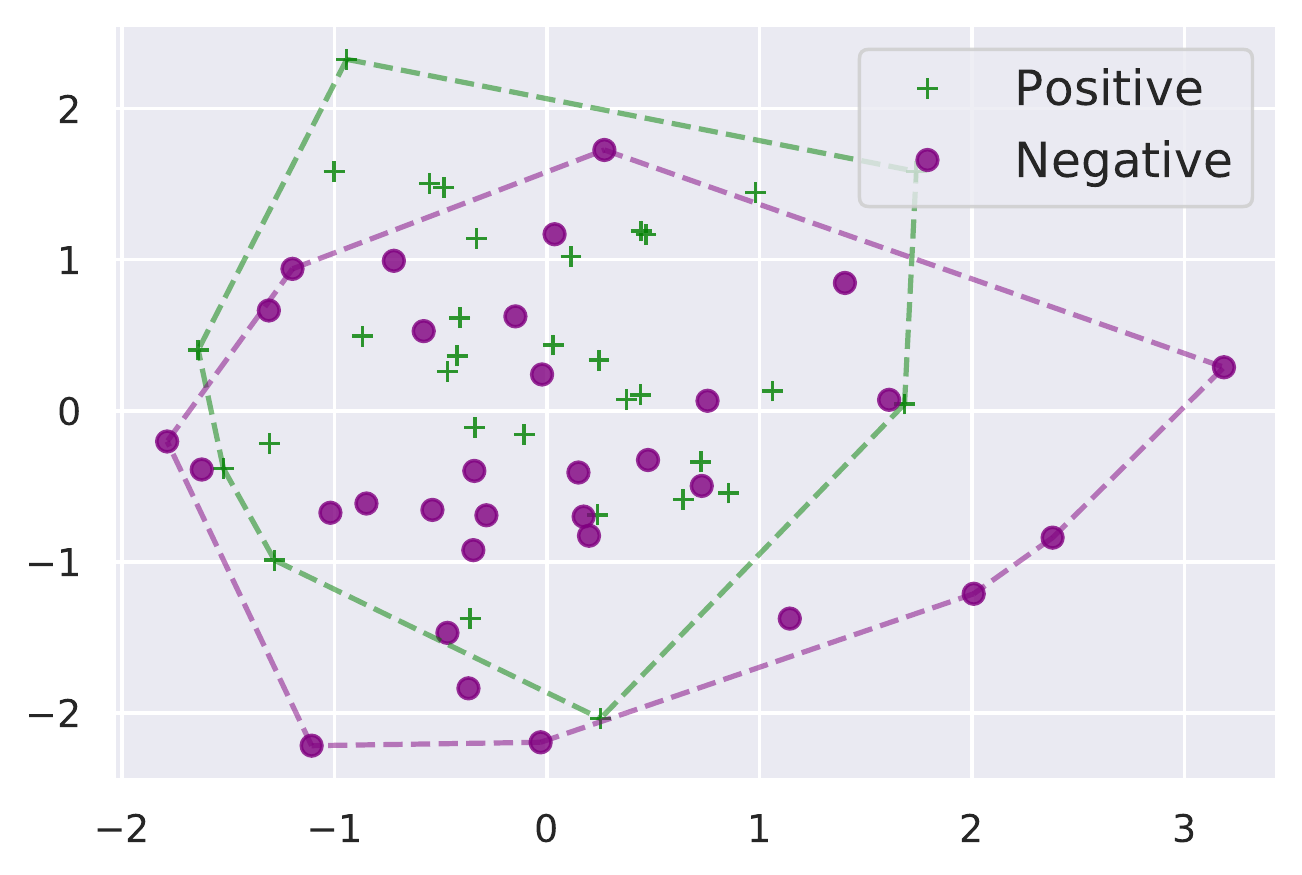}
         \caption{$\Zpi = (\bnpi(\mA_0^0 \mX),\Ypi)$ at initialization.}
         \label{fig:xss-initial}
     \end{subfigure}
     \hfill
     \begin{subfigure}[b]{0.48\textwidth}
         \centering
        \includegraphics[width=\textwidth]{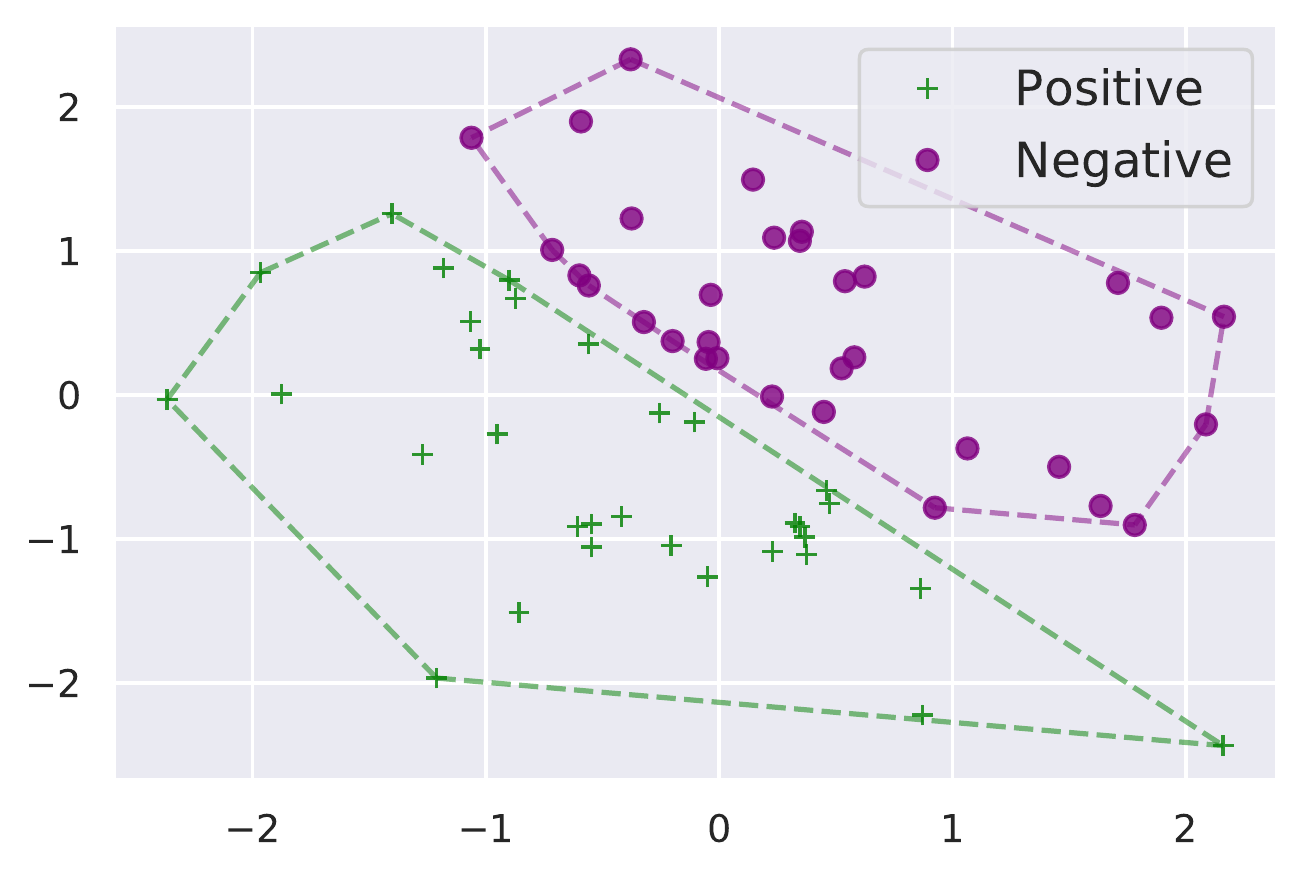}
         \caption{$\Zpi = (\bnpi(\mA_0^T \mX), \Ypi)$ after training.}
         \label{fig:xss-final}
     \end{subfigure}
     \caption{Snapshots of GD dataset $\Zgd$ and SS dataset $\Zpi$ before and after running SS for $T = 10^4$ epochs with $32$ positive and negative synthetic examples. While the GD dataset remains SC, the SS dataset become LS. Here $B = 16$, $\eta = 10^{-2}$, and $\eps = 10^{-5}$ for BN.}\label{fig:conceptual}
\end{figure*}

\subsection{Experiments on linear and nonlinear networks}\label{sec:clf-experiments}
We now verify our theoretical classification results on linear+BN and extend them to nonlinear networks on a variety of real-world datasets. This demonstrates that the separation between SS, RR, and GD is relevant in realistic settings and not merely an artifact of the linear setting. We refer to the linear+BN network $\mW\mGamma\bn(\mX)$ as 1-layer linear network, and also consider deeper linear networks with tunable parameters inside BN layers. We observe strikingly different training behaviors in the shallow and deep linear networks. The networks are formally defined in \cref{app:experiments}; see \url{https://github.com/davidxwu/sgd-batchnorm-icml} for the experiment code.

\begin{figure*}[!ht]
     \centering
    \includegraphics[width=0.95\textwidth]{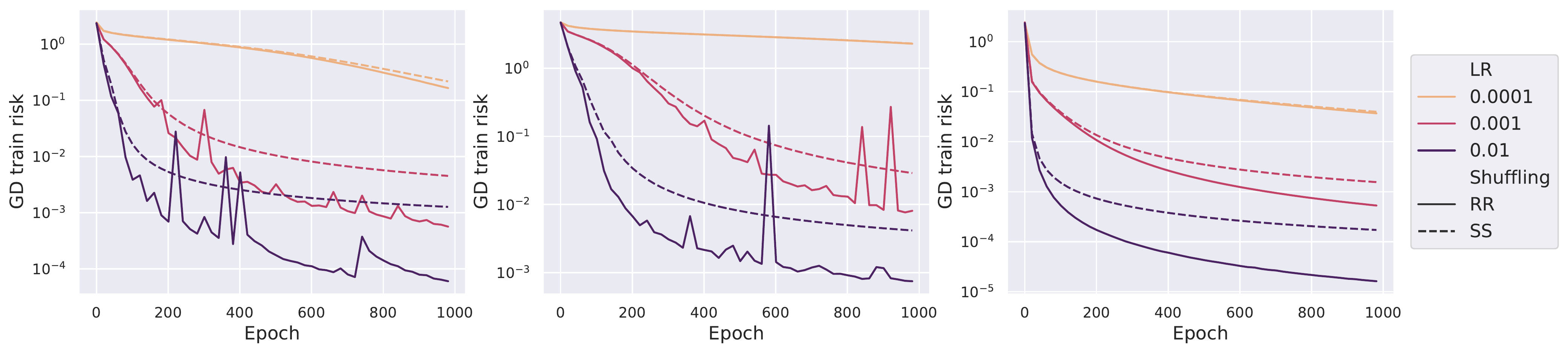}
     \caption{3 layer ReLU+BN MLP on (left to right): CIFAR10, CIFAR100, and MNIST. Note the slower convergence for SS versus RR.}
     \label{fig:fc_slow}
     \hfill
\end{figure*}
\begin{figure*}[!ht]
     \centering     \includegraphics[width=0.95\textwidth]{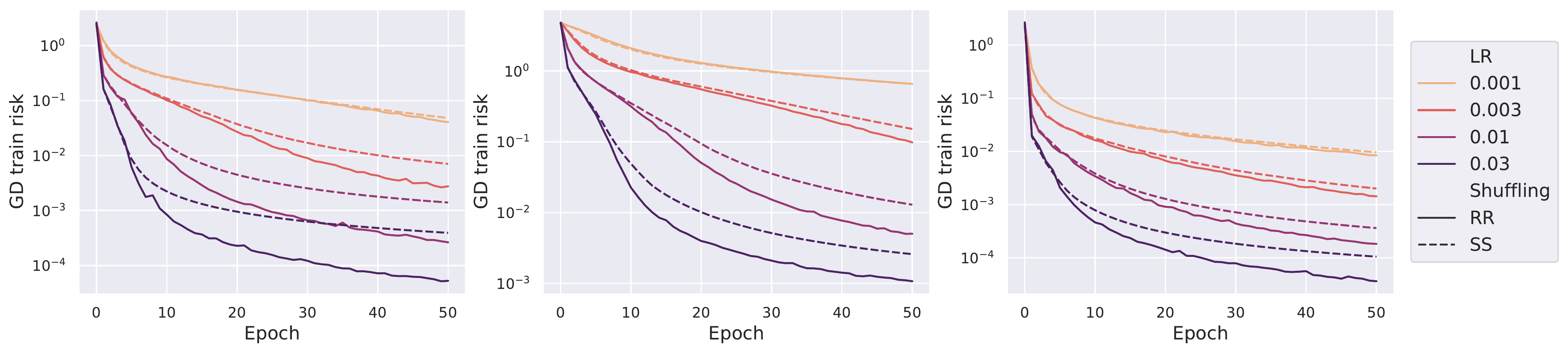}
     \caption{ResNet18 finetuned on (left to right): CIFAR10, CIFAR100, and MNIST. Note the slower convergence for SS versus RR across datasets. For the smallest learning rate $\eta = 10^{-3}$, we observed a separation after 200 epochs.}
     \label{fig:resnet_slow}
\end{figure*}
As a motivating example, we ran an experiment on synthetic data (\cref{fig:conceptual}) with the 2-layer linear network $f(\mX) = \mW\mGamma\bn(\mA\mX)$. Note that the tunable matrix $\mA$ acts before BN.
Intriguingly, we observe that the SS dataset with features $\Xpibn = \bnpi(\mA\mX)$ is SC \emph{at initialization}, but updating $\mA$ with SS makes it LS \emph{after training}. Moreover, the batch size is large relative to $n$, so this dataset satisfies the necessary conditions for divergence in \Cref{prop:divergence-possibility,thm:ss-divergence-informal}. 

More specifically, \Cref{fig:xgd-initial,fig:xss-initial} plot the 2-dimensional GD and SS datasets, respectively, which are SC at initialization. However, after training with SS, we can see from \Cref{fig:xgd-final,fig:xss-final} that SS updates $\mA$ to make the SS dataset LS, whereas the GD dataset stays SC. Hence, by \Cref{prop:divergence-possibility}, the GD risk diverges. This example partially explains the discrepancy in training behavior between the 1-layer and deeper networks. Indeed, whereas the 1-layer architecture has static $\Zpi$, the deeper networks have evolving weights inside BN which can push $\Zpi$ to be LS/PLS.

To exhibit the above divergence on real data, we conducted experiments on the CIFAR10. Using SS and RR, we trained linear+BN networks of depths up to 3 for $T = 10^3$ epochs using stepsize $\eta=10^{-2}$, batch size $B=128$, and 512 hidden units per layer (see \Cref{app:experiments} for precise details).

As depicted in \Cref{fig:3l_lnn_training_loss}, we consistently observed SS divergence for the deeper networks (see \Cref{fig:linear_separability} for more evidence of divergence). As predicted by \cref{thm:rr-divergence-informal}, RR did not exhibit divergence behavior.  These phenomena persisted despite ablating the learning rate in $\qty{0.01, 0.001, 0.0001}$, momentum in $\qty{0, 0.9, 0.99}$, and batch size in $\qty{32, 64, 128}$. The learning rate ablation is shown in \Cref{fig:lnn_divergence}; see \Cref{app:experiments} for the rest.

For the nonlinear experiments, we extended to the CIFAR10, MNIST, and CIFAR100 datasets. We used SS and RR to train 3-layer 512 hidden unit MLPs with BN and ReLU activation for $T=10^3$ epochs, and also to finetune pretrained ResNet18 for $T = 50$ epochs. We consistently observed that in the final stages of training (i.e., relatively small training risk), SS trained slower than RR across all of the datasets, even after tuning the learning rate (see \Cref{fig:fc_slow,fig:resnet_slow}). 
\section{Conclusion}
This paper established that training BN networks with SS can lead to undesirable training behavior, including slower convergence or even divergence of the GD risk. However, RR provably mitigates this divergence behavior, and experimental evidence suggests that using RR usually converges faster than SS. This separation in training behavior between SS, RR, and GD is because data shuffling directly affects how BN operates on mini-batches. Our theoretical results establish a separation for the special case where BN is applied to the input features. The more general and realistic case where BN is applied to dynamically evolving layers is left as an important direction for future work. We also observed in preliminary experiments that a similar separation manifested for generalization, and we hope that adopting a similar perspective will prove fruitful in pursuing this direction. One concrete path towards studying the implicit bias of shuffling SGD and BN with dynamically evolving layers is to combine our techniques with those of \citet{cao2023implicit}. We remark that similar surprising phenomena may arise when using other design choices that are implemented in a mini-batch fashion such as mixup \citep{zhang2017mixup} and Sharpness-Aware Minimization (SAM) \citep{foret2020sharpness}. For these reasons, we generally recommend that practitioners use RR instead of SS. Further future directions include establishing directional convergence for homogeneous classifiers trained with shuffling SGD and theoretically understanding conditions under which deeper networks diverge faster. 

\section*{Acknowledgements}
Part of the work was done while DW was an undergraduate at MIT. DW acknowledges support from NSF Graduate Research Fellowship DGE-2146752. CY acknowledges support by Institute of Information \& communications Technology Planning \& evaluation (IITP) grant (No.\ 2019-0-00075, Artificial Intelligence Graduate School Program (KAIST)) funded by the Korea government (MSIT). CY is also supported by the National Research Foundation of Korea (NRF) grants (No.\ NRF-2019R1A5A1028324, RS-2023-00211352) funded by the Korea government (MSIT). CY acknowledges support from a grant funded by Samsung Electronics Co., Ltd. SS acknowledge support from an NSF CAREER grant (1846088), and NSF CCF-2112665 (TILOS AI Research Institute).
DW appreciates helpful discussions with Xiang Cheng, Sidhanth Mohanty, Erik Jenner, Louis Golowich, Sam Gunn, and Thiago Bergamaschi.

\bibliography{bib}
\bibliographystyle{plainnat}

\appendix

\section{Proofs for regression results}
In this appendix, we provide the full details for the proof of convergence for SS and RR in the regression case.

\paragraph{Additional notation.}
We introduce some additional notation which we will use throughout the proof of \Cref{thm:ss-convergence,thm:rr-convergence}. 
For a matrix $\mA$, we use $\mA_{i,:}$ and $\mA_{:,j}$ to denote the $i$th row and $j$th column of $\mA$, respectively. We also use $A_{i,j}$ to denote the $(i,j)$th entry of $\mA$.
The Hadamard product of two matrices $\mA, \mB \in \RR^{m \times n}$ is denoted by $\mA \odot \mB$, with $(\mA \odot \mB)_{i,j} = A_{i,j}B_{i,j}$. The diagonal operator $\diag: \RR^{m \times m} \to \RR^{m \times m}$ is defined by $\diag(\mA) = \mI \odot \mA$. We denote the Frobenius inner product $\ev{\mA, \mB}_F = \sum_{i,j} A_{i,j}B_{i,j}$ and its induced norm by $\norm{\mA}_F$. 

Also recall from \Cref{sec:setup} that when $\Theta$ is optimized with SS or RR, the $i$th iterate on the $k$th epoch is denoted by $\Theta_i^k$. For simplicity, we will often say the $(i,k)$th iterate to refer to $\Theta_i^k$. Denote the collapsed parameter matrix defined in \Cref{sec:regression} by $\mM \triangleq \mW\mGamma$. We will abuse notation and sometimes denote the $(i,k)$th iterate by $\mM_i^k \triangleq \mW_i^k\mGamma_i^k$. 

Recall that the mini-batch risk used for updating the $(i,k)$th iterate of SS or RR is given by
$\loss(f(\Xpi^{i+1};\Theta_i^k),\Ypi^{i+1}) = \norm{\Ypi^{i+1} - \mW_i^k \mGamma_i^k \bn(\Xpi^{i+1})}_F^2$ where $\pi$ denotes the permutation chosen for the $k$th epoch and $\mX_\pi^j \in \RR^{d \times B}$ and $\mY_\pi^j \in \RR^{p \times B}$ consist of the $(jB-B+1, \ldots, jB)$th columns of $\pi \circ \mX$ and $\pi \circ \mY$, respectively. Since this notation is a bit lengthy, we simplify it to $\loss(\Xpi^{j}; \Theta) \triangleq \loss(f(\Xpi^{j};\Theta),\Ypi^{j})$ for any $j \in [m]$.
Here, we can also view the mini-batch risk as a function of $\mM = \mW\mGamma$, so we will sometimes abuse notation and write 
\begin{align*}
\loss(\Xpi^{j}; \mM) &\triangleq \norm{\Ypi^{j} - \mM \bn(\Xpi^{j})}_F^2,\\
\grad_{\mM} \loss(\Xpi^{j}; \mM) &\triangleq -(\Ypi^{j} - \mM \bn(\Xpi^{j}))\bn(\Xpi^{j})^\top.
\end{align*}

For SS, we work with a fixed permutation $\pi \in \sS_n$ and input dataset $(\mX, \mY)$. Recall that we defined $\Xpibn \triangleq \bnpi(\mX)$ from \Cref{sec:regression}, i.e., the column-wise concatenation of all batches after batch normalization: $\Xpibn = [\bn(\Xpi^1)~\cdots~\bn(\Xpi^m)]$. When the context of parameters $\Theta = (\mW, \mGamma)$ and permutation $\pi \in \sS_n$ chosen by SS are clear, we denote the collection of outputs over the dataset by $\Yhatpi \triangleq \mW\mGamma\Xpibn$. 
Also recall that the distorted SS risk $\ellpi(\Theta)$ we set out to optimize is defined to be $\ellpi(\Theta) = \ellpi(\mW,\mGamma) = \norm{\Ypi - \mW \mGamma \Xpibn}_F^2$. With $\mM \triangleq \mW\mGamma$, we also abuse notation and write 
\begin{align*}
\ellpi(\mM) &\triangleq \norm{\Ypi - \mM \Xpibn}_F^2,\\
\grad_{\mM} \ellpi(\mM) &\triangleq -(\Ypi - \mM \Xpibn)\Xpibn^\top.
\end{align*}

We will use big--$O$ notation throughout to simplify the presentation of the proofs. When we write $O(\eta_k^t)$ for some exponent $t \ge 1$, we hide constants that depend on $m$, $\norm{\Xpibn}_F$, and various absolute constants defined explicitly below. These constants have at most polynomial dependence on these parameters and absolute constants.

\subsection{Proof of convergence for SS}\label{app:ss-proof}
Let us first prove \Cref{thm:ss-convergence}. 
First, we draw the reader's attention to some standard properties in optimization theory that allow us to prove global convergence. We then sketch out the proof in \Cref{sec:ss-proof-sketch} and flesh out the details in subsequent sections.

\subsubsection{Optimization properties}
It is profitable to keep in mind the general idea behind proving global convergence of SGD for a function $\mathcal L(\Theta)$, which has been exploited in \citet*{ahn2020sgd,zhou2017characterization,nguyen2021unified}. 
The following two properties of the optimization problem are critical in such approaches:

\begin{property}[Smoothness]\label{def:smoothness}
$G$-smoothness of $\mathcal L$, i.e., the gradients of $\mathcal L$ are $G$-Lipschitz. In particular, it implies the following two standard properties:
    \begin{enumerate}[label=(\roman*)]
        \item $\mathcal L(\Theta) \le \mathcal L(\Theta') + \ev{\grad_\Theta \mathcal L(\Theta'), \Theta-\Theta'} + \frac{G}{2}\norm{\Theta' - \Theta}^2$ for all $\Theta, \Theta'$ in the domain of $\mathcal L$. \label{prop:smoothness}
        \item The Hessian $\mH = \grad^2_\Theta \mathcal L(\Theta)$ satisfies $\norm{\mH}_2 \le G$ for all $\Theta$ in the domain of $\mathcal L$. \label{prop:hessian}
    \end{enumerate}
\end{property}

\begin{property}[P\L{} condition]\label{def:pl}
The loss function $\mathcal L$ satisfies the $\alpha$-Polyak-\L{}ojasiewicz condition, i.e., $\norm{\grad \mathcal L(\Theta)}^2 \ge 2\alpha(\mathcal L(\Theta) - \mathcal L^*)$ for all $\Theta$ in the domain of $\mathcal L$.
\end{property}

In our case, we can use global smoothness and strong convexity (which implies the P\L{} condition) of $\ellpi$ with respect to $\mM = \mW\mGamma$, but these global properties do not hold with respec to our optimization variables $\Theta = (\mW, \mGamma)$. Importantly, unlike the analyses of \citet*{ahn2020sgd,nguyen2021unified}, we cannot directly leverage the global smoothness and strong convexity as is, because we do not directly perform gradient updates on $\mM$. Instead, we effectively use a ``dynamic'' P\L{} condition which depends on $\mGamma$. The subtlety in the analysis is to show that such behavior can be controlled to ensure convergence in the end. 

Finally, a third property --- which is often exploited to prove convergence results for linear neural networks --- is the notion of an (approximate) invariance property satisfied by the layers of the neural network. Indeed, in the continuous time case, i.e., when we minimize $\ellpi(\Theta(t))$ with gradient flow $\dot\Theta(t) = -\grad_\Theta \ellpi(\Theta(t))$, such an invariance can be directly shown by the differential equations, see \citet*{wu2019global} for instance. To that end, define the following quantity 
\begin{equation}
\mD \triangleq \mI + \diag(\mW^\top \mW - \mGamma^2),
\label{eq:invariance_mat}
\end{equation}
which we refer to as the invariance matrix. For each iterate $\Theta_i^k$ of SS, the corresponding $\mD_i^k$ can also be naturally defined.
In gradient flow, $\mD(t)$ actually remains invariant with time $t \in [0, \infty)$. We quickly prove this property here, and later prove that an approximate version holds in the discrete and stochastic case, although the bounds are messier.  
\begin{fact}\label{fact:invariances}
In the gradient flow formulation, we have $\dv{}{t} \mD(t) = \vzero$. Moreover, in both the gradient flow and discrete time formulation, we have 
\begin{equation}
\diag(\mW^\top \grad_\mW \ellpi) = (\grad_{\mGamma} \ellpi) \mGamma. \label{eq:gradIdentity}
\end{equation}
\end{fact}
\begin{proof}
For the proof, we write out the (full) gradients of $\ellpi$ with respect to $\mW$ and $\mGamma$ for reference:
\begin{align}
    \grad_\mW \ellpi &= -(\Ypi - \Yhatpi)\Xpibn^\top \mGamma, \label{eq:gradW}\\
    \grad_{\mGamma} \ellpi &= -\diag(\mW^\top(\Ypi-\Yhatpi)\Xpibn^\top). \label{eq:gradG}
\end{align}

A direct calculation shows that $\diag(\mW^\top \grad_\mW \ellpi) = (\grad_{\mGamma} \ellpi) \mGamma$. Due to the gradient flow formulation $\dot\Theta(t) = -\grad_\Theta \ellpi(\Theta(t))$ we have $\dv{}{t} \mW(t) = -\grad_\mW \ellpi$ and $\dv{}{t} \mGamma(t) = -\grad_{\mGamma} \ellpi$, so it follows from \Cref{eq:gradIdentity} that $\dv{}{t} \mD(t) = 0$. 
\end{proof}

We now formally state the smoothness and P\L{} guarantees for our setup.

\begin{lemma}[Smoothness with respect to $\mM$]\label{lemma:smoothness-M}
The SS risk $\ellpi$ is $\Gpi$--smooth with respect to $\mM = \mW\mGamma$, where 
\[\Gpi \triangleq \norm{\Xpibn}_2^2.\]
\end{lemma}
\begin{proof}
We directly check the Lipschitz gradient condition. Indeed, we have
\begin{align*}
    &\,\norm{\grad_\mM \ellpi(\mM) - \grad_\mM \ellpi(\mM')}_2  \\
    &\, =\norm{(\Ypi - \mM \Xpibn)\Xpibn^\top - (\Ypi - \mM'\Xpibn)\Xpibn^\top}_2 \\
    &\, =\norm{(\mM - \mM')\Xpibn\Xpibn^\top}_2 \le \norm{\Xpibn}_2^2 \norm{\mM - \mM'}_2,
\end{align*}

Note that the same inequality holds (with the same value of $\Gpi$) if we instead used the Frobenius norm, due to the fact that $\norm{\mA\mB}_F \le \norm{\mB}_2\norm{\mA}_F$ in the last line.
\end{proof} 
\begin{lemma}[Strong convexity with respect to $\mM$]\label{lemma:sc-M}
Under \ref{assumption:full-rank-ss}, SS risk $\ellpi$ is $\alphapi$--strongly convex with respect to $\mM = \mW\mGamma$, where 
\[\alphapi \triangleq \sigma_{\min}(\Xpibn \Xpibn^\top).\]
Hence, $\ellpi$ is also $\alphapi$--P\L{} with respect to $\mM$.
\end{lemma}
\begin{proof}
Take the Hessian of $\ellpi(\mM)$ with respect to the vectorized version $\vec(\mM)$ of $\mM$ to obtain $\grad_{\vec(\mM)}^2 \ellpi(\mM) = \Xpibn \Xpibn^\top \otimes \mI_p$, where $\otimes$ denotes the Kronecker product. Then evidently $\grad_{\vec(\mM)}^2 \ellpi(\mM) \succeq \sigma_{\min}(\Xpibn\Xpibn^\top)\mI_p$. Owing to \ref{assumption:full-rank-ss}, this proves the claim.
\end{proof} 
\subsubsection{Proof sketch of convergence}\label{sec:ss-proof-sketch}

\begin{proof}[Proof sketch of \Cref{thm:ss-convergence}]
The high level idea is this: we want to prove that $\ellpi(\mM_0^k) \to \ellpi^*$ as $k \to \infty$. However, we will instead show the much stronger statement that $\ellpi(\mM_i^k) \to \ellpi^*$ for all $i \in [m]$. Our high level approach is heavily inspired by the proof strategies in \citet{wu2019global,ahn2020sgd}. Indeed, many of the technical lemmas in \cref{sec:norm-bounds} are analogous to ones proved in \citet{wu2019global}, and the motivation for unrolling shuffling mini-batch updates to an epoch update with additional noise comes from \citet{ahn2020sgd}.

As a necessary ingredient of the proof, we will demonstrate that for sufficiently small chosen $\eta_k$, we have an update equation that roughly looks like (modulo constants and noise terms) 
\begin{equation}\label{eq:one-step-iterates}
     \ellpi(\mM_i^{k+1}) - \ellpi^* \lesssim (1-\eta_k)(\ellpi(\mM_i^{k}) - \ellpi^*) + O(\eta_k^2) \quad\quad \text{ for all } 0 \le i \le m-1.
\end{equation} 
\begin{remark}
Note that it is \emph{not} necessarily the case that
\[
 \ellpi(\mM_{i+1}^k) - \ellpi^* \lesssim (1-\eta_k)(\ellpi(\mM_i^{k}) - \ellpi^*) + O(\eta_k^2)
\]
That is, the SS excess risk $\ellpi$ does not necessarily ``decrease'' from one iterate to the next; however, we can instead guarantee that the per-epoch progress bound (\Cref{eq:one-step-iterates}) holds for any fixed iteration index $i \in [m]$ \emph{after every epoch}.
\end{remark}

We impose an ordering relation on pairs $(a,b)$ in the natural way: we say $(a, b) \le (i, k)$ if $k = b$ and $a \le i$, or if $b < k$. This is just tracking whether the iteration index $(a,b)$ (the $a$th iterate of the $b$th epoch) is seen before the iterate $(i, k)$. 
To complete the induction on an iterate $(i, k+1)$ we need three inductive hypotheses $L[a, b]$, $D[a, b]$, and $R[a, b]$ to hold for all $(a, b) < (i,k+1)$. We define them formally below.
\begin{hypothesis}[Loss stays bounded by an absolute constant]\label{hyp:loss-bounded}
For all $a, b$ satisfying $0 \le a \le m-1$ and $b \ge 1$, the inductive property $L[a,b]$ states $\ellpi(\Theta_{a}^{b}) \le C_{L}$, for some appropriately chosen absolute constant $C_L$. 

In particular, we can set $C_L \triangleq \max\qty{\ellpi(\Theta_t^1): 0 \le t \le m-1}$. Since we only look at the loss values for the first epoch, $C_L$ is indeed an absolute constant depending on $\pi$. 
\end{hypothesis}
\begin{hypothesis}[Loss satisfies one-epoch inequality]\label{hyp:risk-update}
For all $a, b$ satisfying $0 \le a \le m-1$ and $b > 1$, the inductive property $R[a,b]$ states that 
\[
\ellpi(\mM_a^{b}) - \ellpi^* \le \qty(1 -\frac{\alphapi \eta_k}{2}) (\ellpi(\mM_a^{b-1}) - \ellpi^*) + O(\eta_k^2),
\]
where the constant hidden in the $O(\eta_k^2)$ does not depend on $k$.
\end{hypothesis}
\begin{hypothesis}[Approximate invariances hold]\label{hyp:approx-invariances}
For all $a, b$ satisfying $0 \le a \le m-1$ and $b \ge 1$, the inductive property $D[a,b]$ states that 
\[
\norm{\mD_a^b}_2 \le 
\begin{cases}
C_D \sum_{t=1}^{b-1} \eta_t^2 \le \frac{1}{2} & \text{ if $a = 0$,}\\
C_D \sum_{t=1}^{b} \eta_t^2 \le \frac{1}{2} & \text{ otherwise,}
\end{cases}
\]
where $C_D$ is an appropriately chosen absolute constant which does not depend on $a$ or $b$.
\end{hypothesis}
Since the first iterate of the $k$th epoch $\Theta_0^k$ is the same as the last iterate of the $(k-1)$th epoch $\Theta_m^k$, the same convention applies to inductive hypotheses; for example, by $L[m,k-1]$ we mean $L[0,k]$.

In particular, the inductive hypotheses imply the following claims.
\begin{enumerate}[label=(\roman*)]
    \item By \cref{cor:inductive-weight-bound}, $L[a, b]$ implies that $\norm{\bm{M}_{a}^{b}}_2 \le \frac{C_{L}^{1/2} + \norm{\Ypi}_F}{\sigma_{\min}(\Xpibn^\top)} \triangleq \xi$.  
    \item Also by \cref{cor:inductive-weight-bound}, $D[a, b]$ and $L[a, b]$ together imply that we have $\norm{\mW_{a}^{b}}_2^2 \le d^2(\frac{1}{2} + \xi)$ and $\norm{\mGamma_{a}^{b}}_2^2 \le \frac{3}{2} + d^2(\frac{1}{2}+\xi)$. For the sake of notational convenience we will write $\Cweight \triangleq \sqrt{\frac{3}{2} + d^2(\frac{1}{2}+\xi)}$, so that 
    $\max\qty{\norm{\mW_{a}^{b}}_2, \norm{\mGamma_{a}^{b}}_2} \le \Cweight$.
    \item By \cref{cor:inductive-pl-bound}, $D[a, b]$ implies that $\sigma_{\min}(\mGamma_{a}^{b})^2 \ge 1/2$. 
    \item By \cref{prop:inductive-one-epoch-update}, if $R[a, b]$ holds for all $(a,b)$, then for appropriately chosen $\eta_k$, the risk $\ellpi(\mM_a^b)$ converges to $\ellpi^*$ at a sublinear rate.
\end{enumerate}

We will explain at a high level how these statements together allow us to conclude that $L[i, k+1]$, $D[i, k+1]$, and $R[i, k+1]$ hold. 
The idea, as in \citet*{ahn2020sgd}, is to accumulate the gradient updates in each epoch and isolate the signal and noise components of each gradient update. For clarity of exposition, we assume for now that $i = 0$. Here are a couple subtleties which we spell out explicitly, including how to generalize to $i > 0$. 
\begin{itemize}
    \item We are not directly performing gradient updates on $\mM$; we instead perform gradient updates on $\mW$ and $\mGamma$. Nevertheless, the \emph{effective} gradient signal for $\mM$ can still be extracted, and we term the remaining noise the \emph{mismatched gradient noise}. For every iterate $(j, k)$, this will formally be denoted by $\vq_j^k$. 
    \item We are not taking a full batch gradient step from $\mM_{0}^k$ to $\mM_0^{k+1}$. Rather, we are taking mini-batch updates which induce path dependency. Nevertheless, as previous works have shown, even at iterate $(j, k)$, we can still extract the full-batch gradient signal evaluated at $\mM_0^k$, and we term the remaining noise the \emph{path dependent noise}. For every iterate $(j, k)$, this will formally be denoted by $\ve_j^k$.
    \item If $i > 0$, then the stepsize changes from $\eta_k$ to $\eta_{k+1}$ in the middle of our pass through the entire dataset. Nevertheless, it's not hard to see that this noise should be relatively small, of order $\eta_{k+1} - \eta_k$ --- which is $O(\eta_k^2)$, as $\eta_k = \Omega(1/k)$. We will call this the \emph{stepsize noise}, the accumulation of which for an epoch update starting from iterate $(i,k)$ to $(i,k+1)$ will be denoted by $\vs_{(i,k+1)}^{(i,k)}$.
\end{itemize}

We can accumulate these noise terms across the update across epoch $k$ to form a composite noise term $\vr^k$. The \emph{full-batch update signal for $\mM$} starting from $\mM_0^k$ will be denoted by $\vgt^k$. We emphasize that $\vgt^k \neq \grad_{\mM} \ellpi(\mM_0^k)$ because we only perform direct gradient updates on the component layers $\mW$ and $\mGamma$. Then as we will show in \cref{sec:epoch-updates}, we can write  
\begin{align}
\mM_0^{k+1} &=  \mM_0^k - \eta_k\vgt^k + \eta_k^2\vr^k.\label{eq:sgd-M-initial}
\end{align}

Next, as seen in \Cref{lemma:smoothness-M}, $\ellpi$ is globally $G_{\pi}$--smooth with respect to $\mM$ for some absolute constant $G_{\pi}$ which \emph{depends} on $\pi$. Thus, using the smoothness inequality as in \Cref{def:smoothness}, we obtain 
\[
\ellpi(\mM_0^{k+1}) - \ellpi(\mM_0^k) \le \ev{\grad_\mM \ellpi(\mM_0^k), \mM_0^{k+1} - \mM_0^k}_F + \frac{\Gpi}{2}\norm{\mM_0^{k+1} - \mM_0^k}_F^2.
\]

The main idea is that we have the following inequality (proved in \cref{lemma:correlation}) that shows that even though $\vgt^k \neq \grad_{\mM} \ellpi(\mM_0^k)$, it is nonetheless correlated to the ``correct'' gradient update $\grad_{\mM} \ellpi(\mM_0^k)$: 
\[
\ev{\grad_{\mM} \ellpi(\mM_0^k), \vgt^k}_F \ge \sigma_{\min}(\mGamma_0^k)^2 \norm{\grad_{\mM} \ellpi(\mM_0^k)}_F^2 \ge \frac{1}{2}\norm{\grad_{\mM} \ellpi(\mM_0^k)}_F^2,
\]
due to the inductive hypothesis $D[0, k]$.

For the stated stepsizes $\eta_k$, one can then plug in the gradient update \Cref{eq:sgd-M-initial} and massage the inequalities a bit to obtain that 
\begin{equation}
    \ellpi(\mM_0^{k+1}) - \ellpi(\mM_0^k) \le -\frac{\eta_k}{4} \norm{\grad_\mM \ellpi(\mM_0^k)}_F^2 + O(\eta_k^2),
\end{equation}
where the constant hidden by the big--$O$ notation is $\poly(m, \Cweight, C_L, \norm{\Xpibn}_F)$.

We now use $\alphapi$-strong convexity of $\ellpi$ with respect to $\mM$ (and hence $\alphapi$-P\L{}) shown in \Cref{lemma:sc-M} to obtain 
\begin{equation}\label{eq:one-step-sgd}
\ellpi(\mM_0^{k+1}) - \ellpi^* \le \qty(1 - \frac{\alphapi \eta_k}{2})(\ellpi(\mM_0^k) - \ellpi^*) + O(\eta_k^2).
\end{equation}
Note that this is precisely the statement of $R[0, k+1]$. 

Provided that we can appropriately bound the noise terms $\vr^k$ to get the asserted $O(\eta_k^2)$ term above, this will imply $R[0, k+1]$. For sufficiently small stepsizes $\eta_k$, we can also use \Cref{eq:one-step-sgd} to prove $L[0, k+1]$.

On the other hand, to prove $D[0, k+1]$, we can directly bound the update $\norm{\mD_0^{k+1} - \mD_{m-1}^k}_2 \le O(\eta_k^2)$ and combine this with the inductive hypothesis $D[m-1,k]$ using the triangle inequality. If the stepsize $\eta_k = O(1/k^\beta)$ for $1/2 < \beta < 1$, then $\sum_{k \ge 1} \eta_k^2 < \infty$, so the absolute constant $C_D$ can be picked such that $\norm{\mD_0^{k+1}}_2 \le \frac{1}{2}$. 

Hence, $R[0,k]$, as stated in \Cref{eq:one-step-sgd}, holds for all $k$ by induction. We can thus unroll the inequality and conclude that $\ellpi(\mM_0^k)$ converges to $\ellpi^*$ under the stated stepsize assumptions, as desired.
\end{proof}

We now outline the structure of the proceeding sections, which fill in the details of the above proof sketch. In \cref{sec:epoch-updates}, we explicitly write out the accumulation of gradient updates across an entire epoch, decomposing into signal and noise components. In \cref{sec:norm-bounds}, we prove some technical lemmas controlling the singular values and norms of various weight matrices and gradients via the approximate invariance matrix $\mD$ and the inductive hypotheses. In \cref{sec:noise} we leverage the norm bounds developed in \cref{sec:norm-bounds} to demonstrate that the accumulated noise terms defined in \cref{sec:epoch-updates} are negligible. Using these results, we are able to establish the $R[i, k+1]$ and $L[i, k+1]$ in \cref{sec:risk-update}. We then turn to bounding the approximate invariances to establish $D[i, k+1]$ in \cref{sec:invariances}. The stray details of the induction are spelled out in \cref{sec:ss-induction}. 
\subsubsection{Rewriting SS epoch gradient updates}\label{sec:epoch-updates}
To show that $L[0, k+1]$ holds, we need to accumulate gradients from $\mM_{0}^{k}$ to $\mM_{0}^{k+1}$. 

First, we look at a single iterate update. For every $j< m$ we have 
\begin{align}
    \mM_{j+1}^k &= (\mW_{j}^k - \eta_k \grad_{\mW} \loss(\Xpi^{j+1}; \Theta_{j}^k))(\mGamma_{j}^k - \eta_k \grad_{\mGamma} \loss(\Xpi^{j+1}; \Theta_{j}^k)) \\ 
    &= \mM_j^k - \eta_k\vg_j^{k} + \eta_k^2\vq_j^k, \label{eq:one-iterate-update}
\end{align}
where we have defined 
\begin{equation}\label{def:pseudo-gradient-signal}
    \vg_j^{k} \triangleq \grad_{\mW} \loss(\Xpi^{j+1}; \Theta_j^k) \mGamma_j^k + \mW_j^k \grad_{\mGamma} \loss(\Xpi^{j+1}; \Theta_j^k),
\end{equation}
which is the gradient of the $(j+1)$th batch of $\Xpibn$ evaluated on the $j$th iterate on epoch $k$, 
and 
\begin{equation}\label{def:mismatched-noise}
    \vq_j^k \triangleq \grad_{\mW} \loss(\Xpi^{j+1}; \Theta_j^k)\grad_{\mGamma} \loss(\Xpi^{j+1}; \Theta_j^k),
\end{equation}
which is the mismatched gradient noise term associated with the fact that we performed gradient updates on $\mW$ and $\mGamma$ rather than $\mM$ directly.

The key observation here is that 
\begin{align*}
\vg_j^k &= \grad_{\mM} \loss(\Xpi^{j+1}; \mM_j^k) (\mGamma_j^k)^2 + \mW_j^k \diag((\mW_j^k)^\top\grad_{\mM} \loss(\Xpi^{j+1}; \mM_j^k)).
\end{align*}
In other words, $\vg_j^k$ is correlated to the ``true'' mini-batch gradient $\grad_{\mM} \loss(\Xpi^{j+1}; \mM_j^k)$ with respect to $\mM$ through the ``interaction terms'' $\mGamma_j^k$ and $\mW_j^k$. 

We show in \cref{lemma:pseudo-M-update} that we can control the size of the noise terms $\vq_j^k$ which arise from the fact that we are not truly taking gradient updates with respect to $\mM$. More specifically, \cref{lemma:pseudo-M-update} implies that $\norm{\vq_j^k}_F = O(1)$.

Next, we actually accumulate gradients. 
The main obstacle we have to deal with is that the mini-batch updates prevent the gradient accumulation from being exactly equal to the full-batch update starting at $\mM_0^k$. Inspired by the approach in \citet[Theorem~1]{ahn2020sgd}, we separate out the gradient update $\vg_j^k$ into a signal term $\vgt_j^k$ and noise term $\ve_j^k$.
Specifically, we write
\begin{equation}\label{eq:signal-noise-update}
    \mM_{j+1}^{k} = \mM_{j}^k - \eta_k\vgt_j^k + \eta_k^2\ve_j^k + \eta_k^2\vq_j^k,
\end{equation}
where 
\begin{equation}\label{def:true-signal}
    \vgt_j^k \triangleq \grad_{\mW} \loss(\Xpi^{j+1}; \Theta_0^k) \mGamma_0^k + \mW_0^k \grad_{\mGamma} \loss(\Xpi^{j+1}; \Theta_0^k),
\end{equation}
is the signal of the gradient update of the $(j+1)$th batch evaluated with parameter values $\Theta_0^k$ (instead of $\Theta_j^k$) and 
\begin{equation}\label{def:noise-term}
    \ve_j^k \triangleq \frac{\vgt_j^k - \vg_j^k}{\eta_k}.
\end{equation}
In particular, in \cref{lemma:noise-bound} below we show that $\norm{\ve_j^k}_F = O(1)$, so that indeed the noise term is negligible with respect to the true gradient signal. 

Taking this as given for now, when we accumulate the gradient updates across epoch $k$, we see that we can define
\begin{equation}\label{def:true-epoch-signal}
\vgt^k \triangleq \sum_{j=0}^{m-1} \vgt_j^k = \grad_{\mW} \ellpi(\Theta_0^k) \mGamma_0^k + \mW_0^k \grad_{\mGamma} \ellpi(\Theta_0^k),
\end{equation}
so that the accumulation reads 
\begin{align}
\mM_0^{k+1} &= \mM_0^k - \eta_k\vgt^k  + \eta_k^2  \sum_{j=0}^{m-1} (\ve_j^k + \vq_j^k)  \label{eq:epoch-update}\\
&=  \mM_0^k - \eta_k\vgt^k + \eta_k^2\vr^k,\label{eq:sgd-M}
\end{align}
where we have additionally defined the composite noise term:
\begin{equation}
    \vr^k \triangleq \sum_{j=0}^{m-1} (\ve_j^k + \vq_j^k),\label{def:composite-noise-term}
\end{equation}
Note that if we instead start from $i>0$, then the composite noise term $\vr^k$ will have an additional noise term $\vs_{(i,k+1)}^{(i,k)}$, which we will address in \cref{sec:noise}. In particular, we show there that the norm of $\vs_{(i,k+1)}^{(i,k)}$ is $O(1)$. Combining this with \cref{lemma:pseudo-M-update,lemma:noise-bound}, we can conclude that $\norm{\vr^k}_F = O(1)$. 

\subsubsection{Norm and singular value bounds based on approximate invariances}\label{sec:norm-bounds}
In this section, we prove several helper lemmas which help us bound noise terms in \cref{sec:noise} and the approximate invariances in \cref{sec:invariances}. 

\paragraph{Upper bounds on the norms of $\mW$ and $\mGamma$.}
Much of \citet*{wu2019global} is dedicated towards showing that the approximate invariances control the weight norms. The trouble with directly extending their strategy lies in the fact that in our setting the invariance $\mD$ is diagonal, which complicates the process of bounding various matrix norms. We first state the following technical lemma which involves the operator norm of Hadamard products.
\begin{lemma}[3.1f in \citet{johnson1990matrix}]\label{lemma:hadamard}
Let $\mA, \mB \in \RR^{d \times d}$ be matrices such that $\mA$ is positive definite. Then $\norm{\mA \odot \mB}_2 \le \norm{\mA}_2 \norm{\mB}_2$ 
\end{lemma}

We leverage \cref{lemma:hadamard} to prove the following useful helper lemma that relates bounds on $\norm{\mI \odot \mW^\top \mW}_2$ to $\norm{\mW}_2$. 
\begin{lemma}\label{lemma:equiv-norms}
Suppose $\norm{\mI \odot \mW^\top \mW}_2 \le \beta$, where $\mW \in \RR^{p \times d}$. Then $\norm{\mW}_2 \le \sqrt{d\beta}$. Conversely, if $\norm{\mW}_2 \le \beta$, then $\norm{\mI \odot \mW^\top \mW}_2 \le \beta^2$. 
\end{lemma}
\begin{proof}
Note that $\mI \odot \mW^\top \mW$ is a diagonal matrix with diagonal entries $\mW_{:,1}^\top \mW_{:,1}, \mW_{:,2}^\top \mW_{:,2}, \ldots, \mW_{:,d}^\top \mW_{:,d}$, where $\mW_{:,i}$ denotes the $i$th column of $\mW$. Hence $\Tr(\mI \odot \mW^\top \mW) = \norm{\mW}_F^2$. Hence $\norm{\mW}_F^2 \le d\beta$ (or tighter by replacing $d$ with the rank of $\mW$), from which it follows that $\norm{\mW}_2 \le \sqrt{d\beta}$. 
For the other direction, we set $\mA = \mI$ and $\mB = \mW^\top \mW$ in \cref{lemma:hadamard}, so $\norm{\mI \odot \mW^\top \mW}_2 \le \norm{\mW}_2^2 \le \beta^2$, as desired. 
\end{proof}

With \Cref{lemma:equiv-norms} in hand, we prove the following technical lemma which gives a uniform bound on the norms of $\mGamma$ and $\mW$ based on $\xi = \norm{\mW\mGamma}_2$. 
\begin{lemma}\label{lemma:weight-bound}
If $\norm{\mD}_2 \le \epsilon < 1$ and $\norm{\mW\mGamma}_2 \le \xi$, we have 
\[
\norm{\mW}_2 \le d\sqrt{1-\epsilon + \xi}, 
\]
and 
\[
\norm{\mGamma^2}_2 \le 1 + \epsilon + d^2\qty(1-\epsilon + \xi).
\]
\end{lemma}
\begin{proof}
We have from $\norm{\mW\mGamma}_2 \le \xi$ that 
\[
\norm{\mW\mGamma^2\mW^\top}_2 \le \xi^2.
\]

Next, our hypothesis that $\norm{\mD}_2 = \norm{\mI + \diag(\mW^\top \mW) - \mGamma^2}_2 \le \epsilon$ implies that 
\[
\mW\mGamma^2\mW^\top \succeq \mW((1-\epsilon)\mI + \diag(\mW^\top \mW))\mW^\top.
\]
Taking norms of both sides and applying the reverse triangle inequality, we obtain that 
\[
\xi^2 \ge \norm{\mW\diag(\mW^\top \mW)\mW^\top}_2 - (1-\epsilon)\norm{\mW}_2^2.
\]
We now lower bound $\norm{\mW\diag(\mW^\top \mW)\mW^\top}_2$. In particular, we expand out the matrix product. Note here that $\diag(\mW^\top \mW)_{i,i} = \norm{\mW_{:,i}}_2^2$. Thus we can write $\mW\diag(\mW^\top \mW)\mW^\top$ as
\[
\mqty[ & & & \\
\mW_{:,1} & \mW_{:,2} & \cdots & \mW_{:,d} \\
& & & ] 
\mqty[\norm{\mW_{:,1}}_2^2 & & & \\
& \norm{\mW_{:,2}}_2^2 & & \\
& & \ddots & \\
& & & \norm{\mW_{:,d}}_2^2] 
\mqty[& \mW_{:,1}^\top & \\
& \mW_{:,2}^\top & \\
& \vdots & \\
& \mW_{:,d}^\top &],
\] 
from which we observe that the $i$th diagonal entry of $\mW\diag(\mW^\top \mW)\mW^\top$ is 
\[
(\mW\diag(\mW^\top \mW)\mW^\top)_{i,i} = 
\sum_{j=1}^d \norm{\mW_{:,j}}_2^2 W_{i,j}^2.
\]
It follows that $\Tr(\mW\diag(\mW^\top \mW)\mW^\top) = \sum_{j=1}^d \norm{\mW_{:,j}}_2^4$. Note that $\norm{\mA}_2 \ge \max_{i, j} \abs{A_{i,j}}$ (the RHS is also known as the \emph{max norm}). For our case we set $\mA = \mW\diag(\mW^\top \mW)\mW^\top$ and note that the diagonal is nonnegative. So in fact in our case we obtain 
\[
\norm{\mW\diag(\mW^\top \mW)\mW^\top}_2 \ge \frac{1}{d} \sum_{j=1}^d \norm{\mW_{:,j}}_2^4.\]

Now notice that $\sum_j \norm{\mW_{:,j}}_2^4 = \sum_j (\sum_i W_{i,j}^2)^2$. Applying Cauchy-Schwarz to the outer sum we find that 
\[
\sum_j \norm{\mW_{:,j}}_2^4 \ge \frac{(\sum_j \sum_i W_{i,j}^2)^2}{d},
\]
but the RHS is equal to $\norm{\mW}_F^4$. Since $\norm{\mW}_F \ge \norm{\mW}_2$, we conclude that  
\[
\norm{\mW\diag(\mW^\top \mW)\mW^\top}_2 \ge \frac{\norm{\mW}_2^4}{d^2}.
\]

In summary, we have 
\[
\frac{\norm{\mW}_2^4}{d^2} - (1-\epsilon)\norm{\mW}_2^2 - \xi^2 \le 0.
\]
Applying the quadratic formula, we find that
\[\norm{\mW}_2 \le d\sqrt{1-\epsilon + \xi}.
\]

For the bound on $\norm{\mGamma}_2$, we start from the definition of $\mD$ and apply the reverse triangle inequality to obtain 
\[
\abs{1 + \norm{\diag(\mW^\top \mW)}_2 - \norm{\mGamma^2}_2} \le \epsilon,
\]
so we obtain 
\[
\norm{\mGamma^2}_2 \le 1 + \epsilon + \norm{\mW}_2^2,
\]
where we used $\norm{\mI \odot \mW^\top \mW}_2 \le \norm{\mW}_2^2$ from \Cref{lemma:equiv-norms}.
From this, the conclusion directly follows.
\end{proof}

Under the inductive hypotheses, \cref{lemma:weight-bound} implies that we can uniformly bound $\max\qty{\norm{\mW_j^k}_2, \norm{\mGamma_j^k}_2}$. This is spelled out in the following corollary.
\begin{corollary}[Norms stay bounded]\label{cor:inductive-weight-bound}
Suppose that $L[j,k]$ and $D[j,k]$ hold. Define 
\[
\Cweight\triangleq \sqrt{\frac{3}{2} +  d^2\qty(\frac{1}{2} + \xi)},
\]
with
\[
\xi \triangleq \frac{C_L^{1/2} + \norm{\Ypi}_F}{\sigma_{\min}(\Xpibn^\top)}.
\]
Here $C_L$ was defined in \cref{hyp:loss-bounded}. 
Then 
\[
\norm{\mM_j^k} \le \xi,
\]
and
\[
\max\qty{\norm{\mW_j^k}_2, \norm{\mGamma_j^k}_2} \le \Cweight.
\]
\end{corollary}
\begin{proof}
We have by triangle inequality that
\[
\norm{\mM_j^k\Xpibn}_2 \le \norm{\mM_j^k\Xpibn}_F \le \norm{\Ypi - \mM_j^k\Xpibn}_F + \norm{\Ypi}_F \le  \ellpi(\mM_j^k)^{1/2} + \norm{\Ypi}_F.
\]
Since $L[j, k]$ holds, we have have $\norm{\Ypi - \mM_j^k\Xpibn}_F^2 \le C_L$. Furthermore, as $n \ge d$, we know that $\norm{\mM_j^k\Xpibn}_2 \ge \sigma_{\min}(\Xpibn^\top)\norm{\mM_j^k}_2$ and by \Cref{assumption:full-rank-ss} we have $\sigma_{\min}(\Xpibn^\top) > 0$. Hence we obtain 
\[
\norm{\mM_j^k}_2 \le \frac{C_L^{1/2} + \norm{\Ypi}_F}{\sigma_{\min}(\Xpibn^\top)} = \xi.
\] 

It follows that $\xi$ works as a bound on $\norm{\mM_j^k}_2$ for the application of \Cref{lemma:weight-bound}. Since $D[j,k]$ holds by assumption, this means that the hypothesis on $\mD_j^k$ is satisfied with $\epsilon = 1/2$. In summary, all the hypotheses of \Cref{lemma:weight-bound} are satisfied. We can thus conclude that 
\[
\max\qty{\norm{\mW_j^k}_2, \norm{\mGamma_j^k}_2} \le \Cweight,
\]
as desired.
\end{proof}

The importance of these upper bounds on weight norms is that they allow us to upper bound the norms of gradients of $\loss$ with respect to various parameters. 

\paragraph{Upper bounding the norms of gradients.}
The following lemma gives an upper bound on the norms of various gradients.
\begin{lemma}\label{lemma:grad-upper-bound}
For any $a \in [m]$ and $\Theta = (\mW, \mGamma)$ we have 
\begin{align*}
    \norm{\grad_{\mW} \loss(\Xpi^a; \Theta)}_F^2 &\le \norm{\mGamma}_2^2\norm{\bn(\Xpi^a)}_2^2\loss(\Xpi^a; \Theta) \\
    \norm{\grad_{\mGamma} \loss(\Xpi^a; \Theta)}_F^2 &\le \norm{\mW}_2^2 \norm{\bn(\Xpi^a)}_2^2 \loss(\Xpi^a; \Theta) \\
    \norm{\grad_{\mM} \loss(\Xpi^a; \mM)}_F^2 &\le \norm{\bn(\Xpi^a)}_2^2 \loss(\Xpi^a; \mM)
\end{align*}
\end{lemma}
\begin{proof}
First, we have by definition 
\[
\loss(\Xpi^a; \Theta) = \norm{\mW\mGamma\bn(\Xpi^a) - \Ypi^a}_F^2.
\]
Hence, the mini-batch gradients can be computed explicitly as
\begin{align}
    \grad_{\mM} \loss(\Xpi^a; \mM) &= -(\Ypi^a - \mM\bn(\Xpi^a))\bn(\Xpi^a)^\top, \label{eq:minibatch-gradM}\\
    \grad_\mW \loss(\Xpi^a; \Theta) &= \grad_{\mM} \loss(\Xpi^a; \mM)\mGamma, \label{eq:minibatch-gradW}\\
    \grad_{\mGamma} \loss(\Xpi^a; \Theta) &= \diag(\mW^\top\grad_{\mM} \loss(\Xpi^a; \mM)).\label{eq:minibatch-gradG}
\end{align}
Since $\loss(\Xpi^a; \mM) = \norm{\Ypi^a - \mM\bn(\Xpi^a)}_F^2$ and $\norm{\mA\mB}_F \le \norm{\mA}_2\norm{\mB}_F$, \Cref{eq:minibatch-gradM} gives  
\[
\norm{\grad_{\mM} \loss(\Xpi^a; \mM)}_F^2 \le \norm{\bn(\Xpi^a)}_2^2 \loss(\Xpi^a; \mM).
\]
It thus follows from \Cref{eq:minibatch-gradW} that
\[
\norm{\grad_{\mW} \loss(\Xpi^a; \Theta)}_F^2 \le \norm{\mGamma}_2^2 \norm{\grad_{\mM} \loss(\Xpi^a; \Theta)}_F^2  \le \norm{\mGamma}_2^2\norm{\bn(\Xpi^a)}_2^2\loss(\Xpi^a; \Theta).
\]

Similarly, inspecting \Cref{eq:minibatch-gradG}, since $\norm{\diag(\mA)}_F^2 \le \norm{\mA}_F^2$, we have 
\[
\norm{\grad_{\mGamma} \loss(\Xpi^a; \Theta)}_F^2 \le \norm{\mW^\top\grad_{\mM} \loss(\Xpi^a; \Theta)}_F^2 \le  \norm{\mW}_2^2 \norm{\bn(\Xpi^a)}_2^2\loss(\Xpi^a; \Theta).
\]
\end{proof}
As a consequence of \cref{cor:inductive-weight-bound}, under the inductive hypotheses we can also bound the gradient norms by absolute constants.
\begin{corollary}\label{cor:inductive-gradient-bound}
Assume $D[j,k]$ and $L[j,k]$ hold. Then, for any $a \in [m]$, we have 
\begin{align*}
    \norm{\grad_{\mM} \loss(\Xpi^a; \mM_j^k)}_F^2 &\le C_L \norm{\bn(\Xpi^a)}_2^2,\\
    \norm{\grad_{\mW} \loss(\Xpi^a; \Theta_j^k)}_F^2 &\le \Cweight^2 C_L\norm{\bn(\Xpi^a)}_2^2,\\
    \norm{\grad_{\mGamma} \loss(\Xpi^a; \Theta_j^k)}_F^2 &\le \Cweight^2 C_L \norm{\bn(\Xpi^a)}_2^2,
\end{align*}
where $\Cweight$ was previously defined in \cref{cor:inductive-weight-bound}.
\end{corollary}

We now turn from upper bounds to lower bounds. The crux here is to start with bounding the minimum singular value of $\mGamma$ away from zero. This in turns allows us to lower bound the correlation between $\vgt^k$ and $\grad_{\mM} \ellpi(\mM_0^k)$ away from zero. As we will see, we can also show similar correlation lower bounds for the cases $i > 0$.

\paragraph{Bounding the minimum singular value of $\mGamma^2$.}

In order to bound $\sigma_{\min}(\mGamma_i^k)$ away from zero, we need to show that the approximate invariances prevent $\mGamma$ from vanishing on any coordinate.
To do so, we appeal to an alternate formulation of the Courant-Fisher theorem for singular values, which we restate below for completeness.
\begin{theorem}[Courant-Fisher]
Let $\mA, \mB \in \RR^{m \times n}$. Then $\abs{\sigma_k(\mA) - \sigma_k(\mB)} \le \norm{\mA-\mB}_2$ for $k \in [\min\{m, n\}]$.
\end{theorem}

With this in mind, we formally prove that the minimum singular value of $\mGamma^2$ is bounded away from zero.
\begin{lemma}\label{lemma:gamma-sv-bound}
Suppose that $\norm{\mD}_2 = \norm{\mI + \diag(\mW^\top \mW - \mGamma^2)}_2 \le \epsilon$. Then we have 
\[
\sigma_{\min}(\mGamma^2) \ge 1 - \epsilon. 
\]
\end{lemma}
\begin{proof}
Setting $\mA \triangleq \mI + \diag(\mW^\top \mW)$ and $\mB \triangleq \mGamma^2$ in Courant-Fisher yields 
\[
\abs{\sigma_d(\mI + \diag(\mW^\top \mW)) - \sigma_d(\mGamma^2)} \le \norm{I + \diag(\mW^\top \mW) - \mGamma^2}_2.
\]
Since the RHS is just $\mD$, we obtain that 
\[
\sigma_{\min}(\mGamma^2) \ge 1 + \sigma_{\min}(\diag(\mW^\top \mW)) - \norm{\mD}_2.
\]
The conclusion easily follows.
\end{proof} 

Under the inductive hypothesis $D[i,k]$, i.e. $\norm{\mD_i^k}_2 \le \frac{1}{2}$, this immediately implies the following corollary. We will see in the following section (in \cref{cor:inductive-correlation}) that this minimum singular value bound for $\mGamma_i^k$ can be interpreted in the following manner. Although the effective P\L{} condition evolves dynamically, the associated P\L{} constant always stays bounded away from zero. 
\begin{corollary}[P\L{} bounded away from zero]\label{cor:inductive-pl-bound}
Assume $D[i,k]$ holds. Then we have 
\[
\sigma_{\min}(\mGamma_i^k)^2 \ge \frac{1}{2}.
\]
\end{corollary}

\paragraph{The accumulated gradient signal is correlated with the full-batch gradient signal.}
\begin{lemma}[Correlation of $\vgt^k$ and $\grad_{\mM} \ellpi(\mM_0^k)$]\label{lemma:correlation}
For all $k$, we have 
\[
\ev{\grad_{\mM} \ellpi(\mM_0^k), \vgt^k}_F \ge \sigma_{\min}(\mGamma_0^k)^2 \norm{\grad_{\mM} \ellpi(\mM_0^k)}_F^2.
\]
\end{lemma}
\begin{proof}
Recall that we previously defined
\[
\vgt^k \triangleq \grad_{\mW} \ellpi(\Theta_0^k) \mGamma_0^k + \mW_0^k \grad_{\mGamma} \ellpi(\Theta_0^k).
\]
Note that if we have $\mA, \mLambda \in \RR^{n \times n}$, with $\mLambda = \diag(\lambda_1, \ldots, \lambda_n)$ a diagonal matrix with nonnegative entries, then 
\[
\ev{\mA, \mA\mLambda}_F = \ev{\mA \mLambda^{1/2}, \mA \mLambda^{1/2}}_F = \norm{\mA \mLambda^{1/2}}_F^2 \ge \min_{i} \lambda_i \norm{\mA}_F^2.
\]
Also, we have 
\[
\ev{\mA, \diag(\mA)}_F = \ev{\diag(\mA), \diag(\mA)}_F = \norm{\diag(\mA)}_F^2 \ge 0.
\]

Hence combining \Cref{eq:minibatch-gradM,eq:minibatch-gradG} and the above inequalities, we have
\begin{align*}
\ev{\grad_{\mM} \ellpi(\mM_0^k), \vgt^k}_F &= \ev{\grad_{\mM} \ellpi(\mM_0^k), \grad_{\mW} \ellpi(\Theta_0^k) \mGamma_0^k}_F \\
&\quad\quad + \ev{\grad_{\mM} \ellpi(\mM_0^k), \mW_0^k \grad_{\mGamma} \ellpi(\Theta_0^k)}_F \\
&= \ev{\grad_{\mM} \ellpi(\mM_0^k), \grad_{\mM} \ellpi(\mM_0^k) (\mGamma_0^k)^2}_F \\
&\quad\quad + \ev{(\mW_0^k)^\top \grad_{\mM} \ellpi(\mM_0^k),  \diag((\mW_0^k)^\top \grad_{\mM} \ellpi(\mM_0^k))}_F \\
& \ge \sigma_{\min}(\mGamma_0^k)^2 \norm{\grad_{\mM} \ellpi(\mM_0^k)}_F^2.
\end{align*}
\end{proof}

We obtain the following corollary of the above lemma and \cref{cor:inductive-pl-bound}.
\begin{corollary}\label{cor:inductive-correlation}
Assume $D[0, k]$ holds. We have 
\begin{equation}\label{eq:correlation}
    \ev{\grad_{\mM} \ellpi(\mM_0^k), \vgt^k}_F \ge \frac{1}{2}\norm{\grad_{\mM} \ellpi(\mM_0^k)}_F^2.
\end{equation}
More generally, assume $D[i,k]$ holds. We have 
\begin{equation}\label{eq:correlation-general}
    \ev{\grad_{\mM} \ellpi(\mM_i^k), \vgt^{(i,k)}}_F \ge \frac{1}{2}\norm{\grad_{\mM} \ellpi(\mM_i^k)}_F^2,
\end{equation}
where $\vgt^{(i,k)}$ is the analogous quantity to $\vgt^k$ for accumulating gradients starting at iterate $(i,k)$ rather than $(0,k)$. It is defined more formally in \cref{def:true-epoch-signal-general}.
\end{corollary}
\subsubsection{Bounding noise terms}\label{sec:noise}
We now turn to bounding the composite noise term $\vr^k$. This is crucial to ensure the global convergence via \Cref{eq:one-step-sgd} and also to control the approximate invariances.

\paragraph{Mismatched gradient noise is negligible.}
As promised, we show that the mismatched gradient noise terms $\vq_j^k$ are negligible when we accumulate gradients from $\mM_0^k$ to $\mM_{0}^{k+1}$. 
\begin{lemma}\label{lemma:pseudo-M-update}
Assume that $L[j, k]$ and $D[j, k]$ hold for $j < m$. Then we have 
\[
\norm{\vq_j^k}_F \le \Cweight^2 C_L \norm{\bn(\Xpi^{j+1})}_2^2.
\]
Furthermore, for any $t < m$ we have 
\[
\sum_{j=0}^{t} \norm{\vq_j^k}_F \le \Cweight^2 C_L \norm{\Xpibn}_F^2 .
\]
\end{lemma}
\begin{proof}
Recall the definition of $\vq_j^k$, reproduced here for reference:
\begin{equation*}
    \vq_j^k \triangleq \grad_{\mW} \loss(\Xpi^{j+1}; \Theta_j^k)\grad_{\mGamma} \loss(\Xpi^{j+1}; \Theta_j^k).
\end{equation*}
Since $L[j,k]$ and $D[j,k]$ hold for $(j,k)$, we can apply \cref{cor:inductive-gradient-bound} to conclude that 
\[
\norm{\vq_j^k}_F \le \Cweight^2 C_L \norm{\bn(\Xpi^{j+1})}_2^2.
\]

Since the inductive hypotheses hold for every $j < m$, when we accumulate the noise terms from $(0, k)$ to $(t, k)$, we can apply the above bound to conclude that
\begin{align*}
\sum_{j=0}^{t} \norm{\vq_j^k}_F &\le \sum_{j=0}^{t} \norm{\bn(\Xpi^{j+1})}_2^2 \Cweight^2 C_L \\
&\le \Cweight^2 C_L  \sum_{j=0}^{m-1} \norm{\bn(\Xpi^{j+1})}_F^2\\ 
&= \Cweight^2 C_{L} \norm{\Xpibn}_F^2,
\end{align*}
where the last equality used the definition of $\Xpibn$.
\end{proof}

\paragraph{Path dependent noise arising from mini-batch updates is negligible.}
In order to bound the noise term coming from mini-batch updates, we first prove the following auxiliary lemma that shows that the iterates don't move far within an epoch.
\begin{lemma}\label{lemma:small-distance-epoch}
Fix $t \le m$ and assume $D[j, k]$ and $L[j, k]$ hold for all $j < t$. Then we have
\[
\norm{\mW_t^k - \mW_0^k}_2 \le \sqrt{t}\eta_k \Cweight C_{L}^{1/2} \norm{\Xpibn}_F .
\]
The same inequality holds true if we replace $\mW$ with $\mGamma$.

We also have 
\[
\norm{\mM_t^k - \mM_0^k}_2 \le 2\sqrt{t}\eta_k \Cweight^2 C_{L}^{1/2} \norm{\Xpibn}_F  + \eta_k^2 \Cweight^2 C_L \norm{\Xpibn}_F^2.
\]
\end{lemma}
\begin{proof}
We have by definition that 
\[
\mW_t^k = \mW_0^k - \eta_k \sum_{j=0}^{t-1} \grad_{\mW} \ell(\Xpi^{j+1}; \Theta_j^k).
\]
Now, we have 
\begin{align*}
\norm{\mW_t^k - \mW_0^k}_2 &\le \eta_k \sum_{j=0}^{t-1} \norm{\grad_{\mW} \ell(\Xpi^{j+1}; \Theta_j^k)}_2 \\
&\le \eta_k \Cweight C_L^{1/2}  \sum_{j=0}^t \norm{\bn(\Xpi^{j+1})}_2\\
&\le \sqrt{t} \eta_k \Cweight C_{L}^{1/2}\norm{\Xpibn}_F 
\end{align*}
where in the first line we have applied the triangle inequality, in the second line we have applied \cref{cor:inductive-gradient-bound}, and in the last line we have applied Cauchy-Schwarz.

The same proof holds for $\mGamma$. 

For $\mM$, \Cref{eq:one-iterate-update} gives
\[
\mM_t^k = \mM_0^k - \eta_k\sum_{j=0}^{t-1} \vg_j^{k} + \eta_k^2\sum_{j=0}^{t-1} \vq_j^k.
\]
Combining \cref{def:pseudo-gradient-signal,cor:inductive-weight-bound,cor:inductive-gradient-bound} yields
\begin{align*}
\norm{\vg_j^k}_2 &\le \norm{\mGamma_j^k}_2\norm{\grad_{\mW} \loss(\Xpi^{j+1}; \Theta_j^k)}_2 + \norm{\mW_j^k}_2\norm{\grad_{\mGamma}  \loss(\Xpi^{j+1}; \Theta_j^k)}_2\\
&\le 2\Cweight^2 C_L^{1/2} \norm{\bn(\Xpi^{j+1})}_2.
\end{align*}
Hence, summing up over $j$, using Cauchy-Schwarz, and applying the noise bound \cref{lemma:pseudo-M-update}, it follows that 
\[
\norm{\mM_t^k - \mM_0^k}_2 \le 2\sqrt{t}\eta_k \Cweight^2 C_{L}^{1/2} \norm{\Xpibn}_F  + \eta_k^2 \Cweight^2 C_L \norm{\Xpibn}_F^2.\]

\end{proof}
Now we show that the noise term $\norm{\ve_j^k}_2$ is $O(1)$.
\begin{lemma}\label{lemma:noise-bound}
If $L[j,k]$ and $D[j, k]$ both hold for all $j < m$ then we have for each $j$ that 
\[
\norm{\ve_j^k}_2 \le 4\sqrt{j}\Cweight^2 C_L^{1/2}\norm{\Xpibn}_F^2(C_L^{1/2} + \Cweight^2\norm{\Xpibn}_F) + O(\eta_k).
\]
Hence, we also have 
\begin{align*}
    \sum_{i=0}^{m-1} \norm{\ve_j^k}_2 \le 4m^{3/2}\Cweight^2 C_L^{1/2}\norm{\Xpibn}_F^2(C_L^{1/2} + \Cweight^2\norm{\Xpibn}_F) + O(\eta_k).
\end{align*}
\end{lemma}
\begin{proof}
Inspecting the definition of $\ve_j^k$ (\cref{def:noise-term}), let us bound the quantity
\begin{align*}
    \eta_k \ve_j^k &= \underbrace{\grad_{\mW} \loss(\Xpi^{j+1}; \Theta_0^k) \mGamma_0^k - \grad_{\mW} \loss(\Xpi^{j+1}; \Theta_j^k) \mGamma_j^k}_{\partI} \\
    &\quad\quad + \underbrace{\mW_0^k \grad_{\mGamma} \loss(\Xpi^{j+1}; \Theta_0^k) - \mW_j^k \grad_{\mGamma} \loss(\Xpi^{j+1}; \Theta_j^k)}_{\partII}.
\end{align*}

First, we have by triangle inequality and the identity \Cref{eq:minibatch-gradW} that the norm of \partI is at most 
\begin{align*}
    & \norm{\grad_{\mW} \loss(\Xpi^{j+1}; \Theta_0^k) \mGamma_0^k - \grad_{\mW} \loss(\Xpi^{j+1}; \Theta_0^k) \mGamma_j^k}_2 \\
    & \quad\quad + \norm{\grad_{\mW} \loss(\Xpi^{j+1}; \Theta_0^k) \mGamma_j^k - \grad_{\mW} \loss(\Xpi^{j+1}; \Theta_j^k) \mGamma_j^k}_2 \\
    &\le  \norm{\mGamma_0^k(\mGamma_j^k - \mGamma_0^k)}_2 \norm{\grad_{\mM} \loss(\Xpi^{j+1}; \mM_0^k)}_2 \\
    &\quad\quad + \norm{\mGamma_j^k}_2\norm{\grad_{\mM} \loss(\Xpi^{j+1}; \mM_0^k)\mGamma_0^k - \grad_{\mM} \loss(\Xpi^{j+1}; \mM_j^k) \mGamma_j^k}_2  \\
    &\le \norm{\mGamma_0^k(\mGamma_j^k - \mGamma_0^k)}_2 \norm{\grad_{\mM} \loss(\Xpi^{j+1}; \mM_0^k)}_2 \\
    &\quad\quad + \norm{\mGamma_j^k}_2 \Bigg( \norm{\grad_{\mM} \loss(\Xpi^{j+1}; \mM_0^k) \mGamma_0^k - \grad_{\mM} \loss(\Xpi^{j+1}; \mM_0^k) \mGamma_j^k}_2 \\
    &\quad\quad\quad\quad\quad\quad\quad + \norm{\grad_{\mM} \loss(\Xpi^{j+1}; \mM_0^k) \mGamma_j^k - \grad_{\mM} \loss(\Xpi^{j+1}; \mM_j^k)\mGamma_j^k}_2 \Bigg) \\
    &\le (\norm{\mGamma_0^k}_2 + \norm{\mGamma_j^k}_2)\norm{\mGamma_j^k - \mGamma_0^k}_2 \norm{\grad_{\mM} \loss(\Xpi^{j+1}; \mM_0^k)}_2 \\
    &\quad\quad + \norm{\mGamma_j^k}_2^2\norm{\grad_{\mM} \loss(\Xpi^{j+1}; \mM_0^k) - \grad_{\mM} \loss(\Xpi^{j+1}; \mM_j^k)}_2.
\end{align*}
Applying the weight bounds in \cref{cor:inductive-weight-bound,cor:inductive-gradient-bound} and \cref{lemma:small-distance-epoch} to the first term yields an upper bound of 
\[
2\Cweight \cdot (\sqrt{j}\eta_k \Cweight C_{L}^{1/2} \norm{\Xpibn}_F ) \cdot (\norm{\bn(\Xpi^{j+1})}_2C_L^{1/2}) \le 2\sqrt{j}\eta_k \Cweight^2 C_L \norm{\Xpibn}_F^2.
\]

Turning to the second term, we can apply the smoothness bound in \cref{lemma:smoothness-M} and the inductive bounds in  \cref{cor:inductive-weight-bound,cor:inductive-gradient-bound,lemma:small-distance-epoch} to obtain an upper bound of 
\[
\Cweight^2 \norm{\bn(\Xpi^{j+1})}_2^2 \norm{\mM_j^k - \mM_0^k}_2 \le 2\sqrt{j}\eta_k \Cweight^4 C_{L}^{1/2} \norm{\Xpibn}_F^3  + \eta_k^2 \Cweight^4 C_L \norm{\Xpibn}_F^4.
\]

Putting it together, we have
\begin{align*}
    \partI &\le 2\sqrt{j}\eta_k \Cweight^2 C_L^{1/2}\norm{\Xpibn}_F^2(C_L^{1/2} + \Cweight^2\norm{\Xpibn}_F) + \eta_k^2 \Cweight^4 C_L \norm{\Xpibn}_F^4.
\end{align*}

Similarly, for \partII we have the exact same bound since we can apply \cref{lemma:hadamard} to remove the diagonal operator and uniformly bound $\norm{\mW_j^k}_2$ and $\norm{\mGamma_j^k}_2$ by $\Cweight$. 

Finally, combining \partI and \partII and dividing through by $\eta_k$, we can conclude that 
\[
\norm{\ve_j^k}_2 \le 4\sqrt{j}\Cweight^2 C_L^{1/2}\norm{\Xpibn}_F^2(C_L^{1/2} + \Cweight^2\norm{\Xpibn}_F) + O(\eta_k).
\]
Summing up $\norm{\ve_j^k}_2$ over all $j$ and crudely bounding $\sum_{j=0}^{m-1} \sqrt{j} \le m^{3/2}$, we see that 
\begin{align*}
    \sum_{i=0}^{m-1} \norm{\ve_j^k}_2 \le 4m^{3/2}\Cweight^2 C_L^{1/2}\norm{\Xpibn}_F^2(C_L^{1/2} + \Cweight^2\norm{\Xpibn}_F) + O(\eta_k).
\end{align*}
\end{proof}

\paragraph{Stepsize noise is negligible for $i > 0$.}
We now quickly show that the effect of generalizing the induction to $i > 0$ is negligible. In particular, we can carry out the same proof, except we will have to redefine the per-epoch update so it can account for gradient updates starting at an arbitrary iterate $(i, k)$ rather than $(0, k)$. We explicitly redefine these terms below by quickly revisiting the signal-noise decomposition in \Cref{sec:epoch-updates}.
Recall that a single-iterate update at iteration $(a,b)$ can be written as (\Cref{eq:one-iterate-update})
\begin{align*}
    \mM_{a+1}^b = \mM_a^b - \eta_b\vg_a^{b} + \eta_b^2\vq_a^b,
\end{align*}
where we defined 
\begin{align*}
    \vg_a^{b} &\triangleq \grad_{\mW} \loss(\Xpi^{a+1}; \Theta_a^b) \mGamma_a^b + \mW_a^b \grad_{\mGamma} \loss(\Xpi^{a+1}; \Theta_a^b),\\
    \vq_a^b &\triangleq \grad_{\mW} \loss(\Xpi^{a+1}; \Theta_a^b)\grad_{\mGamma} \loss(\Xpi^{a+1}; \Theta_a^b).
\end{align*}
Consider carrying out the same accumulation as in \Cref{sec:epoch-updates}, but this time choosing $(i,k)$ instead $(0,k)$ as the ``pivot.'' For this purpose, we will change our notational convention a little bit and use superscripts to denote the pivot or the starting point $(i,k)$. As the redefinitions of the ``signal'' $\vgt_a^b$ (\Cref{def:true-signal}) and path dependent noise $\ve_a^b$ (\Cref{def:noise-term}), we define
\begin{align*}
\vgt_{a}^{(i,k)} &\triangleq \grad_{\mW} \loss(\Xpi^{a+1}; \Theta_i^k) \mGamma_i^k + \mW_i^k \grad_{\mGamma} \loss(\Xpi^{a+1}; \Theta_i^k),\\
\ve_{(a,b)}^{(i,k)} &\triangleq \frac{\vgt_{a}^{(i,k)} - \vg_a^b}{\eta_k},
\end{align*}
for indices $(a,b)$ satisfying $(i,k) \leq (a,b) \leq (i-1,k+1)$.

This way, the accumulation of updates on $\mM$ from iteration $(i,k)$ to $(i-1,k+1)$ can be represented as
\begin{align}
    \mM_i^{k+1} 
    &= \mM_i^k - \eta_k \sum_{j=i}^{m-1} \vg_j^k + \eta_k^2 \sum_{j=i}^{m-1} \vq_j^k - \eta_{k+1} \sum_{j=0}^{i-1} \vg_j^{k+1} + \eta_{k+1}^2 \sum_{j=0}^{i-1} \vq_j^{k+1}\notag\\
    &= \mM_i^k - \eta_k \sum_{(a,b)=(i,k)}^{(i-1, k+1)} \vg_a^b + \eta_k^2 \sum_{(a,b)=(i,k)}^{(i-1, k+1)} \vq_a^b \notag\\
    &\quad\quad\quad- (\eta_{k+1}-\eta_k) \sum_{j=0}^{i-1} \vg_j^{k+1} + (\eta_{k+1}^2 - \eta_k^2) \sum_{j=0}^{i-1} \vq_j^{k+1}\notag\\
    &= \mM_i^k - \eta_k \vgt^{(i,k)} + \eta_k^2 \sum_{(a,b)=(i,k)}^{(i-1, k+1)} \qty(\ve_{(a,b)}^{(i,k)} + \vq_a^b) + \eta_k^2 \vs^{(i,k)}_{(i,k+1)}\label{eq:epoch-update-general}
\end{align}
where in the last equality we used $\vg_a^b = \vgt_{a}^{(i,k)} - \eta_k \ve_{(a,b)}^{(i,k)}$ and also defined the accumulated signal $\vgt^{(i,k)}$ (a redefinition of $\vgt^k$ from \Cref{def:true-epoch-signal}) and the stepsize noise $\vs^{(i,k)}_{(i,k+1)}$ as
\begin{align}
    \vgt^{(i,k)} &\triangleq \sum_{(a,b)=(i,k)}^{(i-1, k+1)} \vgt_{a}^{(i,k)} = \grad_{\mW} \ellpi(\Theta_i^k) \mGamma_i^k + \mW_i^k \grad_{\mGamma} \ellpi(\Theta_i^k),\label{def:true-epoch-signal-general}\\
    \vs^{(i,k)}_{(i,k+1)} &\triangleq - \frac{\eta_{k+1} - \eta_k}{\eta_k^2} \sum_{j=0}^{i-1} \vg_j^{k+1} + \frac{\eta_{k+1}^2 - \eta_k^2}{\eta_k^2} \sum_{j=0}^{i-1} \vq_j^{k+1}.\label{def:stepsize-noise}
\end{align}
As a sanity check, we can quickly see that the stepsize noise $\vs^{(i,k)}_{(i,k+1)}$ is zero if $i = 0$.

Now notice that \Cref{eq:epoch-update-general} can be thought of as a generalization of the per-epoch update (\Cref{eq:epoch-update}) originally obtained for $i = 0$. For now, suppose we ignore the last term of \Cref{eq:epoch-update-general} involving the stepsize noise $\vs^{(i,k)}_{(i,k+1)}$. 
Then, if we carry out the above analysis for bounding the remaining terms in \Cref{eq:epoch-update-general}, there is no difference in the argument up to reindexing; we can consider this as using $\eta_k$ for the stepsize \emph{even for the iterates $(a,b)$ where $b = k+1$}. In particular, the lemmas of the previous sections all hold up to reindexing notation.

Therefore, it now suffices to show that the stepsize noise $\vs^{(i,k)}_{(i,k+1)}$ is of the same order as the other noise terms; in particular, $\norm{\vs_{(i,k+1)}^{(i,k)}}_F = O(1)$.
\begin{lemma}\label{lemma:stepsize-noise}
Assume $L[j,k+1]$ and $D[j,k+1]$ hold for all $j \le i-1$. Suppose that 
\[
\eta_k = O\qty(\frac{1}{k^{\beta}}),
\]
for some $1/2 < \beta < 1$. Then the stepsize noise 
\[
\norm{\vs_{(i,k+1)}^{(i,k)}}_F = O(1).
\]
\end{lemma}
\begin{proof}
Since $D[j, k+1]$ holds, \cref{lemma:pseudo-M-update} demonstrates that $\norm{\vq_j^{k+1}}_2 = O(1)$. On the other hand, \cref{cor:inductive-gradient-bound,cor:inductive-weight-bound} show that $\norm{\vg_j^{k+1}}_2 = O(1)$. Since $\eta_k = O(1/k^{\beta})$ and $\beta \le 1$, we have $\eta_{k+1} - \eta_k = O(1/k^{\beta+1}) = O(\eta_k^2)$. Similarly $\eta_{k+1}^2 - \eta_k^2 = O(1/k^{2\beta+1}) = O(\eta_k^3)$. Plugging these into \cref{def:stepsize-noise}, we conclude that $\norm{\vs_{(i,k+1)}^{(i,k)}}_F = O(1)$, as desired.
\end{proof}

\paragraph{Composite noise term is negligible.}
Now that we have formally defined the stepsize noise that arise for $i > 0$, we also redefine the composite noise term $\vr^k$ (\Cref{def:composite-noise-term}) originally defined for $i = 0$. The updated definition is simply
\begin{equation}\label{def:redefined-composite-noise}
\vr^{(i,k)} \triangleq \sum_{(a,b)=(i,k)}^{(i-1, k+1)} \qty(\ve_{(a,b)}^{(i,k)} + \vq_a^b) + \vs_{(i,k+1)}^{(i,k)},
\end{equation}
which allows us to rewrite the epoch update spelled out in \Cref{eq:epoch-update-general} as
\begin{equation}
\mM_i^{k+1} = \mM_i^k - \eta_k \vgt^{(i,k)} + \eta_k^2\vr^{(i,k)}.\label{eq:sgd-M-general},
\end{equation}
which is a generalization of \Cref{eq:sgd-M}.

It is left to show formally that the composite noise term $\vr^{(i,k)}$ defined in \Cref{def:redefined-composite-noise}, obtained from combining the mismatched gradient noise terms $\vq_a^b$, the path dependent noise $\ve_{(a,b)}^{(i,k)}$ for $(i,k) \le (a,b) \le (i-1, k+1)$, and the stepsize noise $\vs_{(i,k+1)}^{(i,k)}$, is indeed $O(1)$. 
\begin{proposition}[Composite noise term]\label{prop:composite-noise}
Suppose $L[a,b]$ and $D[a,b]$ hold for $(i,k) \le (a,b) \le (i-1, k+1)$, and $\eta_k = O(1/k^{\beta})$ for some $1/2 < \beta < 1$. Then the composite noise term $\vr^{(i,k)}$ satisfies
\[
\norm{\vr^{(i,k)}}_F \le \poly(m, \Cweight, C_L, \norm{\Xpibn}_F) + O(\eta_k). 
\]
\end{proposition}
\begin{proof}
Combining \cref{lemma:pseudo-M-update,lemma:noise-bound,lemma:stepsize-noise}, taking care to analyze the constants (which are all $\poly(m, \Cweight, C_L, \norm{\Xpibn}_F)$) hidden by the big--$O$ notation, yields the desired result.
\end{proof}

\subsubsection{Accumulated loss update}\label{sec:risk-update}
In this section, we formally account for the noise terms and prove that an accumulated loss inequality holds. More precisely, we can use the gradient update spelled out in \Cref{eq:sgd-M-general} and the noise bounds in \Cref{sec:noise} to obtain a single epoch loss update. In other words, this section prove the inductive step that the hypotheses $R[i, k+1]$ and $L[i, k+1]$ holds.
For the sake of simplicity, in this section we focus on the case $i = 0$. However, as discussed above, the case $i > 0$ only adds a negligible stepsize error and the same arguments go through.

We start with the following proposition, which proves the hypothesis $R[0,k+1]$.
\begin{proposition}\label{prop:one-epoch-update}
Assume that $L[j, k]$ and $D[j,k]$ hold for all $j < m$. 
Consider optimizing the linear+BN network with stepsize satisfying
\[
\eta_k \le \frac{1}{2k^\beta},
\] 
for $1/2 < \beta < 1$.

Then 
\begin{equation}\label{eq:one-epoch-sgd-detailed}
\ellpi(\mM_0^{k+1}) - \ellpi^* \le \qty(1 - \frac{\alphapi\eta_k}{2})(\ellpi(\mM_0^k) - \ellpi^*) + \poly(m, \Cweight, C_L, \norm{\Xpibn}_F)\eta_k^2,
\end{equation}
where the $\poly(m, \Cweight, C_L, \norm{\Xpibn}_F)$ term is independent of $k$ and has constant degree. 
\end{proposition}
\begin{proof}
First, we use the $\Gpi$-smoothness of $\ellpi$ with respect to $\mM$ guaranteed by \Cref{lemma:smoothness-M} to obtain 
\[
\ellpi(\mM_0^{k+1}) - \ellpi(\mM_0^k) \le \ev{\grad_\mM \ellpi(\mM_0^k), \mM_0^{k+1} - \mM_0^k}_F + \frac{\Gpi}{2}\norm{\mM_0^{k+1} - \mM_0^k}_F^2.
\]
Using the gradient update \Cref{eq:sgd-M} and Cauchy-Schwarz, we can upper bound the RHS by
\begin{align*}
& \ev{\grad_\mM \ellpi(\mM_0^k), -\eta_k \vgt^k + \eta_k^2 \vr^k}_F + \frac{\Gpi}{2}\norm{\mM_0^{k+1} - \mM_0^k}_F^2 \\
& \le -\eta_k \ev{\grad_\mM \ellpi(\mM_0^k), \vgt^k}_F + \eta_k^2 \norm{\vr^k}_F \norm{\grad_\mM \ellpi(\mM_0^k)}_F + \frac{\Gpi}{2}\norm{\mM_0^{k+1} - \mM_0^k}_F^2.
\end{align*}
Next, we can use \cref{lemma:small-distance-epoch}, with $t = m$, together with the inequality $(a+b)^2 \le 2a^2+2b^2$, to obtain an upper bound of 
\begin{align*}
& \le -\eta_k \ev{\grad_{\mM} \ellpi(\mM_0^k), \vgt^k}_F + \eta_k^2 \norm{\vr^k}_F \norm{\grad_\mM \ellpi(\mM_0^k)}_F \\
& \quad\quad + \frac{\Gpi}{2}(4m\eta_k^2\Cweight^4 C_L \norm{\Xpibn}_F^2 + \eta_k^4\Cweight^4 C_L^2 \norm{\Xpibn}_F^4).
\end{align*}

Then, because the inductive hypotheses apply we can apply \cref{cor:inductive-gradient-bound} to bound gradients, \cref{cor:inductive-correlation} to bound the inner product $\ev{\grad_{\mM} \ellpi(\mM_0^k), \vgt^k}_F$. Moreover, since $\eta_k = O(1/k^\beta)$, we can use \cref{prop:composite-noise} to bound $\norm{\vr^k}_F$. This yields an upper bound of 
\begin{align*}
    & -\frac{\eta_k}{2}\norm{\grad_{\mM} \ellpi(\mM_0^k)}_F^2 + \eta_k^2 \norm{\grad_{\mM} \ellpi(\mM_0^k)}_F \norm{\vr^k}_F + \poly(m, \Cweight, C_L, \norm{\Xpibn}_F)\eta_k^2 \\
    & \quad \quad \le \qty(-\frac{\eta_k}{2} + \frac{\eta_k^2}{2})\norm{\grad_{\mM} \ellpi(\mM_0^k)}_F^2 + \poly(m, \Cweight, C_L, \norm{\Xpibn}_F)\eta_k^2\\
    & \quad\quad \le -\frac{\eta_k}{4} \norm{\grad_{\mM} \ellpi(k)}_F^2 + \poly(m, \Cweight, C_L, \norm{\Xpibn}_F)\eta_k^2
\end{align*}
In the second line, we have used $ab \le \frac{1}{2}(a^2+b^2)$, and throughout, we have used the assumption $\eta_k \le \frac{1}{2}$ to reduce higher order terms of $\eta_k$.

Putting it together, we find that 
\[
    \ellpi(\mM_0^{k+1}) - \ellpi(\mM_0^k) \le -\frac{\eta_k}{4} \norm{\grad_\mM \ellpi(\mM_0^k)}_F^2 + \poly(m, \Cweight, C_L, \norm{\Xpibn}_F)\eta_k^2
\]

We now use $\alphapi$--strong convexity with respect to $\mM$ (and hence $\alphapi$-P\L{}) guaranteed by \cref{lemma:sc-M} to obtain 
\[
\ellpi(\mM_0^{k+1}) - \ellpi^* \le \qty(1 - \frac{\alphapi\eta_k}{2})(\ellpi(\mM_0^k) - \ellpi^*) + \poly(m, \Cweight, C_L, \norm{\Xpibn}_F)\eta_k^2.
\]
\end{proof}
One consequence of \cref{prop:one-epoch-update} is that if the stepsize $\eta_k$ is small enough, we can guarantee that the loss decreases from $\ellpi(\mM_i^k)$ to $\ellpi(\mM_i^{k+1})$. 

Next, from \Cref{prop:one-epoch-update}, we can prove the other inductive hypothesis, namely $L[0,k+1]$.
\begin{corollary}\label{cor:monotonic-loss}
Suppose $L[j,k]$ and $D[j,k]$ hold for all $j < m$ and the stepsize satisfies
\[
\eta_k \le \frac{1}{k^\beta} \min\Biggl\{\frac{1}{2}, \frac{\alphapi C_L}{\poly(m, \Cweight, C_L, \norm{\Xpibn}_F)}\Biggl\},
\]
for some $1/2 < \beta < 1$.

Then $L[0, k+1]$ holds, i.e. 
\[
\ellpi(\mM_0^{k+1}) \le C_L.
\]
\end{corollary}
\begin{proof}
Since $\eta_k \le \frac{1}{2k^\beta}$, we can apply \cref{prop:one-epoch-update} to conclude that 
\Cref{eq:one-epoch-sgd-detailed} holds. Then, for the bound $\ellpi(\mM_0^{k+1}) \le C_L$ to hold, it suffices to show that
\[
\qty(1 - \frac{\alphapi\eta_k}{2})C_L + \poly(m, \Cweight, C_L, \norm{\Xpibn}_F)\eta_k^2 \le C_L.
\]
Equivalently, 
\[
\poly(m, \Cweight, C_L, \norm{\Xpibn}_F)\eta_k \le \frac{\alphapi}{2}C_L.
\]
Clearly this holds for the stated assumption on $\eta_k$. 
\end{proof}

Finally, we show that by inductively unrolling the inequality in \cref{prop:one-epoch-update}, we can show that $\ellpi(\mM_i^k)$ converges to $\ellpi^*$ at a sublinear rate.
\begin{proposition}\label{prop:inductive-one-epoch-update}
Assume we are in the same setup as \cref{prop:one-epoch-update}. Suppose that the stepsize satisfies 
\[
\eta_k = \frac{c}{k^\beta},
\]
for some absolute constant $c$ such that $c \le \min\qty{\frac{1}{2}, \frac{2}{\alphapi}}$ and $1/2 < \beta < 1$. Further suppose that $R[0, b]$ holds for every $b \in [k+1]$. Then if $\beta < 1$ we have
\[
\ellpi(\Theta_0^{k+1}) - \ellpi^* \le (\ellpi(\Theta_0^1)-\ellpi^*)\exp(\frac{c\alphapi}{2(1-\beta)}(2-k^{1-\beta})) + \frac{c^2\poly(m, \Cweight, C_L, \norm{\Xpibn}_F)\log k}{k^{\beta}},
\]

\end{proposition}
\begin{proof}
Note that by inspecting the proof of \cref{prop:one-epoch-update} and \cref{prop:composite-noise}, the term $\poly(m, \Cweight, C_L, \norm{\Xpibn}_F)$ has no dependence on $k$. So for simplicity we will assume that this term is bounded by some absolute constant $\errorterm$. Since $\eta_k \le \frac{1}{2k^\beta}$, \cref{prop:one-epoch-update} implies that
\begin{align}
\ellpi(\mM_0^{k+1}) - \ellpi^* &\le \qty(1 - \frac{\alphapi\eta_k}{2})(\ellpi(\mM_0^k) - \ellpi^*) + \errorterm\eta_k^2
\end{align}

We can unroll the recurrence to obtain 
\begin{equation}
\ellpi(\mM_0^{k+1}) - \ellpi^* \le (\ellpi(\mM_0^1) - \ellpi^*)\prod_{t=1}^k \qty(1-\frac{\alphapi\eta_t}{2}) + \errorterm\sum_{t=1}^k \eta_t^2\qty(\prod_{j=t+1}^k \qty(1- \frac{\alphapi\eta_j}{2})).
\end{equation}

We have for any $c_t \le 1$ that
\begin{align*}
\prod_{t=a}^b \qty(1 - c_t) &\le \exp(\sum_{t=a}^b \log(1 - c_t)) \\
&\le \exp(-\sum_{t=a}^b c_t),
\end{align*}
where we have used $\log(1-x) \le -x$ for $x \le 1$. For $1/2 < \beta < 1$ we have 
\[
\sum_{t=a}^{b} \frac{1}{t^\beta} \ge \int_a^b \frac{1}{t^\beta} dt = \frac{b^{1-\beta}- a^{1-\beta}}{1-\beta}.
\]

Hence, since we assumed $\eta_k = \frac{c}{k^{\beta}}$, and $\frac{\alphapi c}{2} \le 1$, we have
\[
\prod_{t=a}^b \qty(1 - \frac{\alphapi\eta_t}{2}) \le \exp(-\frac{c\alphapi}{2(1-\beta)}(b^{1-\beta} - a^{1-\beta}))
\]

Now we can bound 
\begin{align*}
\errorterm\sum_{t=1}^k \eta_t^2\qty(\prod_{j=t+1}^k \qty(1- \frac{\alphapi\eta_j}{2})) &\le \frac{c^2 \errorterm}{k^{2\beta}} + \sum_{t=1}^{k-1} \frac{c^2 \errorterm}{t^{2\beta}}\exp(-\frac{c\alphapi}{2(1-\beta)}(k^{1-\beta} - (t+1)^{1-\beta}))
\end{align*}
Define $T \triangleq k - Ck^\beta \log k$, where $C>0$ is an absolute constant to be picked later.
We can split up the sum into $t < T$ and $t \ge T$. For the terms $t < T$ we can use concavity to deduce that 
\[
k^{1-\beta} - (t+1)^{1-\beta} \ge (1-\beta)(k-t-1)k^{-\beta} \ge \Theta(\log k),
\]
where the constant hidden in $\Theta(\log k)$ increases with $C$. Hence we pick $C$ so that for $t < T$ we have 
\[
\frac{c\alphapi}{2(1-\beta)} (k^{1-\beta} - (t+1)^{1-\beta}) \ge \beta \log k.
\]
Then, 
\begin{align*}
    \sum_{t < T} \frac{c^2 \errorterm}{t^{2\beta}} \exp(-\frac{c\alphapi}{2(1-\beta)}(k^{1-\beta} - (t+1)^{1-\beta})) & \le \exp(-\beta\log k)  \sum_{t < T} \frac{c^2 \errorterm}{t^{2\beta}}\\
    &\le O\left (\frac{c^2 \errorterm}{k^\beta} \right ).
\end{align*}

On the other hand, for the terms $t \ge T$ we can naively bound the exponential term by $1$ and obtain
\begin{align*}
\sum_{t \ge T} \frac{c^2 \errorterm}{t^{2\beta}} \exp(-\frac{c\alphapi}{2(1-\beta)}(k^{1-\beta} - (t+1)^{1-\beta})) &\le  \sum_{t \ge T} \frac{c^2 \errorterm}{t^{2\beta}} \\
&\le  \Theta(k^\beta \log k)\frac{c^2 \errorterm}{T^{2\beta}} \\
&\le \Theta\qty(\frac{c^2 \errorterm \log k}{k^{\beta}}).
\end{align*}

Hence we have that
\begin{equation}\label{eq:unrolled-sgd}
\ellpi(\mM_0^{k+1}) - \ellpi^* \le (\ellpi(\mM_0^1) - \ellpi^*)\exp(\frac{c\alphapi}{2(1-\beta)}(2-k^{1-\beta})) + \Theta\qty(\frac{c^2 \errorterm \log k}{k^{\beta}}),
\end{equation}
and the inequality in the proposition statement holds by recalling that $\errorterm = \poly(m, \Cweight, C_L, \norm{\Xpibn}_F)$. 
\end{proof}

\subsubsection{Bounding approximate invariances}\label{sec:invariances}

Armed with \Cref{cor:inductive-gradient-bound,cor:inductive-weight-bound}, we slog through the arduous task of inductively bounding the approximate invariance. As a reminder, these corollaries tell us that assuming the inductive hypotheses $L[j, k]$ and $D[j,k]$ hold, all weight norms and losses for iterate $(j,k)$ can be bounded by uniform constants.

\begin{lemma}\label{lemma:invariance-bound}
Suppose $L[j,k]$ and $D[j,k]$ hold for some $j < m$.
We have
\[
\norm{\mD_{j+1}^k - \mD_j^k}_2 \le 2\Cweight^2 C_L \norm{\bn(\Xpi^{j+1})}_2^2 \eta_k^2.
\]
Hence, if $D[t, k]$ holds for all $t \le j$, we have
\[
\norm{\mD_{j+1}^k}_2 \le 2\Cweight^2 C_L \norm{\Xpibn}_F^2 \sum_{t=1}^k\eta_t^2.
\]
\end{lemma}

\begin{proof} 
We have 
\begin{align*}
(\mW_{j+1}^k)^\top \mW_{j+1}^k  &= 
    (\mW_j^k - \eta_k \grad_\mW \loss(\Xpi^{j+1}; \Theta_j^k) )^\top (\mW_j^k - \eta_k \grad_\mW \loss(\Xpi^{j+1}; \Theta_j^k)) \\
    &= (\mW_j^k)^\top \mW_j^k - \eta_k \qty[\grad_\mW \loss(\Xpi^{j+1}; \Theta_j^k)^\top \mW_j^k + (\mW_j^k)^\top \grad_\mW \loss(\Xpi^{j+1}; \Theta_j^k)] \\
    & \quad\quad + \eta^2_k [ \textcolor{orange}{\grad_\mW \loss(\Xpi^{j+1}; \Theta_j^k)^\top \grad_\mW \loss(\Xpi^{j+1}; \Theta_j^k)}]
\end{align*}

Similarly, we have 
\begin{align*}
(\mGamma_{j+1}^k)^2  &= 
    (\mGamma_j^k - \eta_k \grad_{\mGamma} \loss(\Xpi^{j+1}; \Theta_j^k) ) (\mGamma_j^k - \eta_k \grad_\mW \loss(\Xpi^{j+1}; \Theta_j^k)) \\
    &= (\mGamma_j^k)^2  - \eta_k \qty[\grad_{\mGamma} \loss(\Xpi^{j+1}; \Theta_j^k) \mGamma_j^k + \mGamma_j^k \grad_{\mGamma} \loss(\Xpi^{j+1}; \Theta_j^k)] \\
    & \quad\quad + \eta^2_k [ \textcolor{orange}{\grad_{\mGamma} \loss(\Xpi^{j+1}; \Theta_j^k)\grad_{\mGamma} \loss(\Xpi^{j+1}; \Theta_j^k)}]
\end{align*}
The gradient invariance in \Cref{fact:invariances} cancels out the $\eta_k$ term in $\mD_{j+1}^k = \mI +  \diag((\mW_{j+1}^k)^\top \mW_{j+1}^k - (\mGamma_{j+1}^k)^2$. Hence, if we take the operator norm of $\mD_{j+1}^k - \mD_j^k$ and use \cref{lemma:hadamard}, we can ignore the diagonal operator. Then, since the inductive hypotheses hold, we can apply the inductive gradient bound (\cref{cor:inductive-gradient-bound}) to obtain
\begin{align*}
    \norm{\mD_{j+1}^k - \mD_j^k}_2 
    & \le \eta_k^2 \qty[\textcolor{orange}{\norm{\grad_\mW \loss(\Xpi^{j+1}; \Theta_j^k)}_2^2 +\norm{\grad_{\mGamma} \loss(\Xpi^{j+1}; \Theta_j^k)}_2^2}] \\
    &\le 2\Cweight^2 C_L \norm{\bn(\Xpi^{j+1})}_2^2 \eta_k^2.
\end{align*}

To conclude, we apply triangle inequality and the inductive hypothesis on $D[t, k]$, yielding
\begin{align*}
\norm{\mD_{j+1}^k}_2 &\le \norm{\mD_0^k}_2 + 2\Cweight^2 C_L \eta_k^2 \sum_{t=0}^{j} \norm{\bn(\Xpi^{t+1})}_2^2\\
&\le \norm{\mD_0^k}_2 + 2\Cweight^2 C_L \eta_k^2 \norm{\Xpibn}_F^2 \\
&\le 2\Cweight^2 C_L \norm{\Xpibn}_F^2 \sum_{t=1}^k \eta_t^2.
\end{align*}
\end{proof}

As a corollary, we see that for stepsizes of the form $\eta_k = c/k^\beta$ for $1/2 < \beta < 1$, we can select $c$ to guarantee that $\norm{\mD_j^k}_2 \le 1/2$ for all $(j, k)$. 
\begin{corollary}\label{cor:inductive-approximate-invariances}
Assume that $D[t, k]$ holds for all $t \le j < m$ and $\eta_k = \frac{c}{k^\beta}$ for $1/2 < \beta < 1$. If 
\[
c^2 \le \frac{1}{4(1+\frac{1}{2\beta - 1})\Cweight^2 C_L\norm{\Xpibn}_F^2},
\]
then 
\[
\norm{\mD_{j+1}^k}_2 \le 2\Cweight^2 C_L \norm{\Xpibn}_F^2 \sum_{t=1}^k\eta_t^2 \le \frac{1}{2}.
\]
In other words, $D[j+1, k]$ holds.
\end{corollary}
\begin{proof}
\Cref{lemma:invariance-bound} implies that 
\[
\norm{\mD_{j+1}^k}_2 \le 2\Cweight^2 C_L \norm{\Xpibn}_F^2 \sum_{t=1}^k\eta_t^2 \le 2\Cweight^2 C_L \norm{\Xpibn}_F^2 \sum_{t=1}^{\infty}\eta_t^2. 
\]
For $1/2 < \beta < 1$, we have
\[
\sum_{k=1}^{\infty} \frac{1}{k^{2\beta}} \le 1 + \int_1^{\infty} \frac{1}{t^{2\beta}}dt = 1 + \frac{1}{1-2\beta}\eval{t^{1-2\beta}}_{1}^{\infty} = 1 + \frac{1}{2\beta-1}.
\]
Hence, if
\[
c^2 \le \frac{1}{4(1+\frac{1}{2\beta - 1})\Cweight^2 C_L\norm{\Xpibn}_F^2},
\]
then evidently $\norm{\mD_{j+1}^k}_2 \le \frac{1}{2}$,
as desired.
\end{proof}

\subsubsection{Completing the induction}\label{sec:ss-induction}
With all the pieces in place, we formally state the theorem for SS convergence.
\begin{theorem}[Formal statement of convergence for SS]\label{thm:ss-convergence-formal}
Let $\pi \in \sS_n$ be such that \ref{assumption:full-rank-ss} holds. Let $f(\cdot; \Theta) = \mW\mGamma \bn(\cdot)$ be a 2-layer linear+BN network initialized at $\Theta_0^1 = (\mW_0^1,\mGamma_0^1) = (\vzero,\mI)$.  Consider optimizing $f$ using SS with permutation $\pi$ and decreasing stepsize 
\begin{align*}
    \eta_k = \frac{1}{k^\beta} \cdot \min\Biggl\{&\frac{1}{2}, \frac{2}{\alphapi},
    \frac{\alphapi C_L}{\poly(m, \Cweight, C_L, \norm{\Xpibn}_F)}, \sqrt{\frac{1}{4(1+\frac{1}{2\beta - 1})\Cweight^2 C_L\norm{\Xpibn}_F^2}}\Biggl\}
\end{align*}
for any $1/2 < \beta < 1$. Then the SS risk satisfies
\begin{align*}
\ellpi(\Theta_0^{k+1}) - \ellpi^* & \le\frac{\poly(m, d, C_L, \norm{\Xpibn}_F, \frac{1}{\sigma_{\min}(\Xpibn^\top)})\log k}{k^{\beta}}.
\end{align*}
In particular, the SS risk converges to the global optimal risk $\ellpi^*$.
\end{theorem}

\begin{proof}[Proof of \Cref{thm:ss-convergence-formal}]
We proceed by induction on the epoch. We restate the key inductive statements to prove:(with appropriate selection of $\eta_k$):
\begin{itemize}
    \item $L[j,k]$: $\ellpi(\Theta_j^k)$ stays bounded above by some uniform constant $C_L \ge \norm{\Ypi}_F^2$ --- this is the content of \cref{cor:monotonic-loss}.
    \item $R[j, k]$: $\ellpi(\Theta_j^k)$ satisfies the per-epoch loss inequality --- this is the content of \cref{prop:one-epoch-update}.
    \item $D[j,k]$: The approximate invariances stay bounded in norm away from 1. More precisely,  $\norm{\mD_j^k}_2 \le 2\Cweight^2 C_L\norm{\Xpibn}_F^2\sum_{t =1}^{\infty} \eta_t^2 \le \frac{1}{2}$ --- this is the content of \cref{cor:inductive-approximate-invariances}.
\end{itemize}

Notice that the assumptions on $\eta_k$ exactly satisfy the hypotheses of \cref{cor:monotonic-loss,prop:inductive-one-epoch-update,cor:inductive-approximate-invariances}.

The base cases follows from the initialization. Recall that we set  
\[
C_L = \max\qty{\ellpi(\Theta_t^1): 0 \le t \le m-1}.
\]
Since we only look at the loss values for the first epoch, $C_L$ is indeed an absolute constant depending on $\pi$. So $L[j, 1]$, as defined in \cref{hyp:loss-bounded}, holds for all $j < m$. Next, we have from the initialization that $\mD_0^1 = \vzero$. Then, applying \cref{lemma:invariance-bound}, we can conclude $D[j, 1]$ also holds for all $j < m$, so \cref{hyp:approx-invariances} holds. Finally, there is no need to check \cref{hyp:risk-update} because it is only defined for $k>1$. 

Now our inductive hypotheses are that $L[j, k]$, $R[j,k]$, $D[j,k]$ hold for all $j < m$. For the inductive step, we want to prove $L[0, k+1]$, $R[0, k+1]$, $D[0, k+1]$. By construction of $\eta_k$, the hypotheses of \cref{cor:monotonic-loss} is satisfied, so $L[0, k+1]$ holds. Moreover, the hypotheses of \cref{prop:one-epoch-update} are satisfied, so $R[0, k+1]$ holds. Finally, the hypotheses of \cref{lemma:invariance-bound} are satisfied, so that $D[0,k+1]$ holds.

As asserted earlier, the above argument is robust up to reindexing if we want to prove the statement for $i > 0$. In particular, all of the results of the previous section go through, as we showed in \cref{prop:composite-noise} that the stepsize noise is negligible. Hence the induction is completed for all $(i,k)$ and so the unrolled update equation in \cref{prop:inductive-one-epoch-update} holds for all $k$. This gives the formal rate of convergence for the stated stepsize. In particular, we see that the SS risk $\ellpi$ converges to its global minimum.
\end{proof}

\begin{proof}[Proof of \cref{thm:ss-convergence}]
All we need to do is convert the stepsizes requirements. Examining the stepsize requirements in \cref{thm:ss-convergence-formal}, they depend on $C_L$, $\Cweight$, and $\norm{\Xpibn}_F$. 

Now recall the definition of $\Cweight$ in \cref{cor:inductive-weight-bound}:
\[
\Cweight^2 \le \frac{3}{2} + d^2\qty(\frac{1}{2} + \frac{C_L^{1/2} + \norm{\Ypi}_F}{\sigma_{\min}(\Xpibn^\top)}).
\]
Hence $\Cweight = \poly(d, C_L, \norm{\Ypi}_F, 1/\sigma_{\min}(\Xpibn^\top))$. Finally, since $\norm{\Xpi}_F^2 = dn$, $\sigma_{\min}(\Xpibn \Xpibn^\top) = \sigma_{\min}(\Xpibn^\top)^2$, and $C_L$ is an absolute constant, we can convert the stepsize requirements into 
\begin{align*}
    \eta_k = \frac{1}{k^\beta} \cdot \min\Biggl\{&\frac{1}{2}, \frac{2}{\sigma_{\min}(\Xpibn \Xpibn^\top)},
    \frac{\sqrt{2\beta-1} \poly(\sigma_{\min}(\Xpibn^\top))}{\poly(n, d, \norm{\Ypi}_F)}\Biggl\},
\end{align*}
which directly implies the stepsize requirements in  \cref{thm:ss-convergence}.
\end{proof}

\subsection{Proof of convergence for RR}\label{app:rr-proof}
In this section, we prove \cref{thm:rr-convergence}.
With RR, we randomly resample permutation $\pi_k \in \sS_n$ on epoch $k$. Hence, it is natural to seek a convergence bound in expectation. We briefly comment on the complications that arise in this setting.

Since we want to prove convergence an expectation, an inductive approach that controls the approximate invariances and loss evolution is complicated by the necessity for bounds on these quantities that are stronger than merely being in expectation. This is precisely why we need \Cref{assumption:compact}.  

\paragraph{Additional notation.}
As introduced in \cref{sec:setup}, we can view RR as optimizing the risk
\[
\ellrr(\Theta) \triangleq \frac{1}{n!} \sum_{\pi \in \sS_n} \ellpi(\Theta)
\]
via with-replacement SGD on an \emph{epoch level} (i.e., $\ellpi$ is sampled uniformly with replacement at every epoch), albeit with noise terms due to the shuffling algorithm. Motivated by the setup in \cref{sec:setup}, we can also write $\ellrr$ in an equivalent form to $\ellgd$ as follows:
\begin{align*}
\ellrr(\Theta) \triangleq \frac{1}{n!}
\sum_{\pi \in \sS_n} \mathcal L(f(\Xpi;\Theta),\Ypi)
&= \frac{1}{n!}\mathcal L(\mW \mGamma \bnrr(\mX), \Yrr), \\
\text{where } \bnrr (\mX) &\triangleq \begin{bmatrix}\ol{\mX}_{\pi_1} & \ldots  & \cdots & \ol{\mX}_{\pi_{n!}} \end{bmatrix} \in \RR^{d \times (n \cdot n!)}, \\
\Yrr &\triangleq \begin{bmatrix} \mY_{\pi_1} & \cdots & \mY_{\pi_{n!}} \end{bmatrix} \in \RR^{p \times (n \cdot n!)}.
\end{align*}
For notational convenience, we also write $\Xrrbn \triangleq \bnrr(\mX)$, and reiterate that overlines indicate the presence of batch normalization. 

Just as in the SS case, we will abuse notation and refer to the RR risk function as a function of $\mM = \mW\mGamma$ by writing $\ellrr(\mM) \triangleq  \frac{1}{n!} \sum_{\pi \in \sS_n} \ellpi(\mM)$. Similarly, we will often refer to the gradient of the RR risk with respect to $\mM$ as $\grad_{\mM} \ellrr(\mM) \triangleq \frac{1}{n!} \sum_{\pi \in \sS_n} \grad_{\mM} \ellpi(\mM)$. We will find it helpful to use the notation $\Xmax \triangleq \argmax_{\pi \in \sS_n} \norm{\Xpibn}_2$. Similarly we will denote the maximum Frobenius norm batch normalized dataset by $\Xmaxf \triangleq \argmax_{\pi \in \sS_n} \norm{\Xpibn}_F$. Furthermore, it follows from the unit variance constraint in the definition of $\bn$ that $\norm{\Xmax}_2 \le \norm{\Xmaxf}_F \le \sqrt{dn}$.

At a high level, the RR proof of convergence closely follows the SS proof of convergence. Indeed, most of the technical legwork has already been fleshed out in the SS case --- most of the results port over immediately, taking care to replace $\pi$ with $\pi_k$. However, we will be careful in accounting for where we need to deviate from the SS logic.

\subsubsection{Checking optimization properties}
We first check the smoothness and strong convexity property with respected to the merged matrix $\mM$ that we heavily relied on for the proof of convergence for SS. 
\begin{fact}[Smoothness of RR]
Define 
\[\Grr \triangleq \frac{1}{n!} \sum_{\pi \in \sS_n} \Gpi.\]
Then 
$\ellrr(\mM)$ is $\Grr$-smooth with respect to $\mM$.
\end{fact}
\begin{proof}
The statement follows from combining \cref{lemma:smoothness-M} with the fact that if $f_i$ are $G_i$-smooth, then $\sum_{i=1}^n f_i$ is $\sum_{i=1}^n G_i$-smooth.
\end{proof}
As before, we cannot directly use a P\L{} inequality on $\mW$ or $\mGamma$; we must instead bootstrap this from the strong convexity of the risk with respect to $\mM = \mW\mGamma$.
\begin{fact}[Strong convexity]\label{lemma:strong-convexity}
Suppose that \ref{assumption:full-rank-rr} holds for some $\pi \in \sS_n$. Then the loss function $\ellrr(\mM) = \frac{1}{n!} \sum_{\pi \in \sS_n} \ellpi(\mM)$ is $\alpharr$-strongly convex with respect to $\mM$ with $\alpharr \triangleq \frac{1}{n!}\sum_{\pi} \sigma_{\min}(\Xpibn \Xpibn^\top) = \frac{1}{n!}\sum_{\pi} \alphapi$. 
\end{fact}
\begin{proof}
Take the Hessian of $\ellrr$ with respect to $\vec(\mM)$ to obtain $\grad^2_{\vec(\mM)} \ellrr(\mM) = \frac{1}{n!} \sum_{\pi \in \sS_n} \Xpibn \Xpibn^\top \otimes \mI$. Hence we have $\grad^2_{\vec(\mM)} \ellrr(\mM) \succeq \frac{1}{n!} \sum_\pi \sigma_{\min}(\Xpibn \Xpibn^\top)$. Since we assumed \ref{assumption:full-rank-rr} holds, the sum is strictly positive, so $\ellrr$ is indeed strongly convex.
\end{proof}

\subsubsection{Proof sketch of RR convergence}
We first state the modified inductive hypothesis for the one-epoch risk update, which replaces $R[j,k]$.
\begin{hypothesis}\label{hyp:rr}
For $k > 1$, the inductive hypothesis $S[k]$ states that 
\begin{align*}
\EE_{\pi_{k-1}}[\ellrr(\mM_0^{k}) | \mathcal{F}_{k-1}] - \ellrr^* &\le \qty(1 - \frac{\alpharr \eta_{k-1}}{2}) (\ellrr(\mM_0^{k-1}) - \ellrr^*) \\
&\quad\quad + \eta_{k-1}^2 \poly(m, \Bweight, \Bloss, \norm{\Xmaxf}_F),
\end{align*}
where $\EE_{\pi_{k-1}}$ denotes the expectation over random draws of the permutation $\pi_{k-1}$.
\end{hypothesis}

\begin{proof}[Proof sketch of \Cref{thm:rr-convergence}]
As before, we start by writing out the smoothness inequality with respect to $\mM$:
\begin{equation}\label{ineq:rr-smoothness}
\ellrr(\mM_0^{k+1}) \le \ellrr(\mM_0^k) + \ev{\grad_{\mM} \ellrr(\mM_0^k), \mM_0^{k+1} - \mM_0^k} + \frac{\Grr}{2}\norm{\mM_0^{k+1} - \mM_0^k}^2.
\end{equation}

Next, we have the same gradient update 
\[
\mM_0^{k+1} = \mM_0^k - \eta_k\vgt^k + \eta_k^2\vr^k,
\]
but now all quantities involving $\pi$ turn into $\pi_k$. For example,
we redefine
\[
\vgt^k \triangleq \sum_{t=0}^{m-1} \vgt_t^k = \grad_{\mW} \loss_{\pi_k}(\mM_0^k) \mGamma_0^k + \mW_0^k \grad_{\mGamma} \loss_{\pi_k}(\mM_0^k).
\]

Since we want to prove convergence in expectation, it is standard to consider the natural filtration $\mathcal{F}_k$ of which permutations we have picked up to (but not including) epoch $k$. Formally, $\mathcal{F}_k = \sigma(\pi_1, \ldots, \pi_{k-1})$, where $\sigma(Z)$ denotes the $\sigma$-algebra generated by the random variable $Z$. 

Noting the identity
\[
\EE_{\pi}[\grad_\mM \ellpi(\mM_0^k)] = \grad_\mM \ellrr(\mM_0^k),
\]
it follows that if we can apply \cref{cor:inductive-correlation}, then 
\[
\EE_{\pi_k}\qty[\ev{\grad_{\mM} \ellrr(\mM_0^k), \vgt^k}_F] = \ev{\grad_{\mM} \ellrr(\mM_0^k), \EE_{\pi_k}[\vgt^k]}_F \ge \frac{1}{2} \norm{\grad_{\mM} \ellrr(\mM_0^k)}_F^2.
\]

Hence, we can follow the same argument for upper bounding the smoothness inequality for SS and  take a conditional expectation over $\pi_k$ conditioned on $\mathcal{F}_k$. Assuming for now that the weight norms are bounded by some absolute constant $\Bweight$ and the relevant losses are bounded by an absolute constant $\Bloss$, this yields 
\[
\EE_{\pi_k}[\ellrr(\mM_0^{k+1}) | \mathcal{F}_k] \le \ellrr(\mM_0^k) - \frac{\eta_k}{4} \norm{\grad_\mM \ellrr(\mM_0^k)}_F^2 + \eta_k^2 \EE_{\pi_k}[\poly(m, \Bweight, \Bloss, \norm{\mX_{\pi_k}}_F)].
\]

Noting that we can upper bound $\norm{\mX_{\pi_k}}_F$ uniformly by $\norm{\Xmaxf}_F$, which does not depend on $k$, this shows that 
we have 
\[
\EE_{\pi_k}[\ellrr(\mM_0^{k+1}) | \mathcal{F}_k] \le \ellrr(\mM_0^k) - \frac{\eta_k}{4} \norm{\grad_\mM \ellrr(\mM_0^k)}_F^2 + \eta_k^2 \poly(m, \Bweight, \Bloss, \norm{\Xmaxf}_F).
\]

Next, $\alpharr$--strong convexity yields
\[
\EE_{\pi_k}[\ellrr(\mM_0^{k+1}) | \mathcal{F}_k] - \ellrr^* \le \qty(1 - \frac{\alpharr \eta_k}{2}) (\ellrr(\mM_0^k) - \ellrr^*) + \eta_k^2 \poly(m, \Bweight, \Bloss, \norm{\Xmaxf}_F).
\]

This is exactly the statement of $S[k+1]$. To proceed, we must fill in the following details. First, we must show that the relevant losses are bounded by $\Bloss$ --- see \cref{cor:uniform-loss-bound}. Then, we show that the weight norms are bounded by $\Bweight$ --- this is shown in \cref{cor:uniform-norm-bound}. Finally, we must show that $D[j,k]$ holds, i.e., bound the approximate invariances --- this is the content of \cref{cor:inductive-rr-approximate-invariances}. Once we address these technicalities, an inductive argument similar to the SS version proves the theorem. Note that the SS inductive hypotheses $L[j,k]$ and $R[j,k]$ are not active in this proof. 
\end{proof}
\subsubsection{Weight bounds}
In this section we elucidate the connection between weight norms and the loss evolution. In particular, we show that bounds on the weight norms confer a bound on the loss function value. First, we state the following inequality, which follows from a quick application of Cauchy-Schwarz.
\begin{fact}\label{fact:cs-rr}
We have $\frac{1}{n!} \sum_{\pi \in \sS_n} (\ellpi(\mM))^{1/2} \le \ellrr(\mM)^{1/2}$.
\end{fact}

With this identity in hand, we derive the following corollaries about weight and gradient bounds. 
\begin{proposition}\label{prop:rr-gradient-upper-bound}
We have 
\[
\norm{\grad_\mM \ellrr(\mM)} \le \norm{\Xmax}_2 \ellrr(\mM)^{1/2}.
\]
\end{proposition}
\begin{proof}
We have $\norm{\grad_{\mM} \ellrr(\mM)}_2 \le \frac{1}{n!} \sum_{\pi} \norm{\grad_{\mM} \ellpi(\mM)}_2$ by the triangle inequality. Applying the individual gradient bounds in \cref{lemma:grad-upper-bound} and uniformly bounding $\norm{\Xpibn}$ by $\norm{\Xmax}$, we see that 
\[
\norm{\grad_\mM \ellrr(\mM)}_2 \le \norm{\Xmax}_2 \cdot \frac{1}{n!} \sum_{\pi} \ellpi^{1/2}(\mM).
\]
Applying \Cref{fact:cs-rr} yields the desired result.
\end{proof}

As promised, we quantify the relationship between $\ellrr$, $\ellpi$, and $\norm{\mM}_2$ with the following proposition. 
\begin{proposition}\label{prop:weight-loss-equiv}
Let $\sigma_0 \triangleq \frac{1}{n!} \sum_{\pi} \sigma_{\min}(\Xpibn^\top)$. 
We have 
\[
\frac{\ellrr(\mM) - 2\norm{\mY}_F^2}{2\norm{\Xmaxf}_F^2}\le \norm{\mM}_2^2 \le \qty(\frac{\norm{\mY}_F + \ellrr(\mM)^{1/2}}{\sigma_0})^2.
\]
Similarly, for any $\pi \in \sS_n$, we have
\[
\frac{\ellpi(\mM) - 2\norm{\mY}_F^2}{2\norm{\Xpibn}_F^2}\le \norm{\mM}_2^2 \le \qty(\frac{\norm{\mY}_F + \ellpi(\mM)^{1/2}}{\sigma_{\min}(\Xpibn^\top)})^2.
\]
\end{proposition}
\begin{proof}
First, we have 
\begin{align*}
\frac{1}{n!} \sum_{\pi \in \sS_n} \norm{\mM\Xpibn}_2 &\le \norm{\mY}_F + \frac{1}{n!} \sum_{\pi} \norm{\Ypi - \mM\Xpibn}_F \\
&= \norm{\mY}_F + \frac{1}{n!} \sum_{\pi} \ellpi(\mM)^{1/2},
\end{align*}
where we have used in the first line the fact that $\norm{\Ypi}_F = \norm{\mY}_F$ for all $\pi$. 
Therefore by \cref{fact:cs-rr} and using $\norm{\mM\Xpibn}_2 \ge \sigma_{\min}(\Xpibn^\top)\norm{\mM}_2$, we find that
\[
\frac{1}{n!} \sum_{\pi} \sigma_{\min}(\Xpibn^\top) \norm{\mM}_2 \le \norm{\mY}_F + \ellrr(\mM)^{1/2}.
\]
It follows that 
\[
\norm{\mM}_2 \le \frac{\norm{\mY}_F + \ellrr(\mM)^{1/2}}{\sigma_0}.
\]

For the other direction, note that 
\[
\ellpi(\mM) = \norm{\Ypi-\mM\Xpibn}_F^2 \le 2\norm{\mY}_F^2 + 2\norm{\mM}_2^2\norm{\Xpibn}_F^2.
\] 
Averaging over all $\pi \in \sS_n$ gives us
\[
\ellrr(\mM) \le 2\norm{\mY}_F^2 +  \frac{2\norm{\mM}_2^2}{n!}\sum_{\pi} \norm{\Xpibn}_F^2.
\]
Uniformly bounding $\norm{\Xpibn}_F^2$ by $\norm{\Xmaxf}_F^2$ and rearranging yields the desired result.

The set of inequalities for $\pi$ also follow by a similar argument.
\end{proof}

As a corollary of \Cref{prop:weight-loss-equiv}, it follows from \Cref{assumption:compact} that each of the losses $\ellpi(\mM_i^k)$ stay bounded by an absolute constant throughout training.
\begin{corollary}[Uniform bound on SS losses]\label{cor:uniform-loss-bound}
Under \Cref{assumption:compact}, for every $\pi \in \sS_n$ we have 
\[
\ellpi(\mM_i^k) \le \Bloss,
\]
where 
\[
\Bloss \triangleq 2\norm{\mY}_F^2 + 2\Brr^2\norm{\Xmaxf}_F^2.
\]
Here, $\Brr$ was previously defined in \Cref{assumption:compact}.
\end{corollary}

Finally, \Cref{assumption:compact} implies the following inductive statement about the weight norms.
\begin{corollary}[Uniform bound on weight norms]\label{cor:uniform-norm-bound}
Assume \Cref{assumption:compact} and $D[j,k]$ holds. Then for RR, we have
\[
\max\qty{\norm{\mW_i^k}_2, \norm{\mGamma_i^k}_2} \le \Bweight,
\]
where 
\[
\Bweight^2 \triangleq \frac{3}{2} + d^2\qty(\frac{1}{2} + \Brr).
\]
Here, $\Brr$ was previously defined in \Cref{assumption:compact}.
\end{corollary}
\subsubsection{Bounding approximate invariances}
In this section we formally bound the approximate invariances throughout RR training under \Cref{assumption:compact}. In particular, as a consequence of \cref{cor:uniform-loss-bound,cor:uniform-norm-bound}, the following two claims follow almost directly upon inspection of the proofs of \cref{lemma:invariance-bound,cor:inductive-approximate-invariances}. 
\begin{lemma}\label{lemma:rr-invariance-bound}
Suppose $D[j,k]$ holds for some $j < m$.
We have
\[
\norm{\mD_{j+1}^k - \mD_j^k}_2 \le 2\Bweight^2 \Bloss \norm{\bn(\mX_{\pi_k}^{j+1})}_2^2 \eta_k^2.
\]
Hence, if $D[t, k]$ holds for all $t \le j$, we also have
\[
\norm{\mD_{j+1}^k}_2 \le 2\Bweight^2 \Bloss \norm{\Xmaxf}_F^2 \sum_{t=1}^k\eta_t^2.
\]
\end{lemma}

\begin{corollary}\label{cor:inductive-rr-approximate-invariances}
Assume that $D[t, k]$ holds for all $t \le j < m$ and $\eta_k = \frac{c}{k^\beta}$ for $1/2 < \beta < 1$. If 
\[
c^2 \le \frac{1}{4(1+\frac{1}{2\beta - 1})\Bweight^2 \Bloss\norm{\Xmaxf}_F^2},
\]
then 
\[
\norm{\mD_{j+1}^k}_2 \le 2\Bweight^2 \Bloss \norm{\Xmaxf}_F^2 \sum_{t=1}^k\eta_t^2 \le \frac{1}{2}.
\]
In other words, $D[j+1, k]$ holds.
\end{corollary}
\subsubsection{Completing the proof}
With the connection between the RR loss function and weight bounds in hand, we can complete the proof of convergence for RR. Finally, we formally state the RR convergence result.
\begin{theorem}[Formal statement of convergence for RR]\label{thm:rr-convergence-formal}
Suppose \ref{assumption:full-rank-rr} and \Cref{assumption:compact} hold. Let $f(\cdot; \Theta) = \mW\mGamma \bn(\cdot)$ be a 2-layer linear+BN network initialized at $\Theta_0^1 = (\mW_0^1,\mGamma_0^1) = (\vzero,\mI)$.  Consider optimizing $f$ using RR with decreasing stepsize 
\begin{align*}
    \eta_k = \frac{1}{k^\beta} \cdot \min\Biggl\{&\frac{1}{2}, \frac{2}{\alpharr}, \sqrt{\frac{1}{4(1+\frac{1}{2\beta - 1})\Bweight^2 \Bloss\norm{\Xmaxf}_F^2}}\Biggl\}
\end{align*}
for any $1/2 < \beta < 1$. Then the RR risk satisfies
\begin{align*}
\EE[\ellrr(\Theta_0^{k+1})] - \ellrr^* & \le\frac{\poly(n, d, \Brr, \norm{\mY}_F)\log k}{k^{\beta}}.
\end{align*}
In other words, the RR risk converges to the global optimal risk $\ellrr^*$.
\end{theorem}
\begin{proof}[Proof of \Cref{thm:rr-convergence-formal}]
We inductively prove that $D[j,k]$ and $S[k]$ hold. There is no need to check $S[1]$ because \Cref{hyp:rr} is defined for $k > 1$. The proof that the base cases $D[j, 0]$ all hold follows the same proof as that in the SS case.

Now suppose for the sake of induction that $D[j,k]$ and $S[k]$ hold. We will show that $D[j+1, k]$ holds. Once we show that $D[j,k]$ holds for all $j < m$, we can then show that $S[k+1]$ holds.

In particular, by the assumption on $\eta_k$, \Cref{cor:inductive-rr-approximate-invariances} implies that $D[j+1,k]$ holds. Hence by induction $D[j,k]$ holds for all $j < m$. We see that assuming \Cref{assumption:compact} simplified the proof strategy significantly, as we did not have to go through the trouble of proving $L[j,k]$.

Next, let's understand what happens to the per-epoch loss bound in $S[k+1]$. Explicitly, we can follow the same steps as in the proof sketch --- which only required $\eta_k \le \frac{c}{k^\beta}$ where $c \le \min\qty{\frac{1}{2}, \frac{2}{\alpharr}}$ --- to see that
\[
\EE_{\pi_k}[\loss_{\pi_k}(\mM_0^{k+1})|\mathcal{F}_k] - \ellpi^* \le \qty(1 - \frac{\alphapi\eta_k}{2})(\ellrr(\mM_0^k) - \ellrr^*) + \poly(m, \Bweight, \Bloss, \norm{\Xmaxf}_F)\eta_k^2.
\]

Indeed, under \Cref{assumption:compact} and $D[j,k]$ for $j < m$, we can apply \cref{cor:uniform-loss-bound,cor:uniform-norm-bound} to rigorously bound all of the analogous noise terms a.s.. 
We can then follow the same argument as in \cref{prop:inductive-one-epoch-update} to unroll the recurrence, using iterated expectation to obtain a total expectation in the end. We can thus conclude that with $\eta_k = \frac{c}{k^\beta}$ for $1/2 < \beta < 1$ and the constant $c$ chosen as the theorem statement that
\begin{align*}
\EE[\ellrr(\Theta_0^{k+1})] - \ellrr^* & \le (\ellrr(\Theta_0^1)-\ellpi^*)\exp(\frac{c\alpharr}{2(1-\beta)}(2-k^{1-\beta})) \\
&\quad\quad + \frac{\poly(m, \Bweight, \Bloss, \norm{\Xmaxf}_F)\log k}{k^{\beta}} \\
& \le (\ellrr(\Theta_0^1)-\ellpi^*)\exp(\frac{c\alpharr}{2(1-\beta)}(2-k^{1-\beta})) \\
&\quad\quad + \frac{\poly(n, d, \Brr, \norm{\mY}_F)\log k}{k^{\beta}}.
\end{align*}
In the last line we used the fact that $\Bloss = \poly(\Brr, \norm{\mY}_F, \norm{\Xmaxf}_F)$ from \cref{cor:uniform-loss-bound}, $\Bweight = \poly(d, \Brr)$ from \cref{cor:uniform-norm-bound}, and $\norm{\Xmaxf}_F^2 = dn$.
The desired claim immediately follows.

We can also immediately see how the stepsize requirements match that of \cref{thm:rr-convergence}.
\end{proof}

\section{Proofs for classification results}\label{app:classification}
In this section we lay out the groundwork for formally proving our main results \cref{prop:ss-structure,prop:rr-structure} about the separability decomposition of SS+BN and RR+BN (cf. \cref{thm:ss-divergence-informal,thm:rr-divergence-informal}). At a high level, we show that the separability decomposition is closely linked to the presence of monochromatic batches (\cref{lemma:batch-separability}) and the dimensionality of the batch normalized dataset (\cref{lemma:sc-sufficient}). In \cref{sec:characterize-direction}, we formally characterize the optimal directions of linear+BN classifiers. We defer the proofs of the 
more technical lemmas to \cref{sec:cls-proofs}. 

\subsubsection*{Additional notation and setup}
We lay out some additional notation that will aid in our discussion of classification. 
Division of two vectors should be interpreted in a coordinatewise fashion, so $\frac{\vmu}{\vsigma} \in \RR^d$ with $k$th coordinate $\mu_k/\sigma_k$. For a matrix $\mA \in \RR^{d \times n}$, we define $\norm{\mA}_{2,\infty} = \max_{i \in [n]} \norm{\mA_{:, i}}_2$, i.e., the maximum Euclidean norm of the columns. 

We remind the reader of some notation introduced previously, with an important redefinition for $\Xrrbn$. For a dataset $\mZ = (\mX, \mY)$, we write $\mZ^+ = (\mX^+, \mY^+)$ and $\mZ^- = (\mX^-, \mY^-)$ to denote the positive and negative examples, respectively. Recall that we write the dataset batch normalized under a permutation $\pi$ as $\Zpi \triangleq (\Xpibn, \Ypi)$, where $\Xpibn \triangleq \bnpi(\mX) \in \RR^{d \times n}$ and $\Ypi = \pi \circ \mY$.
We also use $\Xpibn^+$ and $\Xpibn^-$ to denote the submatrices of $\Xpibn$ containing its columns corresponding to the positive and negative examples, respectively. 

For the sake of analyzing the rank of $\Xrrbn$, we will redefine it as follows by throwing out redundant batches. Let $\binom{[n]}{B}$ denote the set of all $\binom{n}{B}$ unique (up to permutation) batches of size $B$ that can be created from choosing the columns of $\mX \in \RR^{d \times n}$. Fix an arbitrary labelling of these $\binom{n}{B}$ batches, and let $\mB^j \in \RR^{d \times B}$ refer to the $j$th such batch. Then 
\begin{align*}
    \Xrrbn &\triangleq \bnrr(\mX) \triangleq \mqty[\bn(\mB^1) & \ldots & \bn(\mB^{\binom{n}{B}})] \in \RR^{d \times B\binom{n}{B}}
\end{align*}
Note that the rank of $\Xrrbn$ is the same as the rank of the original definition, since all we did was throw out redundant batches for the purposes of analyzing the rank. 

We now turn to laying down some of the background necessary to introduce our technical results. As a motivating step, recall \cref{prop:divergence-possibility}. It states that SS with permutation $\pi$ can cause divergence of the GD risk if $\Zpi$ is PLS or LS, but not if it is SC. Hence, determining sufficient conditions for when $\Zpi$ is SC is of primary interest. Intuitively, $\Zpi$ being SC should be related to some notion of genericity --- the convex hulls of positive and negative features should be full dimensional. To formalize this intuition, we take a quick detour and recall several standard notions in convex analysis, defined for example in \citet{boyd2004convex}. 

For $S \subseteq \RR^d$, its interior $\interior(S)$ denotes the largest open set contained in $S$. We say that $S$ is \emph{affine} if for any $\vx_1, \vx_2 \in S$, the line $\lambda \vx_1 + (1-\lambda) \vx_2 \subseteq S$. An \emph{affine combination} of $k$ points $\vx_1, \ldots, \vx_k \in \RR^d$ is given by $\sum_{i=1}^k \lambda_i \vx_i$ where $\sum_{i=1}^k \lambda_i = 1$.\footnote{Here, unlike the definition of a convex combination, the $\lambda_i$'s are allowed to be negative.} The affine hull of $S$ is the set of all affine combinations of $S$, and is denoted by $\aff(S)$, and clearly $\aff(S)$ is an affine set. Similarly, the relative interior $\relinterior(S)$ denotes the largest open subset of $\aff(S)$ contained in $S$. For a matrix $\mA$, we slightly abuse notation and write $\aff(\mA)$ to denote the affine hull of its columns. Similarly, we use $\conv(\mA)$ to denote the convex hull of its columns.

Let $\vx_0 \in \aff(S)$ be any element of the affine hull of $S$. It is not hard to see that $\aff(S) = \vx_0 + V$, where $V$ is a linear subspace of $\RR^d$. Furthermore, this $V$ is uniquely determined by $S$. One can think of $\vx_0$ as an offset and $V$ as the space of valid directions to move in to stay in $\aff(S)$. We define $\dim(S) \triangleq \dim(V)$. 

\begin{definition}
A set $S \subseteq \RR^d$ is called \emph{full dimensional} if $\dim(S) = d$. This definition is equivalent to saying that $\interior(\conv(S))$ is nonempty. Similarly, for any matrix $\mA \in \RR^{d \times n}$, we say $\mA$ is full dimensional if the set of its columns is full dimensional.
\end{definition}
This formal definition of full dimensional allows us to identify sufficient conditions for a dataset $\mZ = (\mX, \mY)$ to be SC. 

\subsection{Preliminary results on separability decomposition}\label{sec:clf-results-prelim}
In this section, we introduce the technical results that help us analyze the separation between shuffling SGD and GD. A unifying theme is to understand the effect of monochromatic batches and rank on the separability decomposition --- and thus, divergence.  In particular, we show that having monochromatic batches and being full-rank prevents divergence in the underparameterized regime. In a later section (\cref{sec:characterize-direction}), we also show that these properties also significantly influence the optimal directions under the logistic loss. 

The following lemma formalizes how monochromatic batches affect the separability decomposition.
\begin{lemma}\label{lemma:batch-separability}
Given a permutation $\pi$, suppose there are two batches $\mZ_\pi^1$ and $\mZ_\pi^2$ such that $\mZ_\pi^1$ consists entirely of positive examples and $\mZ_\pi^2$ consists entirely of negative examples. Then, if we consider the resulting $\Zpibn = (\Xpibn, \Ypi)$, the submatrices $\Xpibn^+$ and $\Xpibn^-$ of $\Xpibn$ satisfy
$$\relinterior(\conv(\Xpibn^+)) \cap \relinterior(\conv(\Xpibn^-)) \neq \emptyset.$$ Consequently, $\Zpi$ is not LS.
\end{lemma}
\begin{proof}
Batch normalization ensures that the batch normalized features $\Xpibn^1$ and $\Xpibn^2$ are mean-zero. But this implies that $\vzero$ is in the convex hulls of each batch, which implies that $\conv(\Xpibn^1)$ intersects $\conv(\Xpibn^2)$. In fact, $\vzero$ is in the intersection of their relative interiors as well. To see this, we prove that the mean $\vmu$ of a batch $\mB = \qty{\vx_1, \ldots, \vx_B}$ lies in the relative interior of $\conv(\mB)$. 

This can be shown via an inductive argument on the batch size. If $B=2$, then $\vmu$ is the midpoint between $\vx_1$ and $\vx_2$, which is in the relative interior of $\conv(\mB)$. Now assume it's true for all possible batches of size $B$. When we add $\vx_{B+1}$, we get a new batch $\mB'$, with mean $\vmu' = \frac{B}{B+1}\vmu + \frac{1}{B+1}\vx_{B+1}$. Hence if $\vx_{B+1} \in \conv(\mB)$, clearly $\vmu' \in \relinterior(\conv(\mB))$ by convexity. If $\vx_{B+1} \not\in \conv(\mB)$, then $\vx_{B+1}$ is one of the vertices of $\conv(\mB')$. Since $\vmu \in \relinterior(\conv(\mB))$ and convexity, the segment between $\vmu$ and $\vx_{B+1}$ must stay in the relative interior of $\conv(\mB')$ except at the endpoints. Since $\vmu'$ is in the relative interior of this segment, the conclusion follows.

Hence $\relinterior(\conv(\Xpibn^1))$ intersects $\relinterior(\conv(\Xpibn^2))$. Since we have $\Xpibn^+ \supseteq \Xpibn^1$ and $\Xpibn^- \supseteq \Xpibn^2$, the relative interiors of the larger convex hulls intersect as well. 

Finally, suppose $\Zpi$ was LS. By definition there must exist a strict separating hyperplane for the two hulls. But the hulls intersect, so this is a contradiction. We conclude that $\Zpi$ is not LS. 
\end{proof}

As \cref{lemma:batch-separability} establishes, monochromatic batches lead to $\Zpi$ being PLS or SC. The following lemma synthesizes nicely with the above result; it demonstrates that if $\mX$ is full dimensional, and the relative interiors of the convex hulls of $\mX^+$ and $\mX^-$ intersect, then $\mZ$ is SC. 
\begin{lemma}\label{lemma:sc-sufficient}
Let $\mZ =(\mX, \mY)$ be such that 
\begin{enumerate}[label=(\arabic*)]
    \item $\mX$ is full dimensional.
    \item $\relinterior(\conv(\mX^+))$ intersects $\relinterior(\conv(\mX^-))$.
\end{enumerate}
Then $\mZ$ is SC.
\end{lemma}
\begin{proof}
Consider any halfspace in $\RR^d$. First, since the relative interiors of the hull of positive points and negative points intersect, then so too do the hulls themselves. So there is no hyperplane that separates $\mX^+$ from $\mX^-$, i.e. $\mZ$ is either PLS or SC.

Suppose that $\mZ$ is PLS, i.e., there exists a hyperplane $\vv$ such that $y_i\vv^\top \vx_i \ge 0$ for every $(\vx_i, y_i) \in \mZ$.  
Since $\mX$ is full dimensional, the hyperplane orthogonal to $\vv$ cannot pass through through every point of $\mX$ --- hyperplanes are affine subspaces of dimension at most $d-1$. So $y_i \vv^\top \vx_i > 0$ for some $i$ and also $y_ j\vv^\top \vx_j = 0$ for some $i \neq j$; otherwise, $\mZ$ is LS, which is a contradiction. Hence $\conv(\mX^+)$ and $\conv(\mX^-)$ touch only at the hyperplane defined by $\vv$, which contradicts the assumption that the relative interiors intersect. Hence $\mZ$ is SC. 
\end{proof}

\Cref{lemma:batch-separability,lemma:sc-sufficient} taken together show that to identify sufficient conditions for $\Zpi$ to be SC, one should look for conditions under which $\Xpibn$ is full dimensional and monochromatic batches are present. We already answered the former question in the main text with \cref{prop:full-rank}, which we restate for reference. Its proof is deferred to \cref{sec:full-rank}.
\begin{proposition}\label{prop:full-rank-alt}
Assume that the original features $\mX \in \RR^{d \times n}$ satisfies \Cref{assumption:density} and $B > 2$. Then if we batch normalize and remove one datapoint from each normalized batch, to form a $d \times (B-1)\frac{n}{B}$ matrix in the SS case and a $d \times (B-1)\binom{n}{B}$ matrix in the RR case, the dataset is full-rank almost surely, regardless of which datapoints we remove. In particular, we have $\rank(\Xpibn) = \min\qty{d, (B-1)\frac{n}{B}}$ and $\rank(\Xrrbn) = \min\qty{d, (B-1)\binom{n}{B}}$ almost surely.
\end{proposition}

Let us now consider the other question about the presence of monochromatic batches. Intuitively, under \ref{assumption:ss-balanced}, there should be many monochromatic batches w.h.p.\ as long as $B$ is small. The following lemma formalizes this intuition; its proof is contained in \Cref{app:monochromatic}.
\begin{lemma}\label{cor:ss-sc-whp}
Assume \ref{assumption:ss-balanced}. If $B = o(\log n)$ then there are $\Omega(n)$ monochromatic batches w.h.p.. If $B = \Omega(\log n)$, then there are no monochromatic batches w.h.p.. 
\end{lemma}

The upshot of \Cref{lemma:batch-separability,lemma:sc-sufficient,prop:full-rank-alt,cor:ss-sc-whp} is that small batch sizes naturally prevent divergence. However, there is a natural tradeoff here: small batch sizes also entail significant variance in the batch statistics, so they can lead to large (but non-diverging) values of the GD risk anyway when we train the network with SS.

\begin{remark}[Multiclass classification]
In the multiclass case with $K > 2$ different classes, one can directly generalize the above analysis to look at all $\binom{K}{2}$ pairwise combinations of classes. \Cref{lemma:batch-separability} generalizes by requiring the existence of a monochromatic batch \emph{for each class}. Hence as $K$ increases the batch size must shrink to ensure that $\Zpi$ is SC w.h.p., opening up a wider range of batch sizes for SS divergence.  
\end{remark}

Finally, we formally define a robust notion of the separability decomposition that will prove helpful for quantifying the effects of increasing the batch size. To do so, we rely on the notion of the margin of a linearly separably dataset and the so-called penetration depth of overlapping convex hulls.

\begin{definition}[Margin and penetration depth]\label{definition:margin-penetration-depth}
Let $\mZ = (\mX, \mY)$ be a dataset, and let $\mX^+, \mX^-$ denote the positive and negative features, respectively. If $\mZ$ is linearly separable, then the \emph{margin} of $\mZ$ is defined to be the $\ell_2$ margin corresponding to the maximum margin classifier for $\mZ$.

If $\mZ$ is not linearly separable, then by definition $\conv(\mX^+)$ intersects $\conv(\mX^-)$. The \emph{penetration depth} \citep{agarwal2000penetration} of $\mZ$ is defined as the smallest Euclidean distance of translation of $\conv(\mX^+)$ such that the resulting convex body still intersects $\conv(\mX^-)$. In words, this quantifies the smallest perturbation we need to make $\mZ$ linearly separable.
\end{definition}

\begin{definition}[$\gamma$-robust separability decomposition]\label{assumption:margin}
Let $\mZ = (\mX, \mY)$ be a dataset and $\Sgdls \sqcup \Sgdsc$ be the separability decomposition of $\Zgd \triangleq (\Xgdbn,\mY) \triangleq (\bn(\mX),\mY)$.

For $\gamma > 0$, we say that $\mZ$ is $\gamma$-robust if the following conditions hold. 
\begin{enumerate}[label=\normalfont(\arabic*)]
    \item One of the following: 
    \begin{enumerate}[label=\normalfont(\alph*)]
        \item $\Zgd$ is LS and $\Sgdls$ has margin at least $\gamma$; or,
        \item $\Zgd$ is SC and $\Sgdsc$ has penetration depth at least $\gamma$. 
    \end{enumerate}
    \item Let $\sigma_k$, $a_k$, and $b_k$ denote the standard deviation, min, and max of the $k$th feature of $\mX$. Then $\min_{k \in [d]} \frac{\sigma_k}{b_k-a_k} = \Omega(1)$. 
    \item $\norm{\Xgdbn}_{2,\infty} = O(\sqrt{d})$, i.e., the maximum Euclidean norm of datapoints in $\bn(\mX)$ is $O(\sqrt{d})$. 
\end{enumerate}
\end{definition}

\Cref{assumption:margin} formalizes the informal statement in \Cref{thm:ss-divergence-informal} that ``$\Zgd$'s separability decomposition can change with small perturbations.'' If a dataset is robust, its separability decomposition cannot change easily by small perturbations (e.g., due to batch normalization) on datapoints. Notice here that PLS datasets can never be robust; to see why, note that $\mZ$ is PLS whenever $\Zgd$ is PLS. It follows that if $\mZ$ is PLS then it is not $\gamma$-robust for any $\gamma > 0$. 

This definition of robustness is also natural from the perspective of concentration, since it provides an immediate link between concentration of batch statistics and the separability decomposition of $\Zpi$. In order to estimate $\Xgdbn$ well, we need to estimate $\vsigma$ to within a multiplicative factor, which explains (2). Moreover, the degree to which the SS datapoints concentrate around the GD datapoints also depend multiplicatively on the size of GD datapoints, which explains (3).

The following proposition provides a sufficient condition for the datapoints concentrating within a distance $\gamma$ of the GD datapoints (cf.\ the definition of $\gamma$-robustness); its proof is deferred to \cref{sec:concentration}.
\begin{proposition}\label{prop:concentration-robustness}
Suppose $\mZ$ is $\gamma$-robust and $B = \Omega(d\log(nd)/\gamma^2)$. Then with probability at least $1-1/\poly(n, d)$ over the choice of $\pi$, we have $\norm{\Xpibn - \pi \circ \Xgdbn}_{2,\infty} = O(\gamma)$. 
\end{proposition}
With this setup in hand, we are ready to present our main theorems characterizing the separability decompositions of SS+BN and RR+BN. 
\subsection{Separability decomposition for SS}\label{subsec:ss-structure}
\begin{theorem}[Separability decomposition for SS, formal]\label{prop:ss-structure}
Throughout this theorem, assume that $B > 2$ and \ref{assumption:ss-balanced} and \Cref{assumption:density} hold. 

Suppose that $d \le (B-1)\frac{n}{B}$. Then the following hold.
\begin{enumerate}[label=\normalfont(\arabic*)]
    \item \label{item:underparam-small} $\Zpi$ is SC w.h.p.\ for $B = o(\log n)$. If we relax \Cref{assumption:density}, then $\Zpi$ can be PLS as well. 
    \item \label{item:underparam-large} Suppose further that $\mZ$ is $\gamma$-robust. Then $\Zpi$ has the same separability decomposition as $\Zgd$ w.h.p.\ for $B = \Omega(d\log (nd)/\gamma^2)$. 
\end{enumerate}
Now suppose $d > (B-1)\frac{n}{B}$. Then the following hold. 
\begin{enumerate}[label=\normalfont(\arabic*)]\setcounter{enumi}{2}
    \item \label{item:overparam-small} $\Zpi$ is PLS w.h.p.\ for $B = o(\log n)$. 
    \item \label{item:overparam-large} $\Zpi$ is LS w.h.p.\ for $B = \Omega(\log n)$.
\end{enumerate}
\end{theorem}
\begin{proof}
We consider cases based on whether the batch size is large or small.

\textbf{Small batch size.~~}
By \cref{cor:ss-sc-whp}, when $B = o(\log n)$, w.h.p.\ we get a batch comprised solely of positive examples and a batch comprised solely of negative examples. By \cref{lemma:batch-separability}, this implies that $\Zpi$ is PLS or SC, and the relative interior of the positive features intersects the relative interior of the negative features. By \cref{lemma:sc-sufficient}, if $\Xpibn$ is full dimensional, then $\Zpi$ is SC. Hence it suffices to analyze the rank of $\Xpibn$. 

The maximal rank of $\Xpibn$ is $\min\qty{d, (B-1)\frac{n}{B}}$ because of the mean zero constraint enforced in every batch. Because we assumed that \Cref{assumption:density} holds, \Cref{prop:full-rank} implies that $\Xpibn$ achieves this upper bound almost surely. Putting it all together, we conclude that when $d \le (B-1)\frac{n}{B}$, $\Xpibn$ is full-dimensional. It follows that $\Zpi$ is SC w.h.p.\ over the choice of $\pi$, which proves \ref{item:underparam-small}.

On the other hand, when $d > (B-1)\frac{n}{B}$, there always exists a hyperplane that passes through all of the monochromatic batches of $\Zpibn$ and perfectly classifies non-monochromatic batches (see \cref{sec:characterize-direction}). It follows that $\Zpibn$ is PLS, which proves \ref{item:overparam-small}.

\textbf{Large batch size.~~}
Now consider when $B = \Omega(\log n)$. In this regime, \cref{cor:ss-sc-whp} implies there are no monochromatic batches with high probability. Moreover, \cref{prop:full-rank} implies that the features are full rank. It thus follows that when $d > (B-1)\frac{n}{B}$, $\Zpi$ is LS, which proves \ref{item:overparam-large}. Let us now consider the case $d \le (B-1)\frac{n}{B}$. 

When $\mZ$ is $\gamma$-robust, we can directly apply \cref{prop:concentration-robustness} to prove \ref{item:underparam-large}. Indeed, for $B = \Omega(d\log(nd)/\gamma^2)$, each SS datapoint is within distance $O(\gamma)$ of the corresponding GD datapoint with probability at least $1 - 1/\poly(n,d)$. By increasing the batch size by at most a constant factor, we can guarantee that each SS datapoint is in fact within distance $\gamma/3$ of the corresponding GD datapoint. 

If $\Zgd$ is LS, $\gamma$-robustness implies that the max-margin hyperplane for $\Zgd$ has margin at least $\gamma$. Since each SS datapoint moves at most $\gamma/3$ from the corresponding GD datapoint, this hyperplane still has margin at least $2\gamma/3$, implying that $\Zpi$ is LS. If $\Zgd$ is SC, then $\gamma$-robustness implies that we need to translate the convex hulls of positive and negative points by at least $\gamma$ to separate them. But we can only translate them by a total of $\gamma/3 + \gamma/3 = 2\gamma/3 < \gamma$, so the hulls stay strictly overlapping. This implies $\Zpi$ is also SC.  This concludes the proof of \ref{item:underparam-large}.
\end{proof}
\begin{remark}
If $\Zgd$ is not $\gamma$-robust for any $\gamma > 0$ (for example if $\Zgd$ is PLS), then it is not hard to construct examples where the separability decomposition of $\Zpi$ is LS or PLS, and SS diverges. For a good example of this scenario, see \Cref{fig:xgd-final,fig:xss-final}. Even if $\Zgd$ is $\gamma$-robust, if $B = o(d\log(nd)/\gamma^2)$, then concentration can fail to hold in the worst case, and we can find analogous constructions where SS diverges.
\end{remark}

\subsection{Separability decomposition for RR}\label{subsec:rr-structure}

\begin{theorem}[Separability decomposition for RR, formal]\label{prop:rr-structure}
Suppose that $B > 2$ and \ref{assumption:rr-balanced} and \Cref{assumption:density} hold.

If $d \le (B-1)\binom{n}{B}$, then $\Zrr$ is SC and $\Xrrbn$ is full-rank almost surely, regardless of the separability decomposition of $\Zgd$.

Otherwise, if $d > (B-1)\binom{n}{B}$, then $\Zrr$ is deterministically PLS, regardless of the separability decomposition of $\Zgd$.
\end{theorem}

\begin{proof}
When there are at least $B$ positive and $B$ negative examples, there exists a batch of all positive and a batch of all negative examples. Hence by \cref{lemma:batch-separability}, $\Zrr$ is PLS or SC. By \cref{prop:full-rank-alt}, under \Cref{assumption:density}, $\Xrrbn$ attains the maximal rank of $\min\qty{d, (B-1)\binom{n}{B}}$ almost surely. So by \cref{lemma:sc-sufficient}, if $d \le (B-1)\binom{n}{B}$, $\Zrr$ will be SC and $\Xrrbn$ is full-rank almost surely. On the other hand, if $d > (B-1)\binom{n}{B}$, then $\Zrr$ is PLS deterministically because there always exists a hyperplane that passes through all of the monochromatic batches of $\Zrrbn$ and perfectly classifies non-monochromatic batches.
\end{proof}

\subsection{Characterizing the optimal direction of classifiers}\label{sec:characterize-direction}
Thus far, we have primarily considered the separability decomposition of $\Zpi$ and $\Zrr$. In fact, we can say more about the direction of optimal classifiers under the logistic loss. First, we prove \cref{lemma:infimize}, which constrains optimal directions via the separability decomposition. Next, we leverage overparameterization and the rank properties shown in \cref{prop:full-rank-alt} to characterize the optimal direction under the logistic loss for data drawn from a density. Using these insights, we can prove our main result of this section, \cref{prop:overparam-takeaway}.

We first restate and prove \cref{lemma:infimize}.
\begin{lemma}\label{lemma:infimize-alt}
Let $\mZ = \Sls \sqcup \Ssc$ be the separability decomposition of $\mZ$. If $\vv$ is an optimal direction for $\mathcal{L}$, then $\vv^\top \vx = 0$ for all $\vx \in \Span(\Xsc)$ and $y_i\vv^\top \vx_i > 0$ for every $(\vx_i, y_i) \in \Sls$. 
\end{lemma}
\begin{proof}
By definition of $\Ssc$, there exists some $(\vx_i, y_i) \in \Ssc$ such that $y_i\ev{\vv, \vx_i} \le 0$. If $\ev{\vv, \vx_i} \neq 0$, then $y_i\ev{t\vv, \vx_i} \to -\infty$ as $t \to \infty$, which contradicts the assumption that $\vv$ is an optimal direction. Similarly, if $y_i\ev{\vv, \vx_i} \le 0$ for some $(\vx_i, y_i) \in \Sls$, then clearly $\vu+t\vv$ cannot infimize $\mathcal{L}$, as there exists a hyperplane which strictly separates $\Sls$ and is orthogonal to $\Span(\Ssc)$. 
\end{proof}

Next, we restate and prove  \cref{prop:divergence-possibility}.
\begin{proposition}\label{prop:divergence-possibility-alt}
Suppose \ref{assumption:full-rank-ss} holds. Assume that the iterates $\vpi(t)$ infimize $\ellpi$, and their projections onto $\Span(\Xpisc)^{\perp}$ converge in direction to some optimal direction $\vpi^*$ for $\ellpi$. Then the GD risk $\ellgd$ diverges if and only if $\Zpi$ is PLS or LS and there exists some $(\vx_i, y_i) \in \Zgd$ such that $y_i\vpi^{*\top} \vx_i < 0$. The analogous statement holds true for $\Zrr$ under \ref{assumption:full-rank-rr}. Furthermore, the if part holds true for SS and RR without \Cref{assumption:full-rank-reg}.
\end{proposition}
\begin{proof}
First suppose \ref{assumption:full-rank-ss} holds. Then if $\Zpi$ is SC, \Cref{lemma:infimize-alt} implies that any optimal direction $\vpi^*$ must be orthogonal to all of $\RR^d$, so $\vpi^* = \vzero$. Hence the iterates $\vpi(t)$ converge to a finite optimum, which implies the GD risk cannot diverge. Note that if \ref{assumption:full-rank-ss} doesn't hold, then the only difference is that being orthogonal to $\Xpisc$ does not imply that $\vpi^* = \vzero$. 

Now suppose $\Zpi$ is PLS or LS. Regardless of whether \ref{assumption:full-rank-ss} holds, if $\vpi(t)$ infimizes $\ellpi$, we necessarily have $\norm{\vpi(t)}_2 \to +\infty$. Hence any mistake on $(\vx_i, y_i) \in \Zgd$ implies divergence. And clearly, if there is no mistake on any $(\vx_i, y_i)$ then the GD risk does not diverge. 

The same proof also goes through for $\Zrr$, so this proves the theorem. 
\end{proof}

At first glance, one might expect to be able to perfectly classify all the datapoints in the overparameterized regime. However, under the logistic risk, the optimal direction instead puts monochromatic batches \emph{on the decision boundary}; the following proposition formalizes this notion. 
\begin{proposition}\label{prop:overparam-takeaway}
Suppose \Cref{assumption:density}, $B>2$, and $d > (B-1)\frac{n}{B}$. Almost surely, for any $\pi \in \sS_n$, there exists $\vv \in \RR^d$ which satisfies \textnormal{(1)} for any non-monochromatic batch $\Zpi^j$ we have $\sgn(\vv^\top \Xpibn^j) = \Ypi^j$ and \textnormal{(2)} for any monochromatic batch $\Zpi^k$ we have $\vv^{\top} \Xpibn^k = \vzero^\top$. Furthermore, any optimal direction $\vpi^*$ for the logistic risk $\ellpi$ necessarily satisfies both \textnormal{(1)} and \textnormal{(2)}. The same conclusion holds for RR as well, with the requirement $d > (B-1)\binom{n}{B}$.
\end{proposition}

More precisely, our analysis is motivated by the fact that if $d \ge n$ and $\mX \in \RR^{d \times n}$ is full-rank, then given any $\vc \in \RR^n$ we can find some halfspace $\vv \in \RR^d$ such that $\vv^\top\mX = \vc^\top$. In our binary classification setting, linear separability is implied by $\sgn(\vc^\top) = \mY$. However, we cannot directly apply this fact because BN actually prevents $\Xpibn$ from being full-rank due to the mean zero constraint. However, it turns out that the following slightly weaker statement is true. We can always find a halfspace $\vv \in \RR^d$ that perfectly separates all the non-monochromatic batches. The following lemma proves this half of the proposition.
\begin{lemma}\label{lemma:overparam-separation}
Suppose \Cref{assumption:density}, $B>2$, and $d > (B-1)\frac{n}{B}$. Then almost surely there exists $\vv \in \RR^d$ such that for every non-monochromatic batch $\Zpi^j = (\Xpibn^j, \Ypi^j)$, we have $\sgn(\vv^\top \Xpibn^j) = \Ypi^j$. The same conclusion holds for RR as well, with the requirement $d > (B-1)\binom{n}{B}$.
\end{lemma}
\begin{proof}
We denote $\Xpibn^j = \bn(\Xpi^j) =  \mqty[\vx_1 & \cdots & \vx_B]$ and $\Ypi^j = \mqty[y_1 & \cdots & y_B]$. Since \Cref{assumption:density} holds, it also follows that \cref{prop:full-rank} holds. Hence, within each batch, any $B-1$ of the datapoints are linearly independent almost surely. This implies that we can find $\vv \in \RR^d$ such that for any batch $\Xpibn^j$ and $\vc \in \RR^{B-1}$, we have $\vv^\top \vx_i = c_i$ for all $i \in [B-1]$. In particular, we can pick $c_i$ such that $\sgn(c_i) = y_i$. 
Next, we show that $\vv$ can be picked such that $\sgn(\vv^\top \vx_B) = y_B$. The mean zero constraint enforces that $\vv^\top \vx_B = -\sum_{i=1}^{B-1} c_i$. But if the labels are not monochromatic, we can just increase the absolute value of one of the $c_i$'s so that $\sgn(-\sum_{i=1}^{B-1} c_i) = y_B$, as desired. 
\end{proof}

Hence, in the overparameterized regime minimizing the logistic risk will lead to a classifier which separates all the non-monochromatic batches. What happens to the monochromatic batches? It turns out that minimizing the logistic risk will lead to a classifier whose decision boundary \emph{contains} all of the monochromatic batches. 
\begin{lemma}\label{fact:interpolate}
Assuming $d > (B-1)\frac{n}{B}$, any optimal direction $\vv_* \in \RR^d$ for the logistic risk simultaneously puts all of the monochromatic batches on the decision boundary. More precisely, for any monochromatic batch $\Zpi^j = (\Xpibn^j, \Ypi^j)$ we have $\vv_*^\top \Xpibn^j = \vzero^\top$. Similarly, if $d > (B-1)\binom{n}{B}$, the same conclusion holds for the RR dataset. 
\end{lemma}

\begin{proof}
Again, we denote the monochromatic batch by $\Zpi^j$ by $\Xpibn^j = \mqty[\vx_1 & \cdots & \vx_B]$ and $\Ypi^j = \mqty[y_1 & \cdots & y_B]$. 
The logistic loss on a single input $\vx_i$ for classifier $\vv \in \RR^d$ is $\ell(\vv^\top \vx_i, y_i) = -\log \rho(y_i\vv^\top \vx_i)$, where $\rho(t) = 1/(1+\exp(-t))$ is the sigmoid function. WLOG suppose that $y_i = 1$ for all $i \in [B]$. Hence the minibatch risk is
\[
- \sum_{i=1}^B \log \rho(\vv^\top \vx_i) = - \sum_{i=1}^{B-1} \log \rho(\vv^\top \vx_i) - \log \rho \left (-\vv^\top \sum_{i=1}^{B-1} \vx_i \right ).\] For each $i$ we can look at the first order optimality condition for $s_i = \vv^\top \vx_i$. This yields $\rho(s_i) = \rho(-\sum_{i=1}^{B-1} s_i)$.

Note that this is satisfied when $s_i = 0$ for all $i \in [B]$, and by strict convexity this is the unique minimizer. And because of overparameterization, we can find a $\vv_*$ that will satisfy $s_i = 0$ for each monochromatic batch, i.e., the batch entirely lies on the decision boundary of the classifier defined by $\vv_*$. 

The proof carries over immediately to the RR setting, except in that setting, being overparameterized means $d > (B-1)\binom{n}{B}$. Hence the lemma is proved.
\end{proof}

The geometric interpretation of \cref{fact:interpolate} is that in the overparameterized regime, if the dataset contains any monochromatic batches, then any optimal direction $\vv_*$ must be orthogonal to the subspace spanned by the monochromatic batches. Note, however, that the definition of overparameterized here depends on whether we look at $\Zpi$ or $\Zrr$. In the former case, overparameterized means $d > (B-1)\frac{n}{B}$, whereas in the latter case, overparameterized means $d > (B-1)\binom{n}{B}$. This insight motivates the construction of the toy datasets in \cref{example:clf-divergence}. We conclude our characterization of the optimal direction in the overparameterized regime by proving \cref{prop:overparam-takeaway}.

\begin{proof}[Proof of \cref{prop:overparam-takeaway}]
\Cref{prop:full-rank} implies that for $d > (B-1)\frac{n}{B}$ and assuming \cref{assumption:density}, almost surely we have $\rank(\Xpibn) = (B-1)\frac{n}{B}$. We can lower bound the infimum of the SS logistic risk by the sum of the infima of the mini-batch SS logistic risks. Combining \cref{lemma:overparam-separation,fact:interpolate}, the claim follows.
The same argument holds for $\Zrr$ assuming $d > (B-1)\binom{n}{B}$. 
\end{proof}

\section{Proofs of technical lemmas}\label{sec:cls-proofs}

In this section we spell out the formal details of our guarantees for how permutations interact with BN. In \Cref{app:monochromatic}, we show that given a batch size $B = o(\log n)$ and a constant number of classes $K$, then w.h.p.\ over the choice of $\pi$ there exists (many) monochromatic batches. In other words, for small batch sizes there are many batches consisting solely of positive or negative examples (in case of $K = 2$). Conversely, we show that above this threshold such monochromatic batches do not appear w.h.p.. In \Cref{sec:concentration}, we show that there is a commensurate threshold above which the batch statistics themselves concentrate in the without-replacement setting. To do so we will appeal to the recent results of \citet{Bardenet2015ConcentrationIF} on the concentration of without-replacement estimators. Finally, in \Cref{sec:full-rank}, we prove that assuming the original features were drawn from a density, the batch normalized features have maximal rank almost surely. In other words, batch normalization preserves genericity of the original features modulo the mean zero constraint inside each batch.

\subsection{Presence of monochromatic batches}\label{app:monochromatic}
In this section we formally prove \cref{cor:ss-sc-whp} via standard concentration arguments. One of the potential pitfalls of any mini-batch based algorithm is that its batches may not be representative of the entire dataset. More precisely, let's suppose we have a dataset $\mZ$ with labels coming from $K$ classes. Suppose furthermore that the dataset is balanced --- each class contains $n = Bm$ examples (here $K$ is a constant). For any batch size $B$ which is not sufficiently large, i.e. $B = o(\log n)$, then w.h.p.\ over the permutation $\pi$ we will have $\Theta(n)$ batches which are \emph{monochromatic} --- batches which only consist of examples in the same class. This is closely related to the classic coupon collector problem, but we restate the guarantees here for the sake of completeness.

To prove this claim, we appeal to standard martingale concentration inequalities. Consider the batches $\Zpi^i$ for $i \in [Km]$ and define the indicator variables $T_i = \1[\Zpi^i \text{ is monochromatic}]$, and set $T \triangleq \sum_i T_i$, the total number of monochromatic batches.

Next, note that the sequence 
\[\EE[T], \EE[T|T_1], \EE[T|T_1,T_2], \ldots, \EE[T|T_1,T_2,\ldots,T_{Km}]\]
is a Doob martingale. Indeed the martingale property follows from the tower property:
\[\EE[\EE[T|T_{1:k}]|\EE[T|T_{1:(k-1)}]] = \EE[T|T_{1:(k-1)}].\]
Note that $\EE[T|T_1,\ldots, T_{Km}] = T$ and $\EE[T_1] = K\binom{n}{B}/\binom{Kn}{B}$. Hence by linearity $\EE[T] = \frac{K^2n}{B}\frac{\binom{n}{B}}{\binom{Kn}{B}}$. 

Let us now show that the martingale increments $\EE[T|T_{1:k}] - \EE[T|T_{1:(k-1)}]$ are bounded a.s.. In the worst case, the $k$th batch can only decrease this conditional expectation by at most $K$, since for any fixed class we can only remove at most one monochromatic batch from it. Hence the total number of potential monochromatic batches left can decrease by at most $K$. This worst case guarantee still holds under conditional expectation, so $\abs{\EE[T|T_{1:k}] - \EE[T|T_{1:(k-1)}]} \le K$ a.s.. 

Azuma-Hoeffding then tells us that for any $\epsilon > 0$ we have 
\[
\PP[\abs{T - \EE[T]} \ge \epsilon] \le 2\exp\qty(-\frac{B\epsilon^2}{2nK^3}). 
\]

This gives us the following fact which we use in both the regression and classification setting.
\begin{fact}\label{fact:whp-monochromatic}
For any $\delta \in (0,1)$, and a constant number of classes $K$ with $n$ datapoints each, we have with probability at least $1-\delta$ that the total number of monochromatic batches $T$ satisfies
\[
\abs{T - \frac{K^2n}{B}\frac{\binom{n}{B}}{\binom{Kn}{B}}} \le \sqrt{\frac{2nK^3\log(2/\delta)}{B}}.
\]
\end{fact}

For $K = O(1)$, we note that the above inequality guarantees that we can get within $O(\sqrt{n\log n})$ of the true expectation with at most $1/\poly(n)$ failure probability. For the toy regression dataset in \Cref{sec:opt-analysis-reg}, we have $K = 2$ and $B=2$, so we can use \cref{fact:whp-monochromatic} to deduce that there will be $\Theta(n)$ monochromatic batches with high probability. We can also use \cref{fact:whp-monochromatic} to prove \cref{cor:ss-sc-whp}.
\begin{proof}[Proof of \cref{cor:ss-sc-whp}]
In the classification setting, we have $K = 2$ classes (positive and negative). Using the folklore inequalities 
\[
\qty(\frac{n}{k})^k \le \binom{n}{k} \le \qty(\frac{en}{k})^k,
\]
we can deduce the lower bound on the expectation of $T$:
\[
\EE[T] \ge \frac{4n}{B(2e)^B}. 
\]
The lower bound is $\Omega(n^{1-\epsilon})$ for any $\epsilon > 0$ whenever $B = o(\log n)$, so indeed when when $B = o(\log n)$ we have a positive number of monochromatic batches w.h.p. 

Let us now upper bound the probability of obtaining any monochromatic batches. We have
\[
\PP[T_1 = 1] = 2\frac{\binom{n}{B}}{\binom{2n}{B}} \le 2\frac{\prod_{k=0}^{B-1} (n-k)}{\prod_{k=0}^{B-1}(2n-k)} \le  2\frac{\prod_{k=0}^{B-1} (n-k)}{\prod_{k=0}^{B-1}2(n-k)} \le 2^{-B+1}.
\]
Hence by union bound the probability that $T > 0$ is upper bounded by $\frac{4n}{B \cdot 2^B}$. This is $1/\poly(n)$ for some $B = \Omega(\log n)$, so indeed when $B = \Omega(\log n)$ we have no monochromatic batches with probability at least $1 - 1/\poly(n)$. This concludes the proof. 
\end{proof}

\subsection{Concentration of batch statistics for without-replacement sampling}\label{sec:concentration}

In the following, we are generating $B$ samples of $X_i \in \RR$ without replacement from a population of size $n$, contained in $[a, b]$ a.s.. We let $\mu$ be the population mean and $\sigma^2$ be the population variance. 
\begin{lemma}[Corollary 2.5 in \citep{Bardenet2015ConcentrationIF}]\label{lemma:mean-upper-tail}
Let $\hat{\mu}_B$ denote the sample mean for a sample of size $B$ drawn without replacement from the overall population. 
For all $B \le n$ and $\delta \in (0, 1)$ we have with probability at least $1-\delta$ that
\[
\abs{\hat{\mu}_B - \mu} \le (b-a)\sqrt{\frac{\log(2/\delta)}{B}}.
\]
\end{lemma}

Similarly, they prove the following result about concentration of the empirical variance
\begin{lemma}[Lemma 4.1 in \citep{Bardenet2015ConcentrationIF}]\label{lemma:variance-upper-tail}
Let $\hat{\sigma}_B^2 \triangleq \frac{1}{B} \sum_{i=1}^B (X_i - \hat{\mu}_B)^2$ be the (biased) empirical variance estimator and $\hat{\sigma}_B \triangleq \sqrt{\hat{\sigma}_B^2}$. Then for all $\delta \in (0,1)$ we have with probability at least $1-\delta$ that 
\[
\hat{\sigma}_B \ge \sigma - 3(b-a)\sqrt{\frac{\log(3/\delta)}{2B}}.
\]
\end{lemma}

We now prove the other side of this concentration inequality with a quick application of \citep[Theorem 1]{maurer2006concentration}, following the same notation as in \citet{Bardenet2015ConcentrationIF}. 
\begin{lemma}\label{lemma:variance-lower-tail}
For all $\delta \in (0,1)$ we have with probability at least $1-\delta$ that
\[
\hat{\sigma}_B \le \sigma + (b-a)\sqrt{\frac{\log(1/\delta)}{2B}}.
\]
\end{lemma}
\begin{proof}
We take the self bounded random variable $Z = \frac{B}{(b-a)^2}\Tilde{V}_B$, where
\[
\Tilde{V}_B \triangleq \frac{1}{B} \sum_{i=1}^B (X_i - \mu)^2
\]
is computed with the samples $X_i$ sampled \emph{with replacement}. 
On the other hand, 
\[
V_B \triangleq \frac{1}{B} \sum_{i=1}^B (X_i' - \mu)^2
\]
is computed with the samples $X_i'$ sampled \emph{without replacement}. We can relate the concentration of $V_B$ to that of $\Tilde{V}_B$; the latter quantity is possible to analyze with the entropy method. 
Indeed, a routine modification of the proof of \citet[Lemma~3.3]{Bardenet2015ConcentrationIF} (which uses essentially the same definition of $Z$) implies that 
\[
\PP\qty[V_B - \sigma^2 \ge \frac{(b-a)^2}{B}\epsilon] \le \exp(-\frac{(b-a)^2\epsilon^2}{2B\sigma^2}).
\]
Solving for $\epsilon$ in terms of $\delta$ yields $\epsilon = \sqrt{\frac{2B\sigma^2}{(b-a)^2}\log(1/\delta)}$, so we obtain
\[
\PP\qty[V_B - \sigma^2 \ge (b-a)\sigma \sqrt{\frac{2\log(1/\delta)}{B}}] \le \delta.
\]
Since $V_B = (\hat{\mu}_B - \mu)^2 + \hat{\sigma}_B^2 \ge \hat{\sigma}_B^2$, we can complete the square to obtain 
\[
\PP\qty[\hat{\sigma}_B^2 \ge \qty(\sigma + (b-a)\sqrt{\frac{\log(1/\delta)}{2B}})^2] \le \delta.
\]
So with probability at least $1-\delta$, we have 
\[
\hat{\sigma}_B^2 \le \qty(\sigma + (b-a)\sqrt{\frac{\log(1/\delta)}{2B}})^2,
\]
and taking square roots implies the desired result. 
\end{proof}

Now, let us return to the question of concentration for batch norm with a randomly selected permutation $\pi$. The following proposition shows that assuming the batch size is large enough, the features of corresponding SS and GD datapoints are close to each other.
\begin{proposition}\label{prop:concentration-features}
If 
\[
B = \Omega\qty(\frac{\log(nd)}{\min_{k \in [d]}(\frac{\sigma_k}{b_k-a_k})^2\epsilon^2}),
\]
then with probability at least $1-1/\poly(n, d)$ we have
\[
\norm{\Xpibn - \pi \circ \Xgdbn}_{2, \infty} \le \frac{\epsilon}{1-\epsilon} \norm{\Xgdbn}_{2, \infty} + \frac{\epsilon\sqrt{d}}{1-\epsilon}.
\]
Here, we remind the reader that for a matrix $\mA \in \RR^{d \times n}$, we define $\norm{\mA}_{2,\infty} = \max_{i \in [n]} \norm{\mA_{:, i}}_2$, i.e. the maximum Euclidean norm of the columns. 
\end{proposition}
\begin{proof}
WLOG consider the unnormalized first batch $\Xpi^1 = \qty{\vx_1, \ldots, \vx_B}$. We initially handle concentration along its first coordinate $\qty{x_1, \ldots, x_B}$ for its first datapoint $x_1$. Write $\hat{\vmu},\hat{\vsigma} \in \RR^d$ to denote the mini-batch mean and standard deviation, respectively, and $\hat{\mu}_k, \hat{\sigma}_k \in \RR$ for the $k$th coordinate of these vectors. Similarly let $\vmu, \vsigma \in \RR^d$ denote the full-batch mean and standard deviation, and let $\mu_k, \sigma_k \in \RR$ to denote the $k$th coordinate of these vectors, respectively. Finally, let $\va, \vb \in \RR^d$ denote the coordinate-wise min and max of $\mX$, with $a_k, b_k \in \RR$ denoting the min and max for the $k$th coordinate of $\mX$.

Hence, the first feature of the first datapoint in 
the normalized first batch $\Xpibn^1$ is $\frac{x_1-\hat{\mu}_1}{\hat{\sigma}_1}$, whereas the corresponding quantity in $\Xgdbn$ is $\frac{x_1-\mu_1}{\sigma_1}$. 

Pick 
\[
B = \Omega\qty(\frac{(\log(1/\delta)}{(\frac{\sigma_1}{b_1-a_1})^2\epsilon^2}).
\]
Then from \cref{lemma:mean-upper-tail,lemma:variance-upper-tail,lemma:variance-lower-tail}, we see that $\abs{\hat{\mu}_1 - \mu_1} \le \epsilon\sigma_1$ and $\abs{\hat{\sigma}_1 - \sigma_1} \le \epsilon \sigma_1 $ with probability at least $1 - \delta$ for the first datapoint $x_1$. 

Now, we have 
\begin{align*}
\abs{\frac{x_1-\hat{\mu}_1}{\hat{\sigma}_1} - \frac{x_1-\mu_1}{\sigma_1}} &= \abs{\frac{x_1-\mu_1}{\hat{\sigma}_1} - \frac{x_1-\mu_1}{\sigma_1} + \frac{\mu_1 - \hat{\mu}_1}{\hat{\sigma}_1}} \\
&\le \abs{\frac{\sigma_1 - \hat{\sigma}_1}{\hat{\sigma}_1}}\abs{\frac{x_1-\mu_1}{\sigma_1}} + \frac{\abs{\mu_1 - \hat{\mu}_1}}{\hat{\sigma}_1} \\
&\le \frac{\epsilon\sigma_1}{(1-\epsilon)\sigma_1} \abs{\frac{x_1-\mu_1}{\sigma_1}} + \frac{\epsilon\sigma_1}{(1 - \epsilon)\sigma_1} \\
&\le \frac{\epsilon}{1-\epsilon} \abs{\frac{x_1-\mu_1}{\sigma_1}} + \frac{\epsilon}{1 - \epsilon}.
\end{align*}

 To aggregate this bound across features, we need to pick $B$ such that $\abs{\hat{\mu}_k - \mu_k} \le \epsilon\sigma_k$ and $\abs{\hat{\sigma}_k - \sigma_k} \le \epsilon \sigma_k$ for every feature $k \in [d]$. Indeed, this is achieved via the union bound by picking $\delta = 1/\poly(d)$ and 
\[
B = \Omega\qty(\frac{\log(1/\delta)}{\min_{k \in [d]}(\frac{\sigma_k}{b_k-a_k})^2\epsilon^2}).
\]

For this batch size, we see that 
\[
\norm{\frac{\vx_1 - \hat{\vmu}}{\hat{\vsigma}} - \frac{\vx_1 - \vmu}{\vsigma}}_2 \le \frac{\epsilon}{1-\epsilon} \norm{\frac{\vx_1 - \vmu}{\vsigma}}_2 + \frac{\epsilon\sqrt{d}}{1-\epsilon}.
\]
Note that a similar inequality holds for all of $\vx_1, \dots, \vx_B$. Now, recalling the definition of $\norm{\cdot}_{2, \infty}$, by applying union bound on all the batches, we have
\[
\norm{\Xpibn - \pi \circ \Xgdbn}_{2, \infty} \le \frac{\epsilon}{1-\epsilon} \norm{\Xgdbn}_{2, \infty} + \frac{\epsilon\sqrt{d}}{1-\epsilon},
\]
which occurs with probability at least $1-1/\poly(n, d)$ when we appropriately choose $B = \Omega\qty(\frac{\log(nd)}{\min_{k \in [d]}(\frac{\sigma_k}{b_k-a_k})^2\epsilon^2})$.
\end{proof}

We now use the above proposition to prove \Cref{prop:concentration-robustness}.
\begin{proof}[Proof of \Cref{prop:concentration-robustness}]
Conditions (2) and (3) in the definition of $\gamma$-robustness ensures that $\min_{k \in [d]}(\frac{\sigma_k}{b_k-a_k})^2 = \Omega(1)$ and $\norm{\Xgdbn}_{2, \infty} = O(\sqrt{d})$. Hence taking 
$B = \Omega(\frac{\log(nd)}{\epsilon^2})$ as in \cref{prop:concentration-features}, we conclude that with probability at least $1-1/\poly(n,d)$ we have 
\[
\norm{\Xpibn - \pi \circ \Xgdbn}_{2, \infty} \le O\qty(\frac{\epsilon\sqrt{d}}{1-\epsilon}).
\]
Hence, by taking $\epsilon = O(\frac{\gamma}{\sqrt{d}})$, we see that 
\[
\norm{\Xpibn - \pi \circ \Xgdbn}_{2, \infty} \le O(\gamma).
\]
Plugging this choice of $\epsilon$ back into our definition of $B$, we conclude that when $B = \Omega(\frac{d\log(nd)}{\gamma^2})$, $\Xpibn$ concentrates around $\pi \circ \Xgdbn$ within distance $\gamma$. 
\end{proof}

\subsection{Rank of batch normalized features}\label{sec:full-rank}
In this section we prove \cref{prop:full-rank}, which states that for batch sizes greater than $2$,  \Cref{assumption:density} implies that the batch normalized dataset will be full-rank (and hence full-dimensional) almost surely. 

One shift in perspective that is fruitful for proving linear independence is to view batch normalization as an operation that returns functions of the input dataset. Since BN operates independently on each of the $d$ features, we first handle the case where the input is a batch of scalars. Let $\binom{[n]}{B}$ denote the set of all $\binom{n}{B}$ batches of size $B$ that can be created from $n$ datapoints contained in $\mX$. Fix an arbitrary labelling of these $\binom{n}{B}$ batches, and let $\mB^j$ refer to the $j$th such batch. WLOG suppose that $\mB^1 = \qty{x_1, \ldots, x_B} \in \RR^B$. 

Formally, let $\mathcal{F}^B \triangleq \qty{f: \RR^B \setminus \qty{\vx \in \RR^B \mid x_1 = x_2 = \cdots = x_B} \to \RR}$ denote the space of real valued functions on batches of size $B$ where BN is defined. On batch $\mB^j = \qty{x_{j_1}, \ldots, x_{j_B}}$, $\bn$ is an operation that maps this batch to the set of $B$ functions $\qty{g_{i}^j(\mB^j)}_{i=1}^{B}$, where 
\[
g_i^j(\mB^j) \triangleq \frac{x_{j_i} - \mu^j}{\sigma^j} \in \mathcal{F}^B
\]where $\mu^j$ and $\sigma^j$ are the empirical mean and standard deviation, respectively of $\mB^j$. If $j$ is clear from context, we may drop the superscript $j$ without chance of confusion. We also sometimes abuse notation and write $g_i^j$ as a function of $\mX$, since $\mX$ contains all datapoints in $\mB^j$. From this perspective, $\bnpi$ maps a dataset of $n$ datapoints to $n$ functions.

We first show that, within a batch, the functions have rank $B-1$ over $\RR$.
\begin{lemma}\label{lemma:not-identically-zero}
Viewed as functions, any subset of $B-1$ functions of the batch normalized outputs $\qty{g_i^1} = \qty{\frac{x_1 - \mu}{\sigma}, \ldots, \frac{x_B-\mu}{\sigma}}$ have rank $B-1$ over $\RR$.
\end{lemma}
\begin{proof}
First, we note that the functions $x_i$ are linearly independent over $\RR$ for $i \in [n]$. WLOG take the subset of $B-1$ functions to be $\qty{g_i^1}_{i=1}^{B-1}$. Suppose that there is a dependence relationship
\[
\sum_{i=1}^{B-1} c_i\frac{x_i- \mu}{\sigma} = 0. 
\]
Rearranging, we see that 
\[
\sum_{i=1}^{B-1} c_ix_i = \frac{1}{B}\qty(\sum_{t=1}^{B} x_t) \qty(\sum_{i=1}^{B-1}c_i)
\]
and because of linear independence of the $x_i$'s as functions over $\RR$ and only the right-hand side contains $x_B$, the only way this can happen is if $\sum_{i=1}^{B-1} c_i = 0$. But then we have a dependence relationship between the $x_i$'s for $i \le B-1$. Linear independence of the $x_i$'s thus implies that $c_i = 0$ for each $i$. 
\end{proof}
With the above lemma in hand, we can show that as functions, any collection of batch normalized outputs are essentially full rank. The caveat is we need to throw out one function in each batch, as each batch is trivially dependent because of the zero mean constraint. 
\begin{proposition}\label{prop:full-rank-functions}
Let $B > 2$. Consider the $B\binom{n}{B}$ functions that are the batch normalized outputs of all $\binom{n}{B}$ batches $\mB^j \in \binom{[n]}{B}$. If we take any subset of $(B-1)\binom{n}{B}$ of these functions obtained by removing one function from each of the $\binom{n}{B}$ batches, then the rank of these functions over $\RR$ is $(B-1)\binom{n}{B}$. In particular, the rank of the functions corresponding to any $\pi$ is $(B-1)\frac{n}{B}$.
\end{proposition}
\begin{proof}
Consider the batch $\mB^1 = \qty{x_1, \ldots, x_B}$. By \cref{lemma:not-identically-zero}, we know that the functions $\qty{g_i^1}_{i=1}^{B-1} = \qty{\frac{x_i-\mu}{\sigma}}_{i=1}^{B-1}$ are linearly independent. Consider a dependence relation amongst any subset of the $(B-1)\binom{n}{B}$ described in the theorem statement. WLOG we can suppose that this was formed by throwing out $g_B^j$ from each batch $\mB^j$ and consider a dependence relation between the $(B-1)\binom{n}{B}$ remaining functions. The dependence relation reads
\begin{equation}\label{eq:dependence}
    \sum_{i=1}^{B-1} c_{i,1} \frac{x_i-\mu}{\sigma} = \sum_{i=1}^{B-1}\sum_{j>1} c_{i,j}g_i^j.
\end{equation}

We show that some setting of the input $x_i$ for $i \in [n]$ yields a contradiction unless $c_{i,1} = 0$ for $i \in [B]$. The main insight is that BN has jump discontinuities at the points where it is undefined, i.e., where the entire batch is equal to the same thing. 

More formally, for $i > B$ we set the other datapoints $x_i$ to be arbitrary pairwise distinct positive real numbers. Suppose for the sake of contradiction that $c_{i,1} \neq 0$ for some $i \in [B]$; WLOG we can assume that $i=1$. Since the functions on the LHS are linearly independent by \cref{lemma:not-identically-zero}, the LHS is not identically zero. We show that the LHS  of \cref{eq:dependence} exhibits discontinuous behavior in the punctured neighborhood around $(x_1, \ldots, x_B) = \vzero$, whereas the RHS of \cref{eq:dependence} is continuous on the same neighborhood; this yields a contradiction. 

Indeed, set $(x_1, x_2, x_3, \ldots, x_B) = (\epsilon, \epsilon^2, -\epsilon-\epsilon^2, 0, \ldots, 0)$, where $\epsilon \neq 0$. We have 
\[
\frac{x_1-\mu}{\sigma} = \frac{\sqrt{B}\epsilon}{\sqrt{\epsilon^2 + \epsilon^4 + (\epsilon+\epsilon^2)^2}}.
\]
If $\epsilon \to 0^+$, then the the first normalized coordinate approaches $+\sqrt{\frac{B}{2}}$. However, if $\epsilon \to 0^-$, then the first normalized coordinate approaches $-\sqrt{\frac{B}{2}}$. This is a contradiction, since the RHS of \cref{eq:dependence} is continuous as a function of $\epsilon$. We conclude that $c_{i,1} = 0$ for all $i \in [B-1]$. 

The same argument holds if we replace the LHS with any batch $\mB^j$. Hence, $c_{i,j} = 0$ for all $i \in [B-1]$ and $j \in [\binom{n}{B}]$. We conclude the rank of these $(B-1)\binom{n}{B}$ functions is $(B-1)\binom{n}{B}$, as desired. Notice also that this argument also shows that the functions corresponding to the batches in $\pi$ also have rank $(B-1)\frac{n}{B}$.
\end{proof}
Note that the assumption that $B > 2$ is critical for the construction in the proof. If $B=2$, then actually $\frac{x_i-\mu}{\sigma} \in \qty{\pm 1}$, and the proof breaks down. The batch normalized dataset will be a Boolean matrix; hence, its rank cannot be analyzed by using density arguments.

The following lemma establishes that the zero set of any nontrivial linear combination of the $g_i^j$'s is a measure zero subset of $\RR^n$.
\begin{lemma}\label{claim:zero-set-null}
Suppose that $c_{i,j}$ are not identically zero. Then for $B > 2$, the zero set of $\sum_{i=1}^{B-1} \sum_{j=1}^{\binom{n}{B}} c_{i,j} g_i^j$ is a measure zero subset of $\RR^n$.
\end{lemma}
\begin{proof}
Since $B > 2$, \cref{prop:full-rank-functions} implies that $f(\mX) \triangleq \sum_{i=1}^{B-1} \sum_{j=1}^{\binom{n}{B}} c_{i,j} g_i^j$ is not identically zero. Furthermore, $f(\mX)$ is real analytic on finitely many connected components of $\RR^{n}$, since each of the functions $g_i^j$ are real analytic on finitely many connected components of $\RR^n$. The claim follows by applying Proposition 0 in \citet{mityagin2015zero}. 
\end{proof}

Having established these results, we can finally prove \cref{prop:full-rank}. 
\begin{proof}[Proof of \cref{prop:full-rank}]
Denote 
\[
\mX = \mqty[\mX_1 \\ \vdots \\ \mX_d] \in \RR^{d \times n},
\]
and the $(B-1)\binom{n}{B}$ functions 
\[
\vg_i^j(\mX) = \mqty[g_{i, 1}^j(\mX_1) \\ \vdots\\ g_{i, d}^j(\mX_d)] \in \RR^d 
\]

We can assemble these vector valued functions into a matrix $\mG(\mX)$ of size $d \times (B-1)\binom{n}{B}$ with the $k$th row consisting of the scalar valued functions $g_{i, k}^j(\mX)$. Now consider the determinant of any $\min\qty{d, (B-1)\frac{n}{B}} \times \min\qty{d, (B-1)\frac{n}{B}}$ submatrix of these scalar functions, which is itself a function of $\mX$. 

To prove the claim, it suffices to show that this determinant --- which is a function of $\mX \in \RR^{d \times n}$ --- is analytic almost everywhere and not identically zero; then \cref{claim:zero-set-null} implies that it vanishes on a measure zero set of $\RR^{d \times n}$. 

We prove the above claim by induction on $d$. For $d=1$, note that for any $i, j$, the scalar function $g_{i, 1}^{j}(\mX)$ is not identically zero and is analytic on finitely connected components of $\RR^{d \times n}$. Now suppose the claim is true for $d$, we prove the claim for $d+1$. When we use cofactor expansion along the first row of $\mG$, which has functions $g_{i,1}^j$ which crucially only depend on the $\mX_1$, we obtain 
\[
\sum_{i=1}^{B-1}\sum_{j=1}^{\binom{n}{B}} (-1)^{i+j}g_{i,1}^j(\mX_1) \det(\mM_{i,j}(\mX)),
\]
where $\mM_{i,j}$ denotes the minor corresponding to the the $i$th function of batch $j$. Note that these minors are not functions of $\mX_1$, so they can be treated as constants with respect to $\mX_1$. By induction these constants are nonzero almost surely. Then \cref{prop:full-rank-functions} implies that the determinant, which is a linear combination of the functions $g_{i,1}^j$ is itself is not zero identically. Also, the determinant is manifestly piecewise analytic on finitely connected components of $\RR^{(d+1) \times n}$, being a polynomial of such functions. Hence the induction is completed.

We can now finish off the proof of the proposition. When \Cref{assumption:density} holds, the probability that the data falls in measure zero set is a probability zero event. In other words, almost surely any such $d \times (B-1)\binom{n}{B}$ matrix constructed by batch normalizing and throwing out one normalized point in each batch is full rank. The same argument holds for the $(B-1)\frac{n}{B}$ functions that correspond to the batch normalized outputs for a permutation $\pi$. Hence the proposition is proved. 
\end{proof}

\section{Additional experiments on real data}\label{app:experiments}

In this section we provide detailed explanations of our experimental setup and present our additional experiments for regression and classification. All experiments were implemented in PyTorch 1.12.

\subsection{Experiment details}
We first define the architectures used in our real data experiments outlined in \Cref{sec:clf-experiments}.

For the linear+BN networks, the 1-layer network is
\[
\mX \mapsto \mW_1 \mGamma_1 \bn(\mX),
\]
On the other hand, the 2-layer network is 
\[
\mX \mapsto \mW_2 \mGamma_2 \bn(\mW_1 \mX)
\] 
and the 3-layer network is 
\[
\mX \mapsto \mW_3 \mGamma_3 \bn(\mW_2 \mGamma_2 \bn(\mW_1 \mX)).
\] 
Hence the difference between the 1-layer network and deeper network is that the deeper networks have tunable parameters inside of $\bn$. 

For the MLP experiments, the 3-layer network is
\[
\mX \mapsto \mW_3\relu(\mGamma_3 \bn(\mW_2 \relu(\mGamma_2 \bn(\mW_1 \mX)))).
\]

For the ResNet18 experiments, we used the ResNet18 architecture available through PyTorch, using \texttt{ResNet18\_Weights.DEFAULT} pretrained weights.

The linear and BN layers were all initialized using the default PyTorch initialization. For the linear+BN networks, the linear layers were instantiated with a width of 512 and \texttt{bias=False}. For the 3 layer MLP, the linear layers were instantiated with a width of 512 and \texttt{bias=True}. The BN layers were instantiated with \texttt{track\_running\_stats=False}. As alluded to in \cref{sec:related-work}, to evaluate the training GD risk $\ellgd$ in the eval loop, we passed in the entire dataset as a single batch, thus avoiding EMA altogether. Except for the ResNet18 experiments, the images in the dataset were flattened into vectors. 

Except for in the respective batch size and momentum ablation study (\Cref{fig:bsize_ablation,fig:momentum_ablation}), we used batch size $B=128$ and no momentum. Note that for all of the datasets,  $\log_2 n \approx 16$, which suggests that we are in the asymptotic regime where divergence can happen as stated in \Cref{thm:ss-divergence-informal}. 

We now explain the difference in divergence behavior between the 1-layer and deeper linear+BN networks for SS. As suggested by \Cref{thm:ss-divergence-informal}, divergence can happen if the separability decomposition of $\Zgd$ is not robust to perturbation. In the 1-layer case, the data remains far from being linearly separable. Meanwhile, in the deeper case, the network is incentivized to train the parameters inside $\bn$ such that the final features (e.g., $\bnpi(\mW_1 \mX)$) are closer to being LS. But the nonlinearity of BN is not enough to make $\bnpi(\mW_1 \mX)$ robustly LS. This also explains why introducing nonlinear activations prevents the divergence phenomenon. 

We also note that in reality RR is run for $T$ epochs. Thus, if $T < \frac{\binom{n}{B}}{\frac{n}{B}}$, the optimization routine only sees a proper subset of $\Zrr$. However, there are other forces that help ensure that the subset of $\Zrr$ actually seen during optimization is SC and satisfies \ref{assumption:full-rank-rr}. For example, it is likely that the algorithm sees non-monochromatic batches that also cause the hulls to overlap. One way this can happen is if $\vzero$ is in the relative interiors of the convex hulls of the monochromatic portion of each batch. Moreover, with extremely high probability we never see a repeat batch, so by \cref{prop:full-rank-alt} the rank of the subset of $\Zrr$ is w.h.p.\ equal to $T \cdot \min\qty{d, (B-1)\frac{n}{B}}$. For us, $d = 10 \cdot 512$, $B = 128$, and $n \ge 50000$, so the rank of the subset of $\Zrr$ we see easily outstrips the dimensionality of the final linear layers. This ensures that \ref{assumption:full-rank-rr} holds.

In \Cref{fig:bsize_ablation}, we see that divergence on 3 layer linear+BN networks generally occurs for large batch sizes, which corroborates \cref{thm:ss-divergence-informal}. These batch sizes were picked because they are common choices for batch sizes in practice. For the largest batch size ($B=128$), there does not appear to be divergence within 1000 epochs, which we address below. 

In \Cref{fig:momentum_ablation}, we see that the presence of momentum preserves divergence for SS, and in some cases accelerates it. Note that for each stepsize we used the same permutation. For the $\eta = 10^{-4}$ experiment, although the 0 momentum run did not start to diverge within 1000 epochs, the $0.9$ and $0.99$ momentum runs started to diverge. This further lends evidence to the claim that the apparent reason for no divergence for $\eta = 10^{-4}$ without momentum is that the small learning rate leads to slower convergence to an optimal direction for $\ellpi$.

\begin{figure}[!ht]
    \centering
\includegraphics[width=0.6\textwidth]{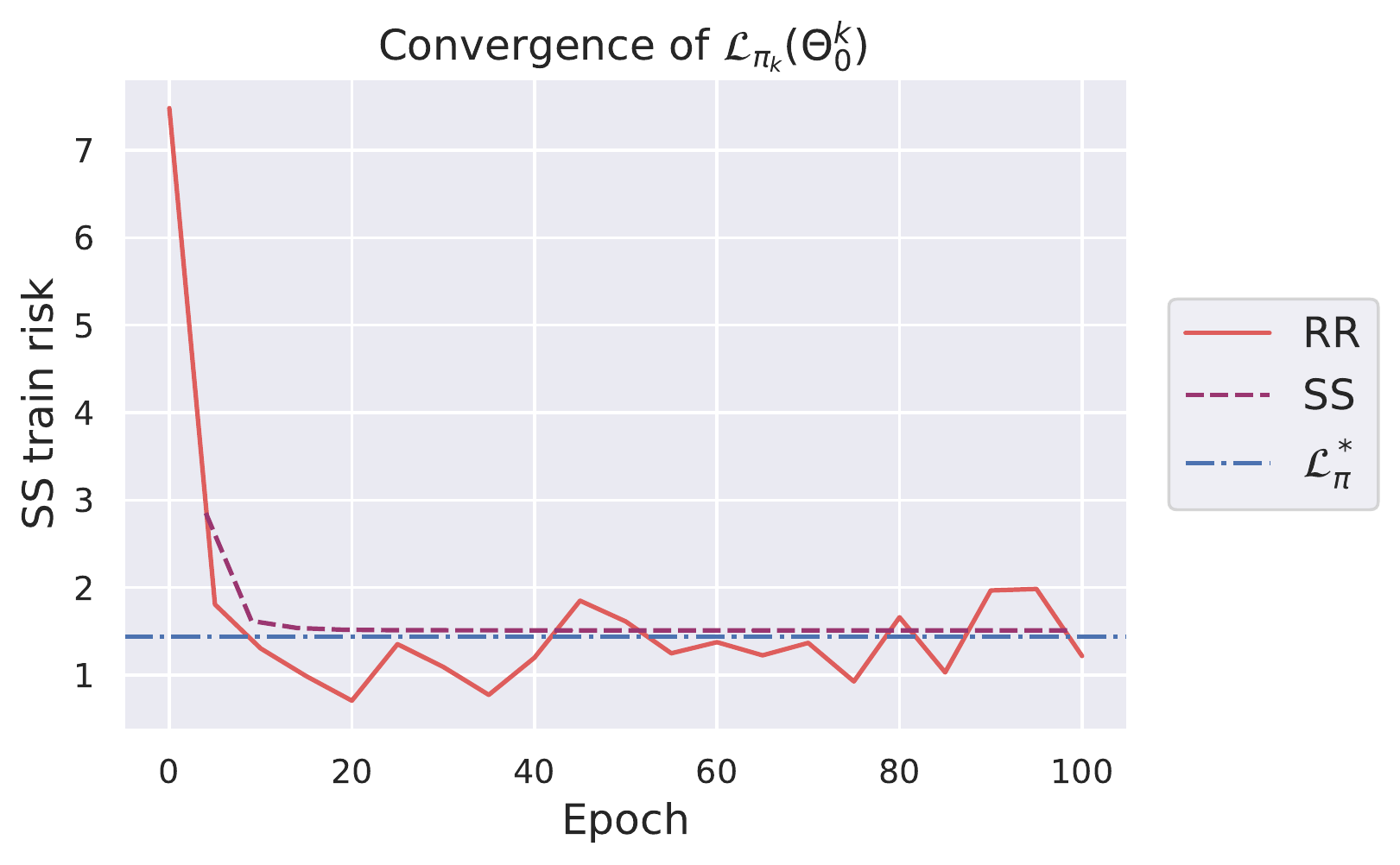}
     \caption{Loss evolution for $\ell_{\pi_k}(\mM(k))$ for experiment described in \cref{sec:reg-experiments}.}
     \label{fig:regression-convergence}
\end{figure}

\begin{figure}[!ht]
    \centering
    \includegraphics[width=0.6\textwidth]{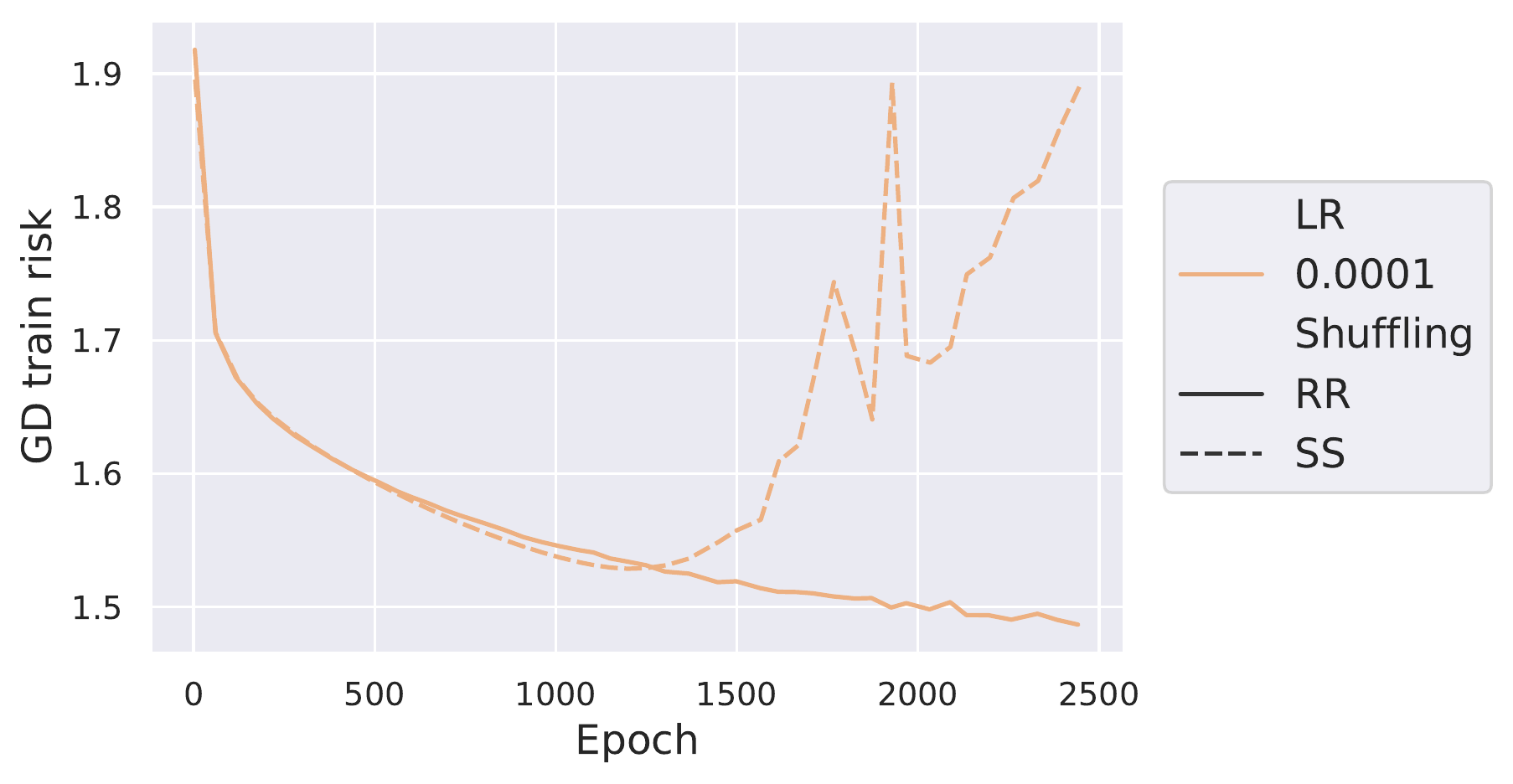}
    \caption{Eventual separation between SS and RR for 3 layer linear+BN network for $\eta=10^{-4}$ after around epoch 1200 on CIFAR10.}
    \label{fig:eventual_separation}
\end{figure}

\begin{figure}[!ht]
    \centering
    \includegraphics[width=0.6\textwidth]{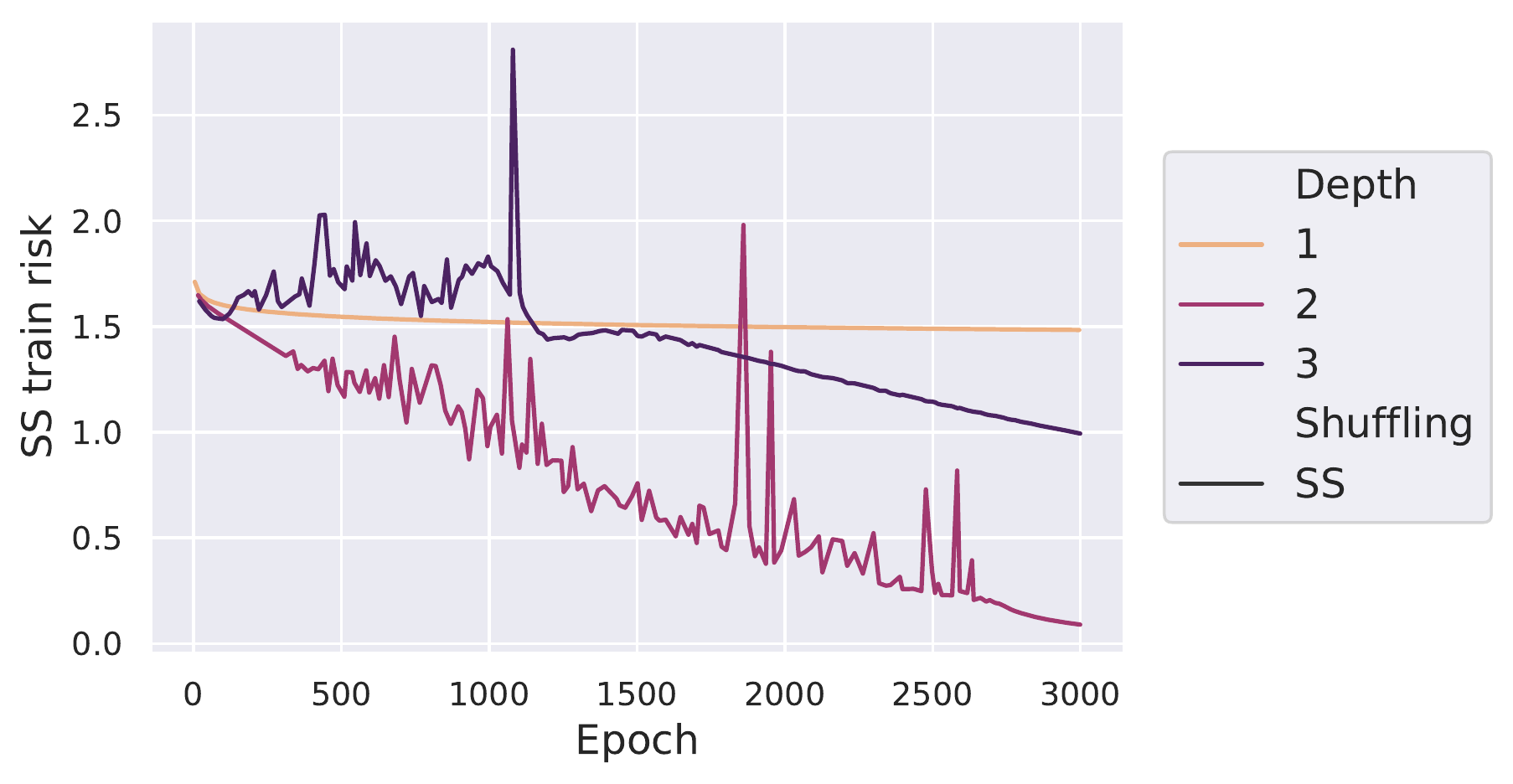}
    \caption{Evidence that the 2 layer and 3 layer linear+BN actually had diverging GD risks when trained with SS. Indeed, the SS risk for both 2 and 3 layer networks continues to decrease, whereas the 1 layer network seems to plateau.}
    \label{fig:linear_separability}
\end{figure}

\begin{figure}[!ht]
    \centering
    \includegraphics[width=0.6\textwidth]{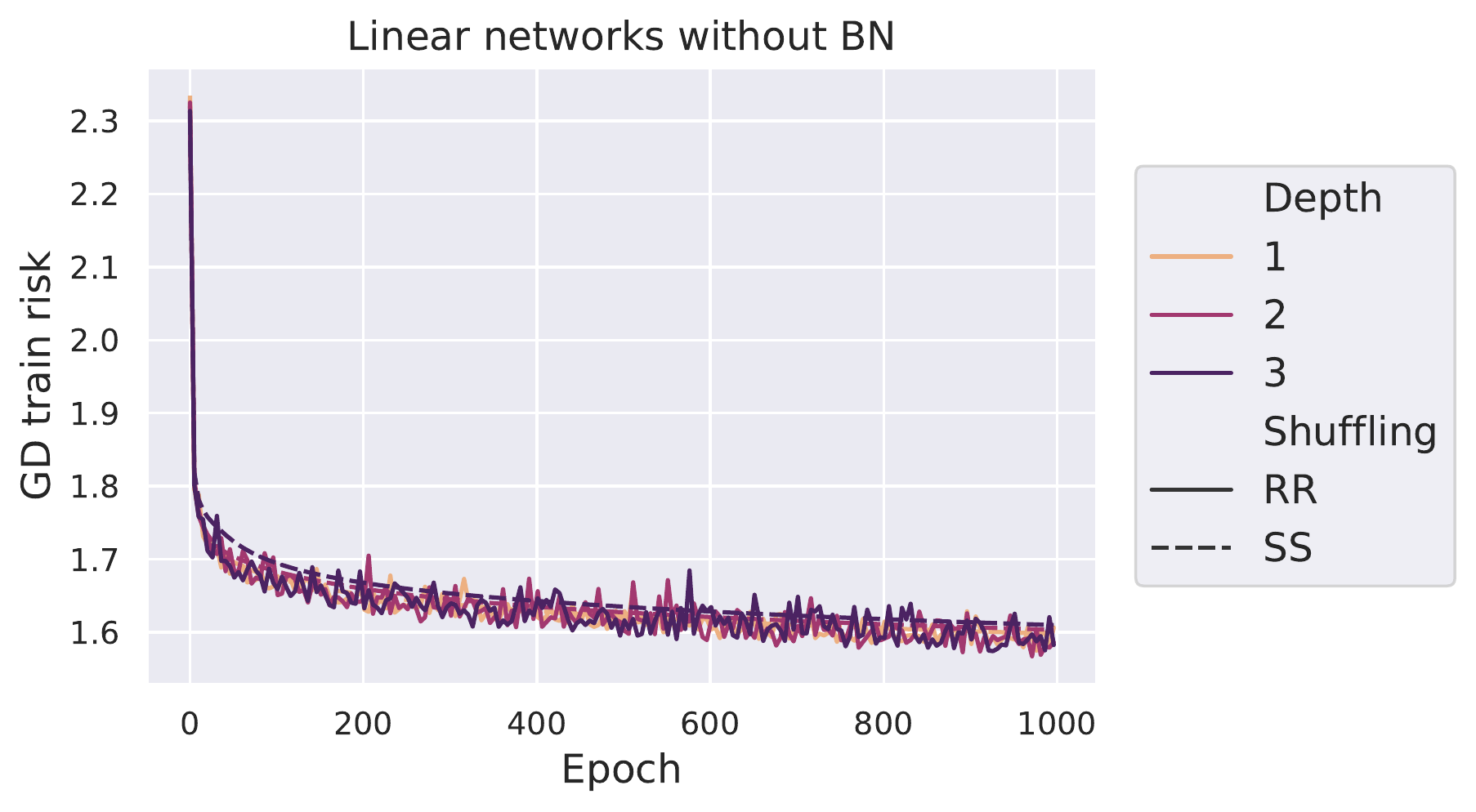}
    \caption{No divergence occurs for 1, 2, and 3 layer linear networks without BN trained with SS. Here $\eta = 10^{-2}$.}
    \label{fig:no_bn}
\end{figure}

\begin{figure}[!ht]
     \centering
     \includegraphics[width=\textwidth]{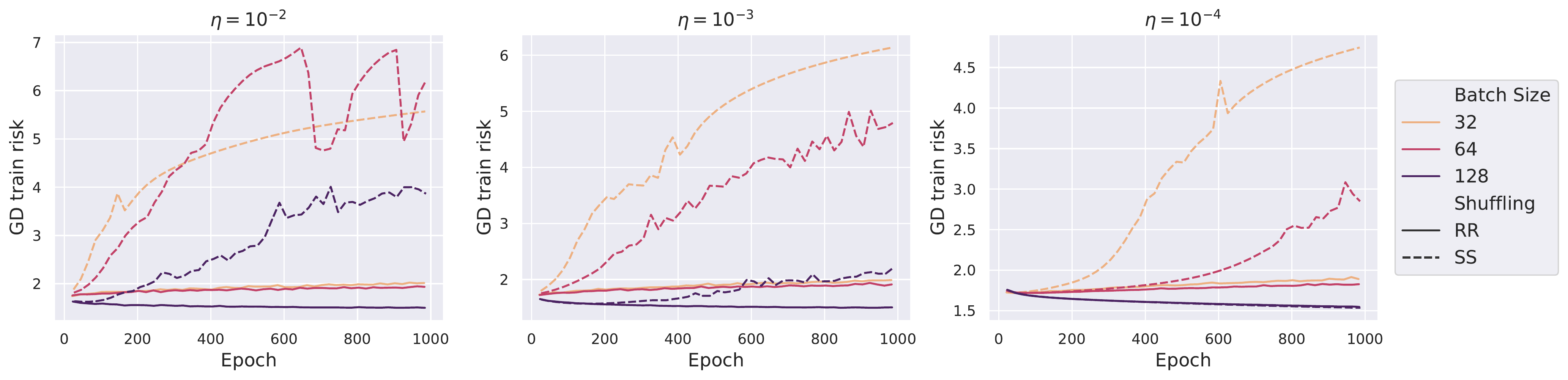}
      \caption{Batch size ablation. These experiments were done on 3 layer linear+BN networks; each subfigure shows the results of training with different stepsizes $\eta \in \qty{10^{-2}, 10^{-3}, 10^{-4}}$. All experiments were performed on CIFAR10. All batch sizes are in the regime where divergence can happen according to \cref{thm:ss-divergence-informal}.}
     \label{fig:bsize_ablation}
     \hfill
\end{figure}

\begin{figure}[!ht]
     \centering
     \includegraphics[width=\textwidth]{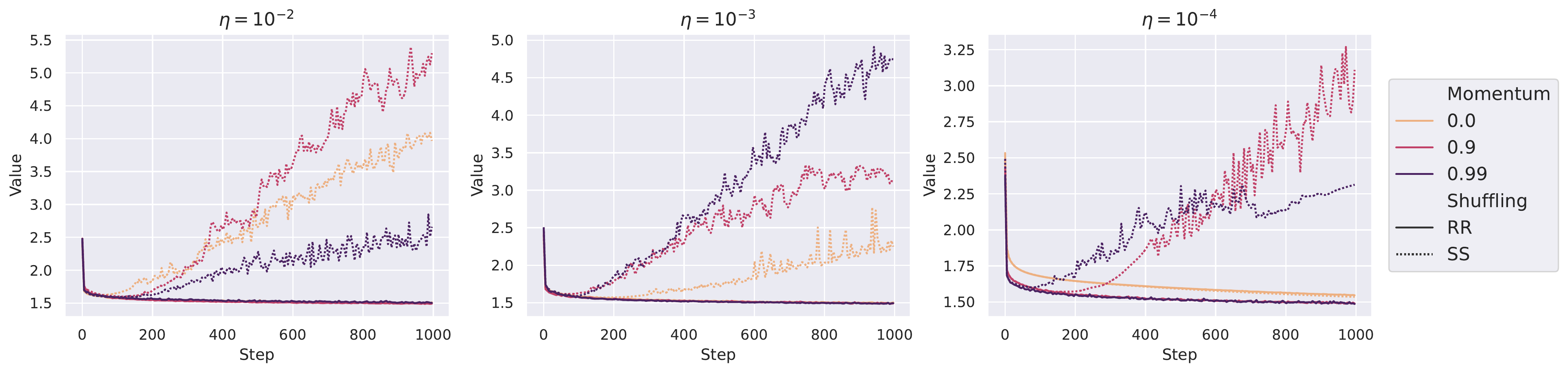}   
     \caption{Momentum ablation. These experiments were done on 3 layer linear+BN networks; each subfigure shows the results of training with different stepsizes $\eta \in \qty{10^{-2}, 10^{-3}, 10^{-4}}$.  All experiments were performed on CIFAR10. For each subfigure, the experiment was run with the same random permutation.}
     \label{fig:momentum_ablation}
     \hfill
\end{figure}
\FloatBarrier

\section{Calculations for toy datasets}\label{sec:toy-datasets}
In this section we do the detailed calculations and provide additional figures to understand the toy datasets we introduced in \Cref{sec:opt-analysis-reg,example:clf-divergence}. Since both constructions use $B=2$, we remind the reader that the batch normalization of any two distinct real numbers is $(-1, 1)$. It follows that if we batch normalize with $B=2$, we obtain $\Xpibn \in \qty{-1, 1}^{d \times n}$. Recall the distorted risk $\ellpi$ reflects the single-shuffle batch normalized dataset $\Zpi$. This $\ellpi$ is a \emph{random} quantity, and it is over this source of randomness (the construction of the batches) that we show that there is a gap between SS, RR, and GD.

\subsection{Regression toy dataset}

\begin{proposition}\label{prop:toy-regression-dataset}
    There exists a regression dataset $\mZ = (\mX, \mY) \in [-1, 1]^{1 \times 16n} \times [-1, 1]^{1 \times 16n}$ such that the following statements hold with batch size $B=2$: 
    \begin{enumerate}[label=\normalfont{(\arabic*)},ref=(\theenumi)]
        \item \label{item:reg-prop-1} $\Mgd^* = \Mrr^* = 0$.
        \item \label{item:reg-prop-2} $\Mpi^* \neq 0$ with probability at least $1 - O(\frac{1}{\sqrt{n}})$. 
        \item \label{item:reg-prop-3} $\abs{\Mpi^*} = \Omega(\frac{1}{\sqrt{n}})$ with constant probability.  
    \end{enumerate}
\end{proposition}
We first describe the construction and then prove that the dataset satisfies the properties outlined above.

\paragraph{Formal construction of regression dataset:}
Construct the dataset as follows. Take $\mA \in \RR^{1 \times 4n}$ to be $4n$ equally spaced points in the interval $(\frac{3}{4}, 1)$. Define
\begin{align*}
    \mX^1 = \mA;& \quad\quad \mX^2 = -\mA \\ 
    \mX^3 = -\mA + \frac{1}{2}\vone^\top;& \quad\quad \mX^4 = \mA - \frac{1}{2}\vone^\top
\end{align*}
and 
\begin{align*}
    \mY^1 = \vone^\top;& \quad\quad \mY^2 = \vone^\top \\
    \mY^3 = -\vone^\top;& \quad\quad \mY^4 = -\vone^\top.
\end{align*}
Notice that the indices also match which quadrant the cluster of points are in. Visually, these four groups of points $\mZ^i \triangleq (\mX^i, \mY^i)$ are clusters with the $i$th cluster in the $i$th quadrant. For brevity, we also refer to these clusters by their index $i$ only, so cluster $1$ refers to the cluster $\mZ^1 = (\mX^1, \mY^1)$, and so on. These definitions are consistent with the clusters depicted in \cref{fig:toy-dataset-reg-unnormalized-alt}. 

Then take $\mX = \mqty[\mX^1 & \mX^2 & \mX^3 & \mX^4] \in \RR^{1 \times 16n}$ and $\mY = \mqty[\mY^1 & \mY^2 & \mY^3 & \mY^4] \in \RR^{1 \times 16n}$. After applying BN with permutation $\pi$ and batch size $2$, we obtain a dataset $\Xpibn$ with every point being located in one of four SS clusters $\Zpi^i \triangleq (\Xpibn^i, \Ypi^i)$ for $i \in [4]$ located at $(\pm 1, \pm 1)$ with the same relative labelling: $\Zpi^1$ is located at $(1, 1)$ and then labelling counterclockwise (see \cref{fig:toy-dataset-reg-normalized}).

In \cref{fig:toy-dataset-reg-unnormalized-alt}, we visualize the construction with $n=3$ (so the depicted dataset has $16n=48$ datapoints). We plot the slopes of the $\Mgd^*$ (green solid line), $\Mrr^*$  (purple dash-dotted line), and typical values for $\Mpi^*$ (yellow dotted line). In \cref{fig:toy-dataset-reg-normalized}, we show what $\Zpi$ looks like for a typical permutation $\pi$. The sizes of the points represent the number of points that end up in the corresponding cluster.

\begin{figure}[!ht]
    \centering
    \begin{subfigure}[b]{0.48\textwidth}
         \centering
         \includegraphics[width=\textwidth]{reg-dataset-unnormalized.pdf}
         \caption{Unnormalized toy regression dataset $\mZ$ demonstrating distortion of SS with constant probability. Notice how the GD and RR lines are aligned with slope 0, but the SS lines are distorted away from 0.}
         \label{fig:toy-dataset-reg-unnormalized-alt}
     \end{subfigure}
     \hfill 
     \begin{subfigure}[b]{0.48\textwidth}
         \centering
         \includegraphics[width=\textwidth]{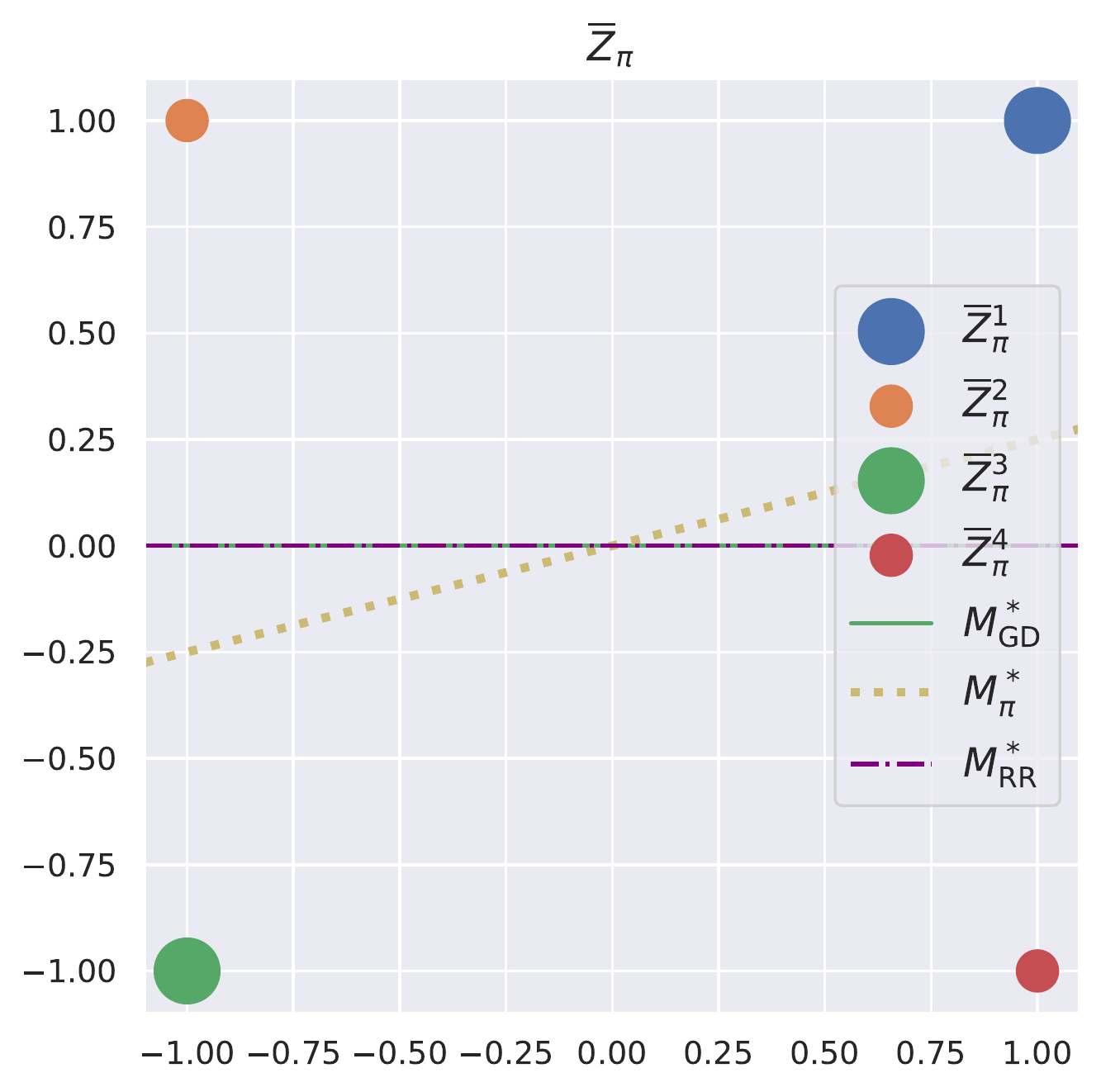}
         \caption{Toy regression dataset after BN with permutation $\pi$. The size of the point corresponding to $\Zpi^i$ represents the number of normalized points in $\Zpi^i$ (not to scale), and here we picked $\pi$ such that $\Mpi^* = \frac{1}{4}$.}
         \label{fig:toy-dataset-reg-normalized}
     \end{subfigure}
     
\end{figure}

\subsubsection{Analyzing the regression toy dataset}

In order to prove this proposition, we will need the following standard technical estimate on the asymptotics of binomial coefficients (see e.g. \citet{thomas2006elements})
\begin{lemma}\label{lemma:binomial-estimate}
For all $n$ and $k$ we have
\[
\sqrt{\frac{n}{8k(n-k)}}2^{H(k/n)n} \le \binom{n}{k} \le \sqrt{\frac{n}{\pi k(n-k)}}2^{H(k/n)n},
\]
where $H(p) \triangleq -p\log_2 p - (1-p)\log_2(1-p)$ is the binary entropy function. 
\end{lemma}
Now let us prove the prove the proposition. 
\begin{proof}[Proof of \cref{prop:toy-regression-dataset}]
First, it is clear from the symmetry of $\mX$ that $\Ygd \Xgdbn^\top = 0$, so $\Mgd^* = 0$. This proves the first half of \ref{item:reg-prop-1}.

\paragraph{Setup:}
For $i \in [4]$, let $\kpi{i}$ denote the number of points that end up in cluster $\Zpi^i$ after normalizing with permutation $\pi$. Since we have $\Mpi^* = \frac{1}{16n}\Ypi \Xpibn^\top$, evidently the $\kpi{i}$ completely determine $\Mpi^*$. In fact, more is true: we claim that $\Mpi^* = 0$ if and only if $\kpi{1} = 4n$. 

To see why this is true, first note that the label $\mY$ is unaffected by BN. Hence, there are necessarily $8n$ points that end up with $y$ coordinate $1$. It follows that if there are $k$ points in $\Zpi^1$, then there are $8n-k$ points in $\Zpi^2$. Similarly, if there are $j$ in $\Zpi^3$, then there are $8n-j$ in $\Zpi^4$. On the other hand, recall that $\bn$ with $B=2$ and $d=1$ always sends one point in each batch to $x=+1$ and one point to $x=-1$. Hence, there are $8n$ points that end up with $x$ coordinate $1$, which means $k+8n-j = 8n$, implying that $k=j$. 

Referring back to the formula for $\Mpi^*$, we see that 
\begin{align*}
   \Mpi^* &= \frac{1}{16n}\sum_{i=1}^4 \Ypi^i (\Xpibn^i)^\top  \\
   &= \frac{1}{16n}(k - (8n-k) + k - (8n-k)) \\
   &= \frac{k-4n}{4n}.
\end{align*}
Hence $\Mpi^* = 0$ if and only if $\kpi{1} = 4n$. Referring back to \cref{fig:toy-dataset-reg-normalized}, we see that the sizes of the clusters represent $\kpi{i}$, and the plotted $\Mpi^*$ with slope $\frac{1}{4}$ corresponds to a $\pi$ where $\kpi{1} = 5n$.

To analyze $\kpi{1}$, for $i,j \in [4]$ we introduce the random variables $\Tpi{i}{j}$ to denote the number of batches formed with one point from cluster $i$ and cluster $j$ with permutation $\pi$. Similarly let $U \triangleq \qty{1, 2}$ denote the upper clusters and $L \triangleq \qty{3, 4}$ denote the lower clusters. Define $\Tpi{i}{L} \triangleq \Tpi{i}{3} + \Tpi{i}{4}$, which represents the total number of batches with one point in cluster $i$ and another point in $L$ when using permutation $\pi$. Finally, define $\Tpi{U}{L} \triangleq \Tpi{1}{L} + \Tpi{2}{L}$, which represents the total number of batches with one point in cluster $U$ and another in $L$ with permutation $\pi$.

With this notation in hand, let us prove the claim.
Evidently, $\kpi{1} = \Tpi{U}{U} + \Tpi{1}{L}$. Similarly we have $\kpi{2} = \Tpi{U}{U} + \Tpi{2}{L}$. 
Because we established earlier that $\kpi{1} + \kpi{2} = 8n$, 
then $\kpi{1} = 4n$ if and only if $\Tpi{1}{L} = \Tpi{2}{L}$. In words, this means that $\Mpi^* = 0$ if and only if the number of $\qty{1, L}$ batches is the same as the number of $\qty{2, L}$ batches.

\paragraph{RR averages out distortion:}
For every $\pi$ we can find $\pi'$ such that $\mM_{\pi'}^* = -\Mpi^*$. This is because we can always find $\pi'$ that swap $\Tpi{1}{L}$ and $\Tpi{2}{L}$, by turning all $\qty{1, L}$ batches into $\qty{2, L}$ batches and vice versa. Then \cref{prop:ss-rr-relationship} implies that $\Mrr^* = 0$, which proves the second half of \ref{item:reg-prop-1}.

\paragraph{SS is distorted:}
Now, let us show that $\PP[\kpi{1} = 4n] = O(\frac{1}{\sqrt{n}})$. The main idea is that conditioned on $\Tpi{U}{L} = 2t$, we can compute the probability that $\Tpi{1}{L} = \Tpi{2}{L} = t$ exactly. Indeed, of the $4n$ points in cluster 1, we pick $t$ of them to form batches with $L$, and similarly for cluster $2$. This gives $\binom{4n}{t}^2$ ways for $\Tpi{1}{L} = \Tpi{2}{L} = t$. In total, there are $8n$ points in $U$ and we picked $2t$ of them to match with $L$, which gives a denominator of $\binom{8n}{2t}$. Hence 
\[
\PP[\kpi{1} = 4n | \Tpi{U}{L} = 2t] = \frac{\binom{4n}{t}^2}{\binom{8n}{2t}}.
\]

In order to obtain the $O(\frac{1}{\sqrt{n}})$ bound we desire, we need to use the fact that $\Tpi{U}{L}$ --- the number of batches between $U$ and $L$ --- concentrates tightly. In fact, if we color the 4 clusters corresponding to membership in $U$ and $L$ and slightly generalize the analysis leading to  \cref{fact:whp-monochromatic}, we obtain that for some absolute constant $C$, we have $\abs{\Tpi{U}{L} - 4n} \le 2C\sqrt{n\log n}$ with probability at least $1-1/n$. 

Applying \cref{lemma:binomial-estimate}, we obtain for all $t$ such that $\abs{t-2n} \le 2C\sqrt{n \log n}$, we have
\begin{align}
    \PP[\kpi{1} = 4n | \Tpi{U}{L} = 2t] &=  O\qty(\frac{\frac{n}{t(4n-t)}2^{8H(\frac{t}{4n})n}}{\sqrt{\frac{n}{t(4n-t)}}2^{8H(\frac{t}{4n})n}}) \\
    &= O\qty(\sqrt{\frac{n}{t(4n-t)}}) \\
    &= O\qty(\frac{1}{\sqrt{n}}). \label{eq:conditional-prob}
\end{align}

Hence we have 
\begin{align*}
    \PP[\kpi{1} = 4n] &= \sum_{t=0}^{4n} \PP[\kpi{1} = 4n | \Tpi{U}{L} = 2t] \PP[\Tpi{U}{L} = 2t] \\
    &\le \frac{1}{n} + \sum_{\abs{t-2n} \le C\sqrt{n\log n}}  \PP[\kpi{1} = 4n | \Tpi{U}{L} = 2t] \PP[\Tpi{U}{L} = 2t]\\
    &\le \frac{1}{n} + O\qty(\frac{1}{\sqrt{n}}) \\
    &\le O\qty(\frac{1}{\sqrt{n}}),
\end{align*}
where in the second line we have used the union bound  along with the fact that $\Tpi{U}{L}$ concentrates, and in the third line we have used \cref{eq:conditional-prob}. This proves \ref{item:reg-prop-2}.

\paragraph{Quantitative SS distortion bounds with constant probability:}
Finally, we show \ref{item:reg-prop-3}. Suppose that $\kpi{1} = 4n+d$ for $\abs{d} = O(\sqrt{n})$. The above analysis for the case of $d = 0$ immediately generalizes to show that, if $t-d > 0$, 
\[
\PP[\kpi{1} = 4n+d|\Tpi{U}{L} = 2t-d] = \frac{\binom{4n}{t}\binom{4n}{t-d}}{\binom{8n}{2t-d}}. 
\]
Notice that since $2t-d$ concentrates around $4n$, it suffices to only consider the high probability regime where $2t-d = 4n + O(\sqrt{n\log n})$. In particular, since $\abs{d} = O(\sqrt{n})$, we have $t-d=O(t)$.

Thus, if we plug in \cref{lemma:binomial-estimate}, we obtain in the regime where $t-d=O(t)$ that
\[
 \frac{\binom{4n}{t}\binom{4n}{t-d}}{\binom{8n}{2t-d}} = O\qty(\frac{1}{\sqrt{t}}) 2^{4n\qty[H(\frac{t}{4n}) + H(\frac{t-d}{4n}) - 2H(\frac{2t-d}{8n})]}.
\]
Concavity of binary entropy implies that $H(\frac{t}{4n}) + H(\frac{t-d}{4n}) - 2H(\frac{2t-d}{8n}) \le 0$. It follows that 
\[
\PP[\kpi{1} = 4n+d|\Tpi{U}{L} = 2t-d] = O\qty(\frac{1}{\sqrt{t}}).
\]
Following the same argument as in the $d=0$ case, we have for $\abs{d} = O(\sqrt{n})$ that
\begin{align*}
    \PP[\kpi{1} = 4n+d] &= \sum_{2t-d=0}^{4n} \PP[\kpi{1} = 4n | \Tpi{U}{L} = 2t-d] \PP[\Tpi{U}{L} = 2t-d] \\
    &\le \frac{1}{n} + \sum_{\abs{t-\frac{d}{2}-2n} \le C\sqrt{n\log n}}  \PP[\kpi{1} = 4n+d | \Tpi{U}{L} = 2t-d] \PP[\Tpi{U}{L} = 2t-d]\\
    &\le \frac{1}{n} + O\qty(\frac{1}{\sqrt{n}}) \\
    &\le O\qty(\frac{1}{\sqrt{n}}),
\end{align*}
From here, it follows that there exists some positive constant $c$ such that
\[
\PP[\abs{\kpi{1} - 4n} > c\sqrt{n}] = \Omega(1),
\]
which proves \ref{item:reg-prop-3}.
\end{proof}

\subsection{Classification toy dataset}
In this section, we motivate how we constructed our toy classification dataset parameterized by $n$. We then give a detailed construction and analysis of the dataset, parameterized by $n$. 
We plot the unnormalized $\mZ$ in \cref{fig:toy-dataset-clf-unnorm-alt} and the normalized $\Zgd$ in \cref{fig:toy-dataset-clf-norm-alt} for $n=10$. 

The main idea is that since $d=2$ and the optimal direction is orthogonal to the SC portion of the separability decomposition (\Cref{lemma:infimize}), we can fix the optimal directions of $\Zpi$ and $\Zgd$ by carefully constraining $\Span(\Xpisc)$ and $\Span(\Xgdsc)$, respectively. For, $\Zgd$ we carefully select the boundary points $\mX_{\mathrm{bdr}}$ which define the boundary of $\conv(\Xgdbn^+)$ and $\conv(\Xgdbn^-)$ so that $\dim(\Span(\Xgdsc)) = 1$. For $\Zpi$, we need to ensure that $\Span(\Xpisc)$ is a one dimensional subspace of $\RR^2$ which is close to orthogonal with $\Span(\Xgdsc)$. Although we can guarantee $\dim(\Span(\Xpisc)) \ge 1$ w.h.p., the main subtlety here is ensuring that equality holds with constant probability. Given the above, we are afforded the luxury of adding datapoints which are misclassified by $\vpi^*$, the optimal direction of $\ellpi$.

\begin{proposition}\label{prop:toy-classification-dataset}
    There exists a classification dataset $\mZ = (\mX, \mY) \in [-3, 3]^{2 \times (2n+6)} \times \qty{-1, 1}^{2n+6}$ such that the following statements hold with batch size $B=2$: 
    \begin{enumerate}[label=\normalfont{(\arabic*)},ref=(\theenumi)]
        \item \label{item:clf-prop-1} $\Zgd$ is PLS and as $n \to \infty$, $\vgd^*$ converges in direction to $\mqty[1 & 2]^\top$. 
        \item \label{item:clf-prop-2} $\Zpi$ is PLS with constant probability. If so, we have $\vpi^* = \mqty[1 & -1]^\top$.
        \item \label{item:clf-prop-3} There exists points $(\vx_i, y_i) \in \Zgd$ such that $y_i \langle \vpi^*, \vx_i \rangle < 0$, i.e., GD points that $\vpi^*$ misclassifies.
    \end{enumerate}
    Hence, the GD risk $\ellgd$ diverges with constant probability if we train with SS.
\end{proposition}

\begin{proof}
In our construction, we separate out $\mZ = (\mX, \mY)$ into several groups of points: $\mX^+_{\mathrm{cor}}$, $\mX^+_{\mathrm{err}}$,  $\mX^+_{\mathrm{bdr}}$, and their negative variants. Let the overlined version of these matrices denote the corresponding points after full-batch BN (i.e., after taking $\Xgdbn \triangleq \bn(\mX)$), and the overlined version with an extra $\pi$ subscript denoting the corresponding points after BN with permutation $\pi$ (i.e., $\Xpibn \triangleq \bnpi(\mX)$). For example, $\ol{\mX}^+_{\mathrm{cor}}$ denotes the features for the positive examples in $\Xgdbn$ that are to be classified correctly by $\vpi^*$, whereas $\ol{\mX}^+_{\mathrm{cor},\pi}$ denotes those same datapoints in $\Xpibn$ after batch normalization with $\pi$. 

\paragraph{Setup:}
Let us first explain the semantic meanings of the different groups of points in our construction. 
\begin{align*}
    \mX^+_{\mathrm{cor}} &: \text{Unnormalized positive examples correctly classified by $\vpi^*$ with positive margin} 
 \\ 
 \mX^+_{\mathrm{err}} &: \text{Unnormalized positive example incorrectly classified by $\vpi^*$}  \\
 \mX^+_{\mathrm{bdr}} &: \text{Unnormalized positive examples on the decision boundary of $\vv^*$} \\
    \ol{\mX}^+_{\mathrm{cor}} &: \text{full-batch-normalized positive examples correctly classified by $\vpi^*$ with positive margin} 
 \\ 
 \ol{\mX}^+_{\mathrm{err}} &: \text{full-batch-normalized positive example incorrectly classified by $\vpi^*$}  \\
 \ol{\mX}^+_{\mathrm{bdr}} &: \text{full-batch-normalized positive examples on the decision boundary of $\vgd^*$}
\end{align*}
The semantic meanings of the negative versions of these points are completely analogous. We also define $\mX_{\mathrm{bdr}} = \mX^+_{\mathrm{bdr}} \cup \mX^-_{\mathrm{bdr}}$, and the normalized quantity analogously. 

We construct $\mX^+_{\mathrm{cor}}$, $\mX^+_{\mathrm{err}}$, and $\mX^+_{\mathrm{bdr}}$ as follows. Take $\mX^+_{\mathrm{cor}}$ to be $n$ equally spaced points on the line segment of width $\frac{1}{n}$ centered at $\mqty[2 & 2]^\top$. For the sake of visual clarity, we increase the spacing of the points in the diagram \cref{fig:toy-dataset-clf-unnorm-alt}. Define $\mX^+_{\mathrm{err}}$ to be $\mqty[3 & 2.5]^\top$. Next, define $\mX^+_{\mathrm{bdr}}$ to be $\mqty[-3 & 1 \\
1.5 & -0.5]$, lying on the line $y = -0.5x$. 
Finally, define 
\begin{align*}
    \mX^-_{\mathrm{cor}} &= -\mX^+_{\mathrm{cor}} \\
    \mX^-_{\mathrm{err}} &= -\mX^+_{\mathrm{cor}} \\
    \mX^-_{\mathrm{bdr}} &= -\mX^+_{\mathrm{bdr}}.
\end{align*}

In \cref{fig:classification-datasets}, we visualize this toy dataset. Note the visual similarity between $\mZ$ in \cref{fig:toy-dataset-clf-unnorm-alt} and $\Zgd$ in \cref{fig:toy-dataset-clf-norm-alt}; this is a feature of the construction. Indeed, as we'll see shortly, as $n \to \infty$, $\Xgdbn$ approaches a uniform rescaling of $\mX$ in all coordinates. We also plotted the decision boundaries corresponding to $\vgd^*$ and $\vpi^*$. We highlight the fact that in \cref{fig:toy-dataset-clf-norm-alt}, $\ol{\mX}^+_{\mathrm{err}}$ and $\ol{\mX}^-_{\mathrm{err}}$ are both on the wrong side of the decision boundary for $\vpi^*$. 
\begin{figure}[!ht]
    \centering
    \begin{subfigure}[b]{0.48\textwidth}
         \centering
        \includegraphics[width=\textwidth]{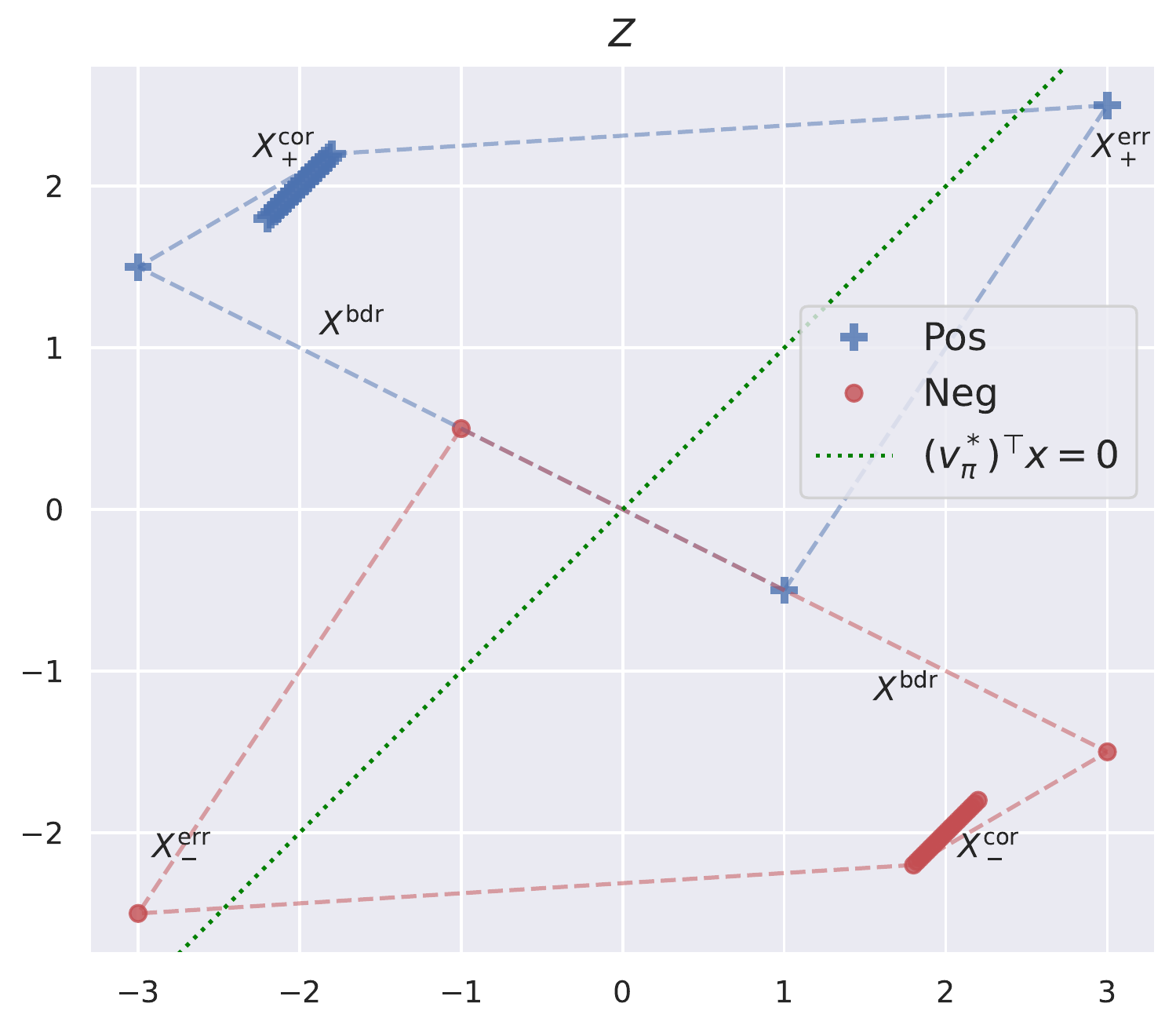}
         \caption{Unnormalized toy classification dataset $\mZ$.\newline}
         \label{fig:toy-dataset-clf-unnorm-alt}
     \end{subfigure}
     \hfill
    \begin{subfigure}[b]{0.48\textwidth}
         \centering
        \includegraphics[width=\textwidth]{toy_dataset_clf_normalized.pdf}
         \caption{Normalized toy classification dataset $\Zpi$ demonstrating divergence of SS with constant probability.}
         \label{fig:toy-dataset-clf-norm-alt}
     \end{subfigure}
    \caption{Toy classification dataset (a) before full-batch BN, i.e. $\mZ$ (b) after full-batch BN, i.e. $\Zgd$. }
    \label{fig:classification-datasets}
\end{figure}
\paragraph{GD is PLS:}
Evidently $\vmu = \vzero$ and one can compute that as $n \to \infty$ that $\vsigma \to \mqty[2 & 2]^\top$. Regardless, we see that $\Zgd$ is PLS with $\Span(\Xgdsc)$ corresponding to the one-dimensional subspace spanned by $\ol{\mX}_{\mathrm{bdr}}$ and $\vgd^*$ correctly classifies all of the other points. In the limit $n \to \infty$, we have that $\Span(\ol{\mX}_{\mathrm{bdr}}) = \Span(\mqty[-2 & 1]^\top)$ and $\vgd^*$ is in the direction $\mqty[1 & 2]^\top$. This proves \ref{item:clf-prop-1}.

\paragraph{SS is PLS with constant probability:}
Now, let us compute what happens to $\Zpi$. Because BN with $B=2$ sends features to $\pm 1$, this implies that a normalized batch is either mapped to 
\begin{enumerate}[label=\normalfont{(\arabic*)}]
    \item $\mqty[-1 & +1 \\
-1 & +1]$, i.e. the normalized batch lies in the direction $\mqty[+1 & +1]^\top$; or 
    \item $\mqty[-1 & +1 \\
+1 & -1]$, i.e the normalized batch lies in the direction $\mqty[+1 & -1]^\top$. 
\end{enumerate} 

Note that due to \cref{fact:whp-monochromatic}, with high probability, there will be a batch drawn from $\mX^+_{\mathrm{cor}}$ and a batch drawn from $\mX^-_{\mathrm{cor}}$. These batches necessarily land in situation (1), so $\mqty[1 & 1]^\top \in \Span(\Xpisc)$  with high probability. On the other hand, with high probability there will be a batch with one point from $\mX^+_{\mathrm{cor}}$ and one from $\mX^-_{\mathrm{cor}}$, which lands in (2). Hence, as long as $\Span(\Xpisc) = \Span(\mqty[1 & 1]^\top)$, this implies that $\Zpi$ is PLS with optimal direction $\mqty[-1 & +1]^\top$. 

By inspection, to ensure that $\Span(\Xpisc) = \Span(\mqty[1 & 1]^\top)$, there are two bad events we need to avoid:
\begin{enumerate}[label=(\alph*)]
    \item We send a positive example to $\mqty[+1 & -1]^\top$.
    \item We send a negative example to $\mqty[-1 & +1]^\top$.
\end{enumerate}
We will use 
\[
\PP[\text{avoid (a) and (b)}] = \PP[\text{avoid (a) }|\text{ avoid (b)}] \PP[\text{avoid (b)}].\]
Note that by symmetry, $\PP[\text{avoid (b)}] = \PP[\text{avoid (a)}]$. 

A little thought reveals that (a) can happen only if the positive boundary example (i.e. in $\mX^+_{\mathrm{bdr}}$) located originally at $\mqty[1 & -0.5]^\top$ is batched together with a point in $\mX_{\mathrm{bdr}} \cup \mX^+_{\mathrm{cor}}$. In turn, this event occurs with probability at most $\frac{2}{3}$. So $\PP[\text{avoid (b)}] \ge \frac{1}{3}$. Also, notice avoiding (a) still happens with probability at least $\frac{1}{3}$ even after conditioning on avoiding (b). Hence, the probability that we avoid both (a) and (b) is at least $\frac{1}{3}^2 = \frac{1}{9}$. So with constant probability, $\Span(\Xpisc) = \Span(\mqty[1 & 1]^\top)$. This proves \ref{item:clf-prop-2}.

Putting it all together, we see that with constant probability, $\Span(\Xpisc) = \Span(\mqty[1 & 1]^\top)$, and $\Zpi$ is PLS with optimal direction $\vpi^* = \mqty[+1 & -1]^\top$. Recall that $\vgd^* \to \mqty[1 & 2]^\top$ as $n \to \infty$. So asymptotically we have 
\[
\frac{\abs{\ev{\vpi^*, \vgd^*}}}{\norm{\vpi^*}\norm{\vgd^*}} = \frac{1}{\sqrt{10}}.
\]

\paragraph{SS misclassifies GD points:}
On the constant probability event that $\Zpi$ is PLS, we have $\vpi^*$ is in the direction $\mqty[1 & -1]^\top$. This misclassifies $\ol{\mX}^+_{\mathrm{err}}$ and $\ol{\mX}^-_{\mathrm{err}}$. So this proves \ref{item:clf-prop-3}.
\end{proof}
\begin{remark}
Note that at the cost of visual clarity, we can modify the construction to obtain optimal classifiers $\vpi^*$ and $\vgd^*$ which are asymptotically orthogonal. In this alternate construction, we take $\mX^+_{\mathrm{bdr}} = \mqty[-1 & 0.5 \\
1 & -0.5]$, lying on the line $y = -x$, and $\mX^+_{\mathrm{cor}}$ to be $n$ equally spaced points on the line segment between $\mqty[-2 & 2]^\top$ and $\mqty[-2+\frac{1}{n} & 2+\frac{1}{n}]^\top$. Then $\vgd^* \to \mqty[1 & 1]^\top$, and $\vpi^* = \mqty[1 & -1]^\top$ with constant probability, as desired.
\end{remark}

\end{document}